\pdfoutput=1
\documentclass[11pt]{article}
\usepackage{amsthm}

\usepackage[final]{acl}

\usepackage{times}
\usepackage{latexsym}

\usepackage[T1]{fontenc}

\usepackage[utf8]{inputenc}

\usepackage{microtype}
\usepackage{microtype}
\usepackage{graphicx}
\usepackage{subcaption} 
\usepackage{booktabs} 
\usepackage{tabularx}
\usepackage{xcolor} 
\usepackage{soul}
\usepackage{amsmath}
\usepackage[most]{tcolorbox}
\tcbset{breakable} 
\usepackage{enumitem}
\usepackage{amssymb}
\usepackage{mathtools}
\usepackage{amsthm}
\usepackage{dblfloatfix}
\usepackage{inconsolata}
\usepackage[table]{xcolor}
\definecolor{highlightblue}{RGB}{235, 240, 255}
\definecolor{highlightred}{RGB}{255, 240, 240}
\definecolor{darkgreen}{RGB}{0, 100, 0}
\definecolor{darkred}{RGB}{139, 0, 0}

\usepackage{threeparttable}
\usepackage{multirow} 
\definecolor{stanfordred}{RGB}{140,21,21} 
\usepackage{hyperref}

\usepackage{graphicx}
\theoremstyle{plain}
\newtheorem{theorem}{Theorem}[section]

\newtheorem{lemma}[theorem]{Lemma}
\newtheorem{corollary}[theorem]{Corollary}
\theoremstyle{definition}

\theoremstyle{remark}
\newtheorem{remark}[theorem]{Remark}
\usepackage[dvipsnames]{xcolor}

\title{Continuous-Utility Direct Preference Optimization}

\author{
  Muhammad Ahmed Mohsin \and
  Muhammad Umer \and
  Emily Fox \\
  Stanford University \\
  \texttt{muahmed@stanford.edu} \quad
  \texttt{mumer@stanford.edu} \quad
  \texttt{ebfox@stanford.edu}
}

\begin{document}
\maketitle
\begin{abstract}
Large language model reasoning is often treated as a monolithic capability, relying on binary preference supervision that fails to capture partial progress or fine-grained reasoning quality. We introduce continuous utility direct preference optimization \textbf{(\textcolor{MidnightBlue}{CU-}\textcolor{NavyBlue}{DPO})}, a framework that aligns models to a portfolio of prompt-based cognitive strategies by replacing binary labels with continuous scores that capture fine-grained reasoning quality. We prove  that learning with $K$ strategies yields a $\Theta(K \log K)$ improvement in sample complexity over binary preferences and that DPO converges to the entropy-regularized utility-maximizing policy. To exploit this signal, we propose a two-stage pipeline: (i)~strategy selection, which optimizes the model to choose the best strategy via best-vs-all comparisons, and (ii)~execution refinement, which trains correct execution using margin-stratified pairs. The framework is domain-agnostic: any task admitting cognitively distinct solution strategies and a decomposable continuous utility signal can be incorporated into the portfolio. On mathematical reasoning benchmarks, \textbf{\textcolor{MidnightBlue}{CU-}\textcolor{NavyBlue}{DPO}} improves strategy selection accuracy from 35--46\% to 68--78\% across seven base models, yielding downstream reasoning gains of up to +6.6 points on in-distribution datasets with effective out-of-distribution transfer. \textbf{\textcolor{MidnightBlue}{CU-}\textcolor{NavyBlue}{DPO}} demonstrates consistent gains on code generation and causal reasoning benchmarks, confirming generalization beyond the mathematical domain.
\end{abstract}

\section{Introduction}
Large language models (LLMs) exhibit strong mathematical reasoning 
capabilities, yet they often commit to a single ``thinking style'' 
and collapse when problems demand different cognitive procedures, 
applying uniform strategies regardless of problem structure. This 
contrasts with human mathematical cognition, where experts fluidly 
select among diverse reasoning modalities based on problem 
characteristics~\citep{schoenfeld1985mathematical, polya1945how}. This same strategy collapse appears in other domains as well.

\begin{figure*}[t!]
    \centering
    \includegraphics[width=\textwidth]{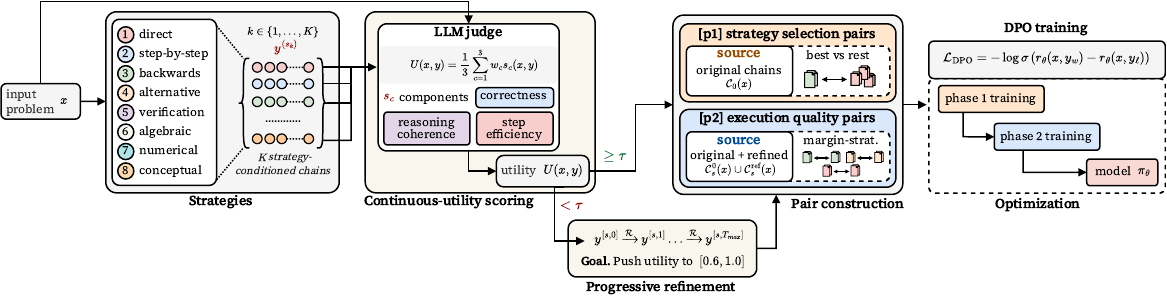}
    \caption{\textbf{\textcolor{MidnightBlue}{CU-}\textcolor{NavyBlue}{DPO}} 
    \textbf{overview.} Strategy-conditioned sampling $\rightarrow$ 
    LLM-judged continuous utilities $\rightarrow$ progressive 
    refinement $\rightarrow$ high-signal pair construction 
    (Phase~1 [p1]: strategy selection, Phase~2 [p2]: execution 
    refinement) $\rightarrow$ utility-weighted DPO training.}
    \label{fig:flowchart}
\end{figure*}

Furthermore, strategy collapse is not the only obstacle: beyond
\emph{obtaining} a diverse portfolio of reasoning strategies, one must
also \emph{reward} how well the chosen strategy is executed. The
dominant alignment paradigm, direct preference optimization~(DPO)
~\citep{rafailov2023direct}, addresses neither cleanly, as it relies on
binary labels that misrepresent the continuous nature of reasoning
quality: a chain containing correct sub-derivations that fails at a
single arithmetic step receives the same reject label as a
fundamentally flawed chain. We introduce
\textbf{\textcolor{MidnightBlue}{CU-}\textcolor{NavyBlue}{DPO}}, which
replaces binary supervision with \emph{continuous, high-fidelity
utilities} that aggregate LLM-judged scores over answer correctness,
reasoning quality, and completeness. Concretely, each candidate
reasoning chain (outcome) $y_i$ is assigned a continuous utility
$U(x, y_i) \in [0,1]$, and preferences are modeled through
Bradley-Terry as $P(y_w \succ y_\ell \mid x) = \sigma\!\big(U(x, y_w) -
U(x, y_\ell)\big)$. Under this model, the DPO-optimal policy satisfies
$r_\theta(x, y_i) - r_\theta(x, y_j) = U(x, y_i) - U(x, y_j)$ for every
pair of chains, so DPO recovers the underlying continuous-utility
structure up to a problem-dependent constant
(Theorem~\ref{thm:dpo-convergence}). This continuous signal is also far
more sample-efficient: recovering the latent quality ranking of $K$
candidate chains across $N$ problems requires $\Omega(NK^2 \log K)$
binary comparisons, whereas observing the $K$ utilities directly needs
only $O(NK)$ samples, a $\Theta(K \log K)$ multiplicative
improvement.

A further challenge is conflicting supervision signals: comparing 
a correctly executed ``bad strategy'' to a poorly executed ``good 
strategy'' confuses whether to optimize strategy choice or execution 
quality. We address this by constructing two distinct preference 
pair sets for sequential training --- Phase~1 for strategy 
selection via best-vs-all comparisons, Phase~2 for execution 
refinement via margin-stratified intra-strategy pairs.

Empirical analysis across 450 problems from DeepMath, HARDMath2, 
and ProofNet reveals that strategy-conditioned generation produces 
utility scores ranging from 0.1 to 0.9 \emph{for the same problem} 
depending on strategy choice~(Appendix~\ref{sec:ablation_dataset_characteristics}), 
demonstrating that strategy selection alone accounts for substantial 
solution quality variance. While grounded in mathematical reasoning, 
CU-DPO extends naturally to any domain admitting diverse 
problem-solving strategies and decomposable continuous quality 
metrics, as we validate on code generation and causal reasoning.

\textbf{Contributions.} We present 
\textbf{\textcolor{MidnightBlue}{CU-}\textcolor{NavyBlue}{DPO}}, a 
unified framework for structured reasoning alignment applicable to 
any domain admitting cognitively distinct solution strategies and a 
decomposable continuous utility signal. \textbf{(1)}~We train models 
to select optimal reasoning strategies from a domain-agnostic 
portfolio, instantiated for mathematical reasoning~(8 strategies), 
code generation~(4 strategies), and causal reasoning~(6 strategies). 
\textbf{(2)}~We replace binary preferences with continuous utility 
scores, proving a $\Theta(K \log K)$ sample complexity improvement 
and establishing convergence to the entropy-regularized optimal 
policy with learned implicit reward $r_\theta(x, y) = U(x,y) + c(x)$ 
(Theorems~\ref{thm:sample-efficiency} 
and~\ref{thm:dpo-convergence}). \textbf{(3)}~We separate strategy 
selection~(Phase~1) from execution refinement~(Phase~2), preventing 
supervision conflicts; combined with step-level supervision, CU-DPO 
addresses orthogonal failure modes and yields strictly complementary 
gains~(Appendix~\ref{appendix:step_comparison}). \textbf{(4)}~We 
validate across in-distribution and out-of-distribution math 
datasets, code generation, and causal reasoning across seven base 
models.

\section{Utility and Pairwise Preferences} \textbf{\textcolor{MidnightBlue}{CU-}\textcolor{NavyBlue}{DPO}} addresses two key questions: \textbf{(i)}~how to define the utility function and \textbf{(ii)}~how to construct preference pairs that cleanly separate strategy selection from execution quality. As illustrated in Figure~\ref{fig:flowchart}, input problems are transformed into high-signal preference pairs via strategy-conditioned sampling, continuous utility scoring~(correctness, preciseness, coherence) through an LLM judge, and selective refinement. This produce yields 3,150 strategy selection pairs~(Phase~1) and 2,680 execution refinement pairs~(Phase~2) for mathematical reasoning. For code generation and causal reasoning, the same pipeline is instantiated with domain-specific strategy portfolios and utility decompositions, confirming that pair construction generalizes without architectural modification.

\subsection{Strategy-conditioned chain sampling}
\label{sec:chain_sampling}
For each problem instance $x \in \mathcal{X}$, we generate $K{=}8$ candidate reasoning chains from a fixed base model $\pi_0$ using a set of strategy prompts $\mathcal{S}=\{s_1,\dots,s_K\}$. Each strategy $s \in \mathcal{S}$ is implemented by a prompt template $g_s(\cdot)$ that induces a distinct problem-solving approach (e.g., step-by-step, backwards reasoning, verification, more details in Appendix~\ref{app:strategy_prompts}). We sample one chain per strategy as $y^{(s)} \sim \pi_0\!\left(\cdot \mid x, g_s(x)\right),$ yielding a problem-specific set of reasoning chains as $\mathcal{C}_0(x) = \{y^{(s)}\}_{s\in\mathcal{S}}$. Across $N{=}450$ problems, this produces $NK{=}3{,}600$ strategy-conditioned chains.

The motivation is mechanism mismatch: models commit to 
inappropriate approaches even when the correct one is within 
their capabilities. Strategy conditioning makes this observable 
at the instance level, inducing a utility landscape over 
approaches for the same $x$. The per-problem utility range 
has mean $0.295$ (min $0.050$, max $0.889$), with $87\%$ of 
problems having range $> 0.3$, providing the core signal for 
Phase~1.

\subsection{Continuous utility scoring via LLM Judge}\label{sec:utility}
Each chain $y$ is evaluated by an LLM judge that outputs a decomposed
utility score $U(x,y) = \frac{1}{3}\sum_{c=1}^{3} w_c \cdot s_c(x,y) \in
[0,1]$, where $s_c \in [0,1]$ denotes three scoring components, namely
\textbf{(1)} correctness verification, \textbf{(2)} step efficiency, and
\textbf{(3)} reasoning coherence. The LLM judge (Qwen 2.5 7B) produces
both the scoring components $s_c$ and the weights $w_c$, rather than
relying on hand-tuned weights, so that the relative importance of each
component can adapt to the problem being scored. It consistently assigns
the highest weight to correctness and lower weights to coherence and
efficiency, which keeps the aggregate utility strongly aligned with
ground-truth correctness ($r = 0.85$,
Appendix~\ref{sec:utility-quality-correlation}). This alignment is a
validity check rather than the objective. Placing all weight on
correctness would collapse the utility back to a binary
correct-or-incorrect signal and discard exactly the graded structure we
seek. The coherence and efficiency components instead supply
fine-grained distinctions among chains that share the same correctness
outcome, which binary labels cannot express. For instance, a chain with
correct reasoning but a minor arithmetic error scores moderately
($U \approx 0.5$) rather than receiving the same reject label as a
fundamentally flawed chain ($U \approx 0.2$). To mitigate length bias,
the step-efficiency component penalizes verbosity, and utility
correlates only weakly with chain length ($r = 0.23$). Further details
appear in Appendix~\ref{app:llm_judge}.

\subsection{Progressive refinement for signal amplification}\label{sec:refine}
Low-utility chains often contain recoverable substructure, such as
correct initial steps or valid algebraic manipulations that fail only
due to localized errors. We selectively refine chains that fall below
the threshold $\tau = 0.4$ using a refinement operator $\mathcal{R}$,
which produces a corrected chain $y' = \mathcal{R}(x, y)$. At a high
level, $\mathcal{R}$ re-prompts the same base model with the flawed
chain and asks it to repair localized mistakes while preserving the
original solution strategy, so the refined chain stays within its
strategy rather than switching to a different one
(Appendix~\ref{app:refinement_details}). Refining precisely these
below-threshold chains is beneficial because many of them sit close to
acceptable quality yet are treated as negatives, so a small targeted
correction converts them into high-utility positives that amplify the
fine-grained execution signal used in Phase~2 while retaining the
partial progress they already contain. The threshold $\tau = 0.4$
targets chains in the bottom 40th percentile of the utility
distribution (Appendix~\ref{sec:ablation_dataset_characteristics}). This
set is not restricted to repairable chains. Fundamentally flawed chains
also score low, for example $U \approx 0.2$, and therefore fall below
the threshold as well. We attempt refinement on every below-threshold
chain but retain a refined chain only when it strictly improves utility,
so chains whose errors are fundamental rather than local are discarded
when refinement fails to raise their score. The retained refinements
are consequently dominated by chains with fixable errors. Refinement
proceeds iteratively for up to $T_{\max}=5$ rounds with empirically
monotone utility improvement $y^{[s,0]} \xrightarrow{\mathcal{R}}
y^{[s,1]} \xrightarrow{\mathcal{R}} \cdots \xrightarrow{\mathcal{R}}
y^{[s,T]}$, so that $U(x, y^{[s,t+1]}) \geq U(x, y^{[s,t]})$, where
$y^{[s,t]} \in \mathcal{C}^{\text{ref}}_s(x)$ denotes the $t$-th
refinement iteration using strategy $s$, and
$\mathcal{C}^{\text{ref}}_s(x)$ is the set of all refined chains using
strategy $s$ for problem $x$. Refinement terminates when $U(x,
y^{[s,T]}) \geq 0.6$ or $T = T_{\max}$. Empirically, 87.8\% of
refinement attempts succeed within three rounds, with post-refinement
utilities concentrated in $[0.6,1.0]$. Additional details are provided
in Appendix~\ref{app:refinement_details}.

\subsection{High-signal preference pair construction}
\label{sec:pairs}
We construct preference pairs in two stages, each targeting a distinct learning objective. The first stage isolates coarse-grained strategy selection, while the second focuses on fine-grained execution quality within a fixed strategy. Separating these objectives prevents conflicting supervision signals that arise when cross-strategy pairs allocate gradients to ordering suboptimal strategies (formalized in Section~\ref{subsec:two_phase_dpo}).

\paragraph{Strategy selection pairs (Phase 1).}
To maximize signal for strategy selection, we construct preference pairs that directly compare the best strategy against all alternatives. For each problem $x$, we identify the optimal strategy $s^*(x) = \arg\max_{s \in \mathcal{S}} U(x,s)$ from the original chains $\mathcal{C}_0(x)$ and form the set
\begin{multline}
\label{eq:pref_construction}
\mathcal{D}_{\text{Phase1}}(x)
= \Big\{(x, y^{(s^*(x))}, y^{(s)}) : \\
s \in \mathcal{S} \setminus \{s^*(x)\}, \\
U(x, y^{(s^*(x))}) > U(x, y^{(s)}) \Big\},
\end{multline}
where $|\mathcal{D}_{\text{Phase1}}(x)| = K - 1$, yielding exactly one comparison per suboptimal strategy. This construction ensures that every training pair reinforces the optimal strategy, providing large utility margins that clearly signal which reasoning mechanism best matches the problem.

\paragraph{Execution quality pairs (Phase 2).}
After establishing strategy selection, we train on refined chains to
improve execution quality. For each strategy $s \in \mathcal{S}$, let
$\mathcal{C}_s(x) = \mathcal{C}^0_s(x) \cup \mathcal{C}^{\text{ref}}_s(x)$
denote all chains, both original and refined, using strategy $s$ for
problem $x$. We construct margin-stratified pairs that compare
executions \emph{within the same strategy}, that is pairs $(y_i, y_j)$
where both $y_i, y_j \in \mathcal{C}_s(x)$ for some $s$. In practice we
primarily sample pairs $(y_i, y_j)$ with $y_i \in
\mathcal{C}^{\text{ref}}_s(x)$ and $y_j \in \mathcal{C}^0_s(x)$ for the
same strategy $s$, so the preference signal reflects execution
improvements rather than a change in strategy. Sampling uniformly from
all possible pairs instead yields an imbalanced distribution that fails
to leverage subtle comparisons. We therefore use margin-stratified
sampling in Phase~2 to prioritize weak-margin comparisons
(margin $< 0.15$), forcing the model to learn subtle step-level
reasoning improvements. Refined chains naturally populate this regime
due to their concentrated high-utility distribution
(see Section~\ref{subsec:two_phase_dpo}).

\begin{figure*}[t]
    \centering
    \includegraphics[width=0.85\linewidth]{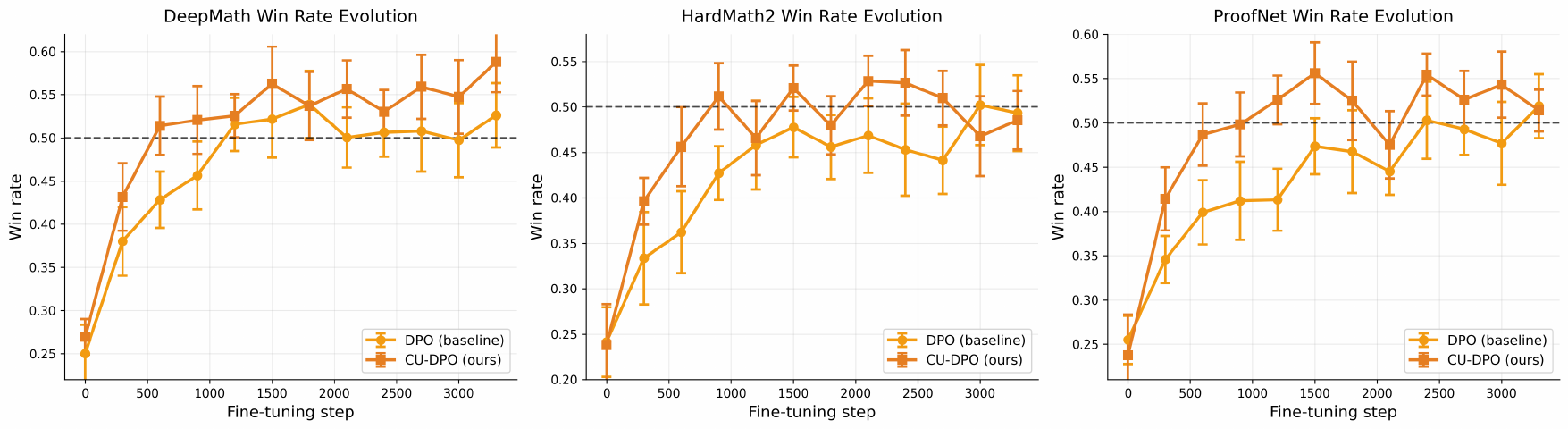}
    \caption{\textbf{Win-rate evolution per preference optimization step.} Win rate versus fine-tuning steps for DeepMath, HARDMath2, and ProofNet. \textbf{\textcolor{MidnightBlue}{CU-}\textcolor{NavyBlue}{DPO}} surpasses the baseline earlier and maintains a consistent advantage, demonstrating improved sample efficiency. Error bars indicate variability across evaluation batches and runs; the dashed line marks the 50\% win-rate threshold (DeepSeek-R1-8B).}
    \label{fig:sample_efficiency_steps}
\end{figure*}

\section{Continuous-Utility Direct Preference Optimization}

Given original chains $C_0(x)$ and refined chains 
$C_{\mathrm{ref}}(x)$ (Section~2.4), we describe our two-phase 
training procedure and its theoretical justification; strategy 
selection and execution refinement require different utility 
distributions and are therefore trained sequentially to avoid 
conflating objectives. Sequential training provides clear supervision. As such, in phase~1, we leverage best-vs-rest pairs from original chains in $\mathcal{C}_0$ to optimize for strategy selection, whereas in Phase~2, we turn to strategy-matched pairs from chain refinements in $\mathcal{C}_{\text{ref}}$ to optimize for execution quality.

\subsection{Bradley-Terry preference model and DPO loss}
\label{sec:cu_dpo_formulation}
We first establish the preference modeling framework underlying our approach.

\paragraph{Bradley-Terry preference model.}
The Bradley-Terry model~\citep{bradley1952rank} converts continuous quality scores into pairwise preference probabilities. For two reasoning chains $y_w$ and $y_\ell$ with continuous utilities $U(x, y_w)$ and $U(x, y_\ell)$, the probability that $y_w$ is preferred over $y_\ell$ is
\begin{equation}
P(y_w \succ y_\ell \mid x) = \sigma(\Delta U) = \frac{1}{1 + \exp(-\Delta U)},
\label{eq:bt-model}
\end{equation}
where $\Delta U = U(x, y_w) - U(x, y_\ell)$ and $\sigma$ is the sigmoid function. Larger utility margins correspond to stronger, more confident preferences, a property we exploit in margin-stratified sampling.

\paragraph{DPO loss.}
DPO~\citep{rafailov2023direct} bypasses explicit reward modeling by reparameterizing the reward function in terms of the policy. The implicit reward is $r_\theta(x,y) = \beta \left[\log \pi_\theta(y \mid x) - \log \pi_{\mathrm{ref}}(y \mid x)\right],$ where $\pi_{\mathrm{ref}}$ is a reference policy and $\beta$ is the inverse temperature controlling regularization strength. Given a preference pair $(y_w, y_\ell)$ where $y_w$ is preferred, DPO minimizes the negative log-likelihood under the Bradley-Terry model:
\begin{multline}
\mathcal{L}_{\text{DPO}}(x; y_w, y_\ell) \\
= -\log \sigma\left(r_\theta(x, y_w) - r_\theta(x, y_\ell)\right).
\label{eq:dpo_loss}
\end{multline}
At optimality, the implicit reward difference matches the utility difference: $r_\theta(x, y_w) - r_\theta(x, y_\ell) = U(x, y_w) - U(x, y_\ell)$ (Lemma~\ref{lem:reward-utility}).

\subsection{Two-phase training}
\label{subsec:two_phase_dpo}
Consider a single-phase DPO objective trained on pairs $(y^{s_i}, y^{s_j})$ where chains are generated using different strategies $s_i \neq s_j$. If the dataset contains pairs $(y^{s_1}_{\text{best}}, y^{s_2}_{\text{good}})$ and $(y^{s_2}_{\text{good}}, y^{s_3}_{\text{bad}})$ where utilities decrease monotonically, jointly optimizing produces ambiguous gradients: the first comparison encourages preferring $s_1$ over $s_2$, while the second encourages $s_2$ over $s_3$, allocating gradient updates to ordering suboptimal strategies rather than consistently reinforcing the optimal choice.

To prevent these conflicting signals, we decompose training into two sequential phases using the preference datasets constructed in Section~\ref{sec:pairs}.

\paragraph{Phase~1: Strategy selection.}
Recall the construction of $\mathcal{D}_{\text{Phase1}}$ from Section~\ref{sec:pairs}. For each problem $x$, we identify the optimal strategy $s^*(x) = \arg\max_{s \in \mathcal{S}} U(x, s)$ and form preference pairs comparing the best strategy against every alternative:
\begin{multline}
\label{eq:phase1_dataset}
\mathcal{D}_{\text{Phase1}} = \Big\{(x, y^{(s^*(x))}, y^{(s_k)}) : \\
s_k \in \mathcal{S} \setminus \{s^*(x)\}, \\
U(x, y^{(s^*(x))}) > U(x, y^{(s_k)}) \Big\}.
\end{multline}
This yields exactly $K-1$ pairs per problem, and every pair directly reinforces the optimal strategy choice with large utility margins (mean 0.35).

\paragraph{Phase~2: Execution quality refinement.}
After Phase~1 stabilizes strategy selection, we train on intra-strategy pairs to improve execution quality. For each strategy $s \in \mathcal{S}$, let $\mathcal{C}_s(x) = \mathcal{C}^0_s(x) \cup \mathcal{C}^{\text{ref}}_s(x)$ be the set of all chains (original and refined) using strategy $s$ on problem $x$. We form strategy-matched pairs:
\begin{multline}
\label{eq:phase2_dataset}
\mathcal{D}_{\text{Phase2}} = \bigcup_{s \in \mathcal{S}} \Big\{(x, y_i, y_j) : y_i, y_j \in \mathcal{C}_s(x), \\
U(x, y_i) > U(x, y_j) \Big\},
\end{multline}
For training in phase 2, we further stratify by utility margin $\Delta U = U(x, y_i) - U(x, y_j)$. We partition pairs into three margin bins: strong ($\Delta U \geq 0.30$, 45\%), medium ($0.15 \leq \Delta U < 0.30$, 30\%), and weak ($\Delta U < 0.15$, 25\%). The mixture emphasizes strong margins for robust learning while including weak-margin pairs that force discrimination of subtle execution differences, which is critical for refined chains whose post-refinement utilities concentrate in $[0.6, 1.0]$. This stratification yields mean margins of 0.18-0.29 (combined 0.244) versus 0.15-0.20 under uniform sampling (Appendix~\ref{app:refinement_details}).

\subsection{Sample efficiency for continuous utilities}
\label{subsec:sample_efficiency}
We formalize the sample complexity advantage of continuous utilities.
Our goal is to learn a reward function
$r:\mathcal{X}\times\mathcal{Y}\to\mathbb{R}$ from a hypothesis class
$\mathcal{H}$ with VC-dimension $d$ such that the induced policy
$\pi_r(y \mid x) \propto \pi_{\text{ref}}(y \mid x) \exp(r(x,y)/\beta)$
maximizes expected utility.

We compare two observation models across $N$ problems. The first is a
\emph{binary preference model} that observes $\ell_{ij} =
\mathbb{1}[U(x,y_i) > U(x,y_j)]$ for sampled pairs, where
$\mathbb{1}[\cdot]$ denotes the indicator function, equal to $1$ when
its argument is true and $0$ otherwise. The second is a \emph{continuous
utility model} that observes the utilities
$\{U(x,y_1),\dots,U(x,y_K)\}$ directly. From
Appendix~\ref{app:binary_lower_bound}, learning from binary pairwise
preferences requires at least
\begin{equation}
    m_{\mathrm{binary}}=\Omega\!\left(NK^2\log K+\frac{d}{\varepsilon^2}\right)
\end{equation}
comparisons to achieve $\varepsilon$-suboptimal expected utility with
probability at least $1-\delta$.

\begin{theorem}[Sample complexity under continuous utilities]
\label{thm:utility-upper-main}
When continuous utilities $\{U(x,y_1),\dots,U(x,y_K)\}$ are observed for each problem, a reward function can be learned using $m_{\mathrm{utility}} = O\!\left(NK + \frac{d}{\varepsilon^2}\right)$ samples.
\end{theorem}
\begin{proof}
See Appendix~\ref{app:utility_upper_bound}.
\end{proof}

\begin{theorem}[Utility-guided sample efficiency gain]
\label{thm:sample-efficiency}
In the regime where $NK^2 \log K \gg d/\varepsilon^2$, continuous utility observation provides a multiplicative efficiency gain of $\frac{m_{\mathrm{binary}}}{m_{\mathrm{utility}}} = \Theta(K \log K).$
\end{theorem}
\begin{proof}
Under the regime assumption, the PAC learning term becomes negligible, yielding $\frac{m_{\mathrm{binary}}}{m_{\mathrm{utility}}} = \frac{\Omega(NK^2 \log K)}{O(NK)} = \Omega(K \log K).$ Matching upper bounds from Appendix~\ref{app:binary_lower_bound} establishes $\Theta(K \log K)$. 
\end{proof}

For $K = 8$, this corresponds to an approximately $24\times$ theoretical
speedup. Empirically, our $5{,}830$ preference pairs represent a
$2.16\times$ reduction from the $12{,}600$ pairs required by exhaustive
enumeration ($450 \times \binom{8}{2}$). This reduction follows from
Phase~1's best-vs-all construction, which uses only $K - 1 = 7$ pairs
per problem instead of the full $\binom{K}{2} = 28$, together with the
fact that binary DPO cannot form informative comparisons between chains
that fall in the same correctness category.

\begin{table*}[b]
\centering
\caption{\textbf{Strategy-problem alignment across domains.} Distribution of optimal strategies across 450 training problems grouped by mathematical domain. Avg.\ utility is the mean utility of the best strategy per problem; avg.\ margin is the mean utility gap between best and worst strategies, indicating problem-strategy sensitivity.}
\label{tab:domain_strategy}
\small
\setlength{\tabcolsep}{6pt}
\begin{tabular}{lrllcc}
\toprule
\textbf{Domain} & \textbf{N} & \textbf{Best strategy} & \textbf{Second-best} & \textbf{Avg.\ utility} & \textbf{Avg.\ margin} \\
\midrule
Calculus & 200 & \texttt{step\_by\_step} (25\%) & \texttt{alternative} (16\%) & 0.763 & 0.32 \\
Proof    & 135 & \texttt{direct} (27\%)         & \texttt{verification} (16\%) & 0.768 & 0.27 \\
Other    &  75 & \texttt{step\_by\_step} (19\%) & \texttt{verification} (17\%) & 0.819 & 0.30 \\
Algebra  &  22 & \texttt{algebraic} (27\%)      & \texttt{backward} (18\%)     & 0.755 & 0.35 \\
Analysis &  18 & \texttt{alternative} (33\%)    & \texttt{direct} (22\%)       & 0.847 & 0.28 \\
\midrule
\rowcolor{gray!10} \textbf{Overall} & \textbf{450} & — & — & \textbf{0.777} & \textbf{0.30} \\
\bottomrule
\end{tabular}
\end{table*}

\begin{figure}[t!]
    \centering
    \includegraphics[width=0.7\linewidth]{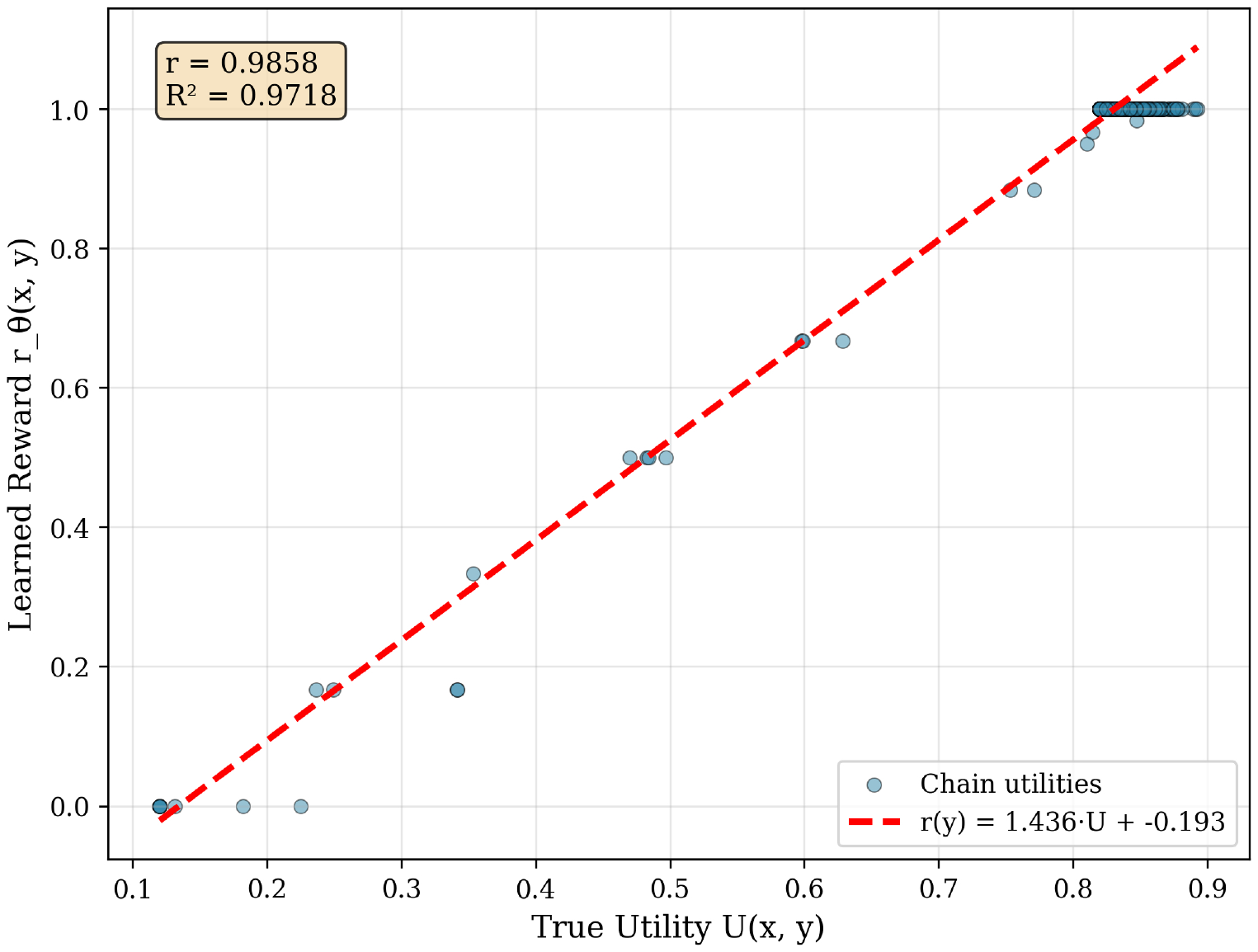}
    \caption{\textbf{Empirical evidence for reward--utility alignment.}
    Learned implicit reward $r_\theta(x,y)=\beta(\log\pi_\theta-\log\pi_{\mathrm{ref}})$ aligns linearly with utility $U(x,y)$, supporting the relation $r_\theta(x,y)=U(x,y)+c(x)$ and Theorem~\ref{thm:dpo-convergence}.}
    \label{fig:reward_utility_alignment}
\end{figure}

\begin{table}[b]
\centering
\caption{Comparison with step-level supervision on in-distribution 
benchmarks (Pass@1, \%) for DeepSeek-R1~(8B).}
\label{tab:step_comparison}
\small
\setlength{\tabcolsep}{6pt}
\renewcommand{\arraystretch}{1.1}
\begin{tabular}{lcccc}
\toprule
\textbf{Setting} & \textbf{DM} & \textbf{HM2} & \textbf{PN} 
& \textbf{$\Delta$} \\
\midrule
Baseline                  & 64.2 & 42.1 & 38.5 & --- \\
+ Step-DPO                & 66.8 & 44.3 & 40.1 & +2.4 \\
\rowcolor{highlightblue}
+ \textbf{\textcolor{MidnightBlue}{CU-}\textcolor{NavyBlue}{DPO}} 
                          & 69.8 & 48.7 & 44.2 & +5.0 \\
\rowcolor{highlightblue}
+ \textbf{\textcolor{MidnightBlue}{CU-}\textcolor{NavyBlue}{DPO}} 
+ Step-DPO                
                          & \textbf{71.9} & \textbf{51.2} 
                          & \textbf{46.3} & \textbf{+7.1} \\
\bottomrule
\end{tabular}
\end{table}

\begin{table}[b]
\centering
\caption{\textbf{Out-of-distribution strategy selection.} DeepSeek-R1-Distill-Llama-8B evaluated on benchmarks unseen during training to test zero-shot transfer of learned strategy selection capabilities.}
\label{tab:strategy_selection_ood}
\small
\renewcommand{\arraystretch}{0.92}
\begin{tabular}{lccc}
\toprule
\textbf{Benchmark} & \textbf{Baseline} & \textbf{+ \textbf{\textcolor{MidnightBlue}{CU-}\textcolor{NavyBlue}{DPO}}} & \textbf{$\Delta$} \\
\midrule
GSM8K (Acc., \%) & 48.3 & \textbf{62.7} & \textcolor{darkgreen}{\textbf{+14.4}} \\
MATH-500 (Acc., \%) & 42.6 & \textbf{65.9} & \textcolor{darkgreen}{\textbf{+23.3}} \\
U-MATH (Acc., \%) & 31.2 & \textbf{49.8} & \textcolor{darkgreen}{\textbf{+18.6}} \\
\bottomrule
\end{tabular}
\end{table}

\begin{table*}[t]
\centering
\caption{In-distribution and out-of-distribution reasoning accuracy (Pass@1, \%) after two-phase \textbf{\textcolor{MidnightBlue}{CU-}\textcolor{NavyBlue}{DPO}} training. $\Delta$ shows Phase~2 improvement over Phase~1 within each split.}
\label{tab:reasoning_cudpo}
\small
\setlength{\tabcolsep}{5pt}
\renewcommand{\arraystretch}{0.92}
\begin{tabular}{ll ccc c ccc c}
\toprule
& & \multicolumn{4}{c}{\textbf{In-distribution}} & \multicolumn{4}{c}{\textbf{Out-of-distribution}} \\
\cmidrule(lr){3-6}\cmidrule(lr){7-10}
\textbf{Base model} & \textbf{Setting} & \textbf{DM} & \textbf{HM2} & \textbf{PN} & \textbf{$\Delta$} & \textbf{GSM} & \textbf{M500} & \textbf{UM} & \textbf{$\Delta$} \\
\midrule
\multirow{2}{*}{\shortstack[l]{DeepSeek-R1\\(8B)}}
    & Baseline & 64.2 & 42.1 & 38.5 & --- & 69.8 & 59.5 & 6.9 & --- \\
    & \cellcolor{highlightblue}+ \textbf{\textcolor{MidnightBlue}{CU-}\textcolor{NavyBlue}{DPO}} & \cellcolor{highlightblue}\textbf{69.8} & \cellcolor{highlightblue}\textbf{48.7} & \cellcolor{highlightblue}\textbf{44.2} & \cellcolor{highlightblue}\textbf{+5.0} & \cellcolor{highlightblue}\textbf{72.3} & \cellcolor{highlightblue}\textbf{64.8} & \cellcolor{highlightblue}\textbf{8.2} & \cellcolor{highlightblue}\textbf{+1.8} \\
\midrule
\multirow{2}{*}{\shortstack[l]{Qwen3\\(8B)}}
    & Baseline & 61.5 & 39.4 & 35.2 & --- & 66.4 & 56.2 & 5.8 & --- \\
    & \cellcolor{highlightblue}+ \textbf{\textcolor{MidnightBlue}{CU-}\textcolor{NavyBlue}{DPO}} & \cellcolor{highlightblue}\textbf{67.2} & \cellcolor{highlightblue}\textbf{45.9} & \cellcolor{highlightblue}\textbf{41.3} & \cellcolor{highlightblue}\textbf{+5.4} & \cellcolor{highlightblue}\textbf{68.9} & \cellcolor{highlightblue}\textbf{61.3} & \cellcolor{highlightblue}\textbf{7.1} & \cellcolor{highlightblue}\textbf{+2.2} \\
\midrule
\multirow{2}{*}{\shortstack[l]{Llama-3.1\\(8B)}}
    & Baseline & 28.4 & 18.2 & 14.8 & --- & 58.2 & 48.3 & 4.2 & --- \\
    & \cellcolor{highlightblue}+ \textbf{\textcolor{MidnightBlue}{CU-}\textcolor{NavyBlue}{DPO}} & \cellcolor{highlightblue}\textbf{34.8} & \cellcolor{highlightblue}\textbf{24.7} & \cellcolor{highlightblue}\textbf{20.9} & \cellcolor{highlightblue}\textbf{+5.6} & \cellcolor{highlightblue}\textbf{60.7} & \cellcolor{highlightblue}\textbf{52.9} & \cellcolor{highlightred}3.8 & \cellcolor{highlightblue}\textbf{+1.9} \\
\midrule
\multirow{2}{*}{\shortstack[l]{Gemma-2\\(9B)}}
    & Baseline & 26.1 & 15.5 & 12.1 & --- & 52.7 & 42.8 & 3.1 & --- \\
    & \cellcolor{highlightblue}+ \textbf{\textcolor{MidnightBlue}{CU-}\textcolor{NavyBlue}{DPO}} & \cellcolor{highlightblue}\textbf{32.4} & \cellcolor{highlightblue}\textbf{21.8} & \cellcolor{highlightblue}\textbf{17.6} & \cellcolor{highlightblue}\textbf{+5.6} & \cellcolor{highlightblue}\textbf{55.4} & \cellcolor{highlightblue}\textbf{47.6} & \cellcolor{highlightblue}\textbf{3.4} & \cellcolor{highlightblue}\textbf{+2.2} \\
\midrule
\multirow{2}{*}{\shortstack[l]{Mistral-8x7B\\~}}
    & Baseline & 22.8 & 12.4 & 9.5 & --- & 48.5 & 38.6 & 2.4 & --- \\
    & \cellcolor{highlightblue}+ \textbf{\textcolor{MidnightBlue}{CU-}\textcolor{NavyBlue}{DPO}} & \cellcolor{highlightblue}\textbf{28.9} & \cellcolor{highlightblue}\textbf{18.3} & \cellcolor{highlightblue}\textbf{14.8} & \cellcolor{highlightblue}\textbf{+5.5} & \cellcolor{highlightblue}\textbf{50.2} & \cellcolor{highlightred}37.9 & \cellcolor{highlightred}2.1 & \cellcolor{highlightblue}\textbf{+0.1} \\
\midrule
\multirow{2}{*}{\shortstack[l]{Llama-3\\(8B)}}
    & Baseline & 18.3 & 8.1 & 6.4 & --- & 44.2 & 34.1 & 1.8 & --- \\
    & \cellcolor{highlightblue}+ \textbf{\textcolor{MidnightBlue}{CU-}\textcolor{NavyBlue}{DPO}} & \cellcolor{highlightblue}\textbf{24.7} & \cellcolor{highlightblue}\textbf{14.2} & \cellcolor{highlightblue}\textbf{11.8} & \cellcolor{highlightblue}\textbf{+5.6} & \cellcolor{highlightblue}\textbf{46.8} & \cellcolor{highlightblue}\textbf{37.4} & \cellcolor{highlightred}1.5 & \cellcolor{highlightblue}\textbf{+1.6} \\
\midrule
\multirow{2}{*}{\shortstack[l]{Mistral-7B\\~}}
    & Baseline & 14.2 & 5.3 & 3.8 & --- & 38.6 & 28.4 & 1.2 & --- \\
    & \cellcolor{highlightblue}+ \textbf{\textcolor{MidnightBlue}{CU-}\textcolor{NavyBlue}{DPO}} & \cellcolor{highlightblue}\textbf{20.3} & \cellcolor{highlightblue}\textbf{11.7} & \cellcolor{highlightblue}\textbf{9.2} & \cellcolor{highlightblue}\textbf{+5.7} & \cellcolor{highlightred}37.9 & \cellcolor{highlightblue}\textbf{30.8} & \cellcolor{highlightred}1.0 & \cellcolor{highlightblue}\textbf{+0.4} \\
\bottomrule
\end{tabular}
\end{table*}

\begin{table}[t]
\centering
\caption{\textbf{Empirical sample efficiency (data scaling).} Win-rate (\%) on held-out preference comparisons as a function of training data fraction.}
\label{tab:sample_efficiency_scaling}
\footnotesize
\setlength{\tabcolsep}{4pt}
\renewcommand{\arraystretch}{0.92}
\begin{tabular}{llccc}
\toprule
\textbf{Base model} & \textbf{Train \%} & \textbf{Binary DPO} & \textbf{\textbf{\textcolor{MidnightBlue}{CU-}\textcolor{NavyBlue}{DPO}}} & \textbf{$\Delta$} \\
\midrule
\multirow{4}{*}{\shortstack[l]{DeepSeek-R1\\(8B)}} 
& 25\%  & \cellcolor{highlightred}46.8 & 45.6 & \textcolor{darkred}{$-1.2$} \\
& 50\%  & 50.7 & \cellcolor{highlightblue}\textbf{52.6} & \textcolor{darkgreen}{+1.9} \\
& 75\%  & 53.9 & \cellcolor{highlightblue}\textbf{56.1} & \textcolor{darkgreen}{+2.2} \\
& 100\% & 56.4 & \cellcolor{highlightblue}\textbf{58.9} & \textcolor{darkgreen}{+2.5} \\
\midrule
\multirow{4}{*}{\shortstack[l]{Qwen3\\(8B)}} 
& 25\%  & \cellcolor{highlightred}45.9 & 45.1 & \textcolor{darkred}{$-0.8$} \\
& 50\%  & 49.6 & \cellcolor{highlightblue}\textbf{51.4} & \textcolor{darkgreen}{+1.8} \\
& 75\%  & 52.8 & \cellcolor{highlightblue}\textbf{55.0} & \textcolor{darkgreen}{+2.2} \\
& 100\% & 55.3 & \cellcolor{highlightblue}\textbf{57.6} & \textcolor{darkgreen}{+2.3} \\
\midrule
\multirow{4}{*}{\shortstack[l]{Gemma-2\\(9B)}} 
& 25\%  & \cellcolor{highlightred}44.1 & 43.3 & \textcolor{darkred}{$-0.8$} \\
& 50\%  & 48.2 & \cellcolor{highlightblue}\textbf{50.0} & \textcolor{darkgreen}{+1.8} \\
& 75\%  & 51.6 & \cellcolor{highlightblue}\textbf{53.7} & \textcolor{darkgreen}{+2.1} \\
& 100\% & 54.0 & \cellcolor{highlightblue}\textbf{56.2} & \textcolor{darkgreen}{+2.2} \\
\bottomrule
\end{tabular}
\end{table}

\begin{table}[t]
\centering
\caption{\textbf{In-distribution strategy selection.} Comparison between prompt-only baselines and the same base model after \textbf{\textcolor{MidnightBlue}{CU-}\textcolor{NavyBlue}{DPO}} Phase~1 fine-tuning, isolating improvements in \emph{strategy selection} rather than reasoning accuracy.}
\label{tab:strategy_selection_id}
\small
\setlength{\tabcolsep}{5pt}
\renewcommand{\arraystretch}{0.92}
\begin{tabular}{llccc}
\toprule
\textbf{Base model} & \textbf{Setting} & \textbf{Acc.} & \textbf{Top-3} & \textbf{Spear.\ $\rho$} \\
\midrule
\multirow{2}{*}{\shortstack[l]{DeepSeek-Math\\(7B)}} 
    & Baseline & 0.46 & 0.76 & 0.41 \\
    & \cellcolor{highlightblue}+ \textbf{\textcolor{MidnightBlue}{CU-}\textcolor{NavyBlue}{DPO}} & \cellcolor{highlightblue}\textbf{0.78} & \cellcolor{highlightblue}\textbf{0.94} & \cellcolor{highlightblue}\textbf{0.71} \\
\midrule
\multirow{2}{*}{\shortstack[l]{DeepSeek-R1\\(8B)}}   
    & Baseline & 0.44 & 0.74 & 0.39 \\
    & \cellcolor{highlightblue}+ \textbf{\textcolor{MidnightBlue}{CU-}\textcolor{NavyBlue}{DPO}} & \cellcolor{highlightblue}\textbf{0.76} & \cellcolor{highlightblue}\textbf{0.93} & \cellcolor{highlightblue}\textbf{0.69} \\
\midrule
\multirow{2}{*}{\shortstack[l]{Gemma-2\\(9B)}}       
    & Baseline & 0.39 & 0.69 & 0.33 \\
    & \cellcolor{highlightblue}+ \textbf{\textcolor{MidnightBlue}{CU-}\textcolor{NavyBlue}{DPO}} & \cellcolor{highlightblue}\textbf{0.71} & \cellcolor{highlightblue}\textbf{0.91} & \cellcolor{highlightblue}\textbf{0.65} \\
\midrule
\multirow{2}{*}{\shortstack[l]{Mistral-7B\\~}}         
    & Baseline & 0.35 & 0.66 & 0.31 \\
    & \cellcolor{highlightblue}+ \textbf{\textcolor{MidnightBlue}{CU-}\textcolor{NavyBlue}{DPO}} & \cellcolor{highlightblue}\textbf{0.68} & \cellcolor{highlightblue}\textbf{0.89} & \cellcolor{highlightblue}\textbf{0.62} \\
\midrule
\multirow{2}{*}{\shortstack[l]{Mistral-8x7B\\~}}
    & Baseline & 0.44 & 0.74 & 0.39 \\
    & \cellcolor{highlightblue}+ \textbf{\textcolor{MidnightBlue}{CU-}\textcolor{NavyBlue}{DPO}} & \cellcolor{highlightblue}\textbf{0.77} & \cellcolor{highlightblue}\textbf{0.93} & \cellcolor{highlightblue}\textbf{0.70} \\
\midrule
\multirow{2}{*}{\shortstack[l]{Qwen3\\(8B)}}         
    & Baseline & 0.43 & 0.73 & 0.38 \\
    & \cellcolor{highlightblue}+ \textbf{\textcolor{MidnightBlue}{CU-}\textcolor{NavyBlue}{DPO}} & \cellcolor{highlightblue}\textbf{0.75} & \cellcolor{highlightblue}\textbf{0.92} & \cellcolor{highlightblue}\textbf{0.68} \\
\midrule
\multirow{2}{*}{\shortstack[l]{Llama-3\\(8B)}}       
    & Baseline & 0.40 & 0.70 & 0.35 \\
    & \cellcolor{highlightblue}+ \textbf{\textcolor{MidnightBlue}{CU-}\textcolor{NavyBlue}{DPO}} & \cellcolor{highlightblue}\textbf{0.73} & \cellcolor{highlightblue}\textbf{0.91} & \cellcolor{highlightblue}\textbf{0.66} \\
\midrule
\multirow{2}{*}{\shortstack[l]{Llama-3.1\\(8B)}}     
    & Baseline & 0.42 & 0.72 & 0.37 \\
    & \cellcolor{highlightblue}+ \textbf{\textcolor{MidnightBlue}{CU-}\textcolor{NavyBlue}{DPO}} & \cellcolor{highlightblue}\textbf{0.74} & \cellcolor{highlightblue}\textbf{0.92} & \cellcolor{highlightblue}\textbf{0.67} \\
\bottomrule
\end{tabular}
\end{table}

\begin{table}[t]
\centering
\caption{Code generation results. Pass@1 (\%) on HumanEval for DeepSeek-R1~(8B) and Qwen3~(8B).
}
\label{tab:humaneval_cudpo}
\small
\setlength{\tabcolsep}{6pt}
\renewcommand{\arraystretch}{1.1}
\begin{tabular}{lcc}
\toprule
\textbf{Method} & \textbf{DeepSeek-R1 (8B)} & \textbf{Qwen3 (8B)} \\
\midrule
Baseline (greedy) & 74.3 & 71.8 \\
Binary DPO        & 76.1 & 73.5 \\
Best-of-8         & \textbf{79.8} & 77.1 \\
\rowcolor{highlightblue}
\textbf{\textcolor{MidnightBlue}{CU-}\textcolor{NavyBlue}{DPO}} (ours)
                  & \underline{79.2} & \underline{77.4} \\
\bottomrule
\end{tabular}
\end{table}

\begin{table}[t]
\centering
\caption{Causal reasoning results. Accuracy~(\%) on $\alpha$-NLI 
(Abductive NLI) for DeepSeek-R1~(8B) and Qwen3~(8B).
}
\label{tab:causal_cudpo}
\small
\setlength{\tabcolsep}{6pt}
\renewcommand{\arraystretch}{1.1}
\begin{tabular}{lcc}
\toprule
\textbf{Method} 
& \textbf{DeepSeek-R1 (8B)} 
& \textbf{Qwen3 (8B)} \\
\midrule
Baseline (greedy)  & 78.4 & 76.1 \\
Binary DPO         & 80.1 & 77.8 \\
Best-of-8          & \textbf{83.9} & \textbf{81.6} \\
\rowcolor{highlightblue}
\textbf{\textcolor{MidnightBlue}{CU-}\textcolor{NavyBlue}{DPO}} (ours) 
                   & \underline{83.3} & \underline{80.8} \\
\bottomrule
\end{tabular}
\end{table}

\subsection{Convergence to utility-maximizing policy}
\label{subsec:convergence}

We show that DPO with utility-weighted preferences converges to the entropy-regularized utility-maximizing policy.

\begin{lemma}[Reward-utility alignment]
\label{lem:reward-utility}
At the global minimum of the DPO objective (Eq.~\ref{eq:dpo_loss}), the implicit reward satisfies $r_\theta(x,y_i) - r_\theta(x,y_j) = U(x,y_i) - U(x,y_j)$ for all pairs $(i,j)$, implying $r_\theta(x,y) = U(x,y) + c(x)$ for some problem-dependent constant $c(x)$.
\end{lemma}
\begin{proof}[Proof sketch]
The DPO objective minimizes negative log-likelihood under the Bradley-Terry model. At the optimum, the model's predicted preference probabilities must match the true preference probabilities: $\sigma(\Delta r_{ij}) = \sigma(\Delta U_{ij})$. Since $\sigma$ is strictly monotone, this implies $\Delta r_{ij} = \Delta U_{ij}$ for all pairs. These pairwise constraints uniquely determine $r_\theta(x,y)$ up to an additive constant $c(x)$. The full proof is in Appendix~\ref{app:reward_utility_proof}.
\end{proof}

\begin{corollary}[DPO converges to utility-maximizing policy]
\label{thm:dpo-convergence}
The optimal policy $\pi_\theta^*$ learned by DPO with utility-based preferences (Eq.~\ref{eq:bt-model}) satisfies
\begin{equation}
\pi_\theta^*(y \mid x) \propto \pi_{\mathrm{ref}}(y \mid x) \exp\left(\frac{U(x,y)}{\beta}\right),
\end{equation}
which is the entropy-regularized utility-maximizing policy.
\end{corollary}
\begin{proof}
By Lemma~\ref{lem:reward-utility}, at optimality $r_\theta(x,y) = U(x,y) + c(x)$. The DPO closed-form solution~\citep{rafailov2023direct} gives $\pi_\theta^*(y \mid x) \propto \pi_{\mathrm{ref}}(y \mid x) \exp(r_\theta(x,y)/\beta)$. Substituting and noting that $\exp(c(x)/\beta)$ is constant in $y$ yields the result.
\end{proof}

\begin{remark}[Problem-dependent constant invariance]
\label{rem:constant-invariance}
The problem-dependent constant $c(x)$ does not affect policy ranking: for any two chains $y_i, y_j$ given problem $x$, the probability ratio satisfies $\frac{\pi_\theta^*(y_i \mid x)}{\pi_\theta^*(y_j \mid x)} = \frac{\pi_{\mathrm{ref}}(y_i \mid x)}{\pi_{\mathrm{ref}}(y_j \mid x)} \exp\left(\frac{U(x,y_i) - U(x,y_j)}{\beta}\right),$ where $c(x)$ cancels. Thus, even if $c(x)$ varies widely across problems, the learned policy correctly ranks chains within each problem by utility differences $\Delta U$, which is the training objective. Cross-problem comparisons are not required for generation.
\end{remark}

\noindent Figure~\ref{fig:reward_utility_alignment} empirically validates Lemma~\ref{lem:reward-utility}: the learned implicit reward exhibits strong linear correlation ($R^2=0.97$) with ground-truth utilities, confirming that DPO recovers the utility function up to a problem-dependent constant.

\section{Experiments} \label{sec:Experiments} We evaluate \textbf{\textcolor{MidnightBlue}{CU-}\textcolor{NavyBlue}{DPO}} across three domains to demonstrate both task performance and domain generalization. For mathematical reasoning, we use three in-distribution datasets~(DeepMath, HARDMath2, ProofNet) and three out-of-distribution datasets~(GSM8K, MATH-500, U-MATH), with 450 training problems drawn from the in-distribution benchmarks and held-out test sets of 150 problems per dataset to ensure no data leakage. We further evaluate on code generation~(HumanEval) and causal reasoning~($\alpha$-NLI) to confirm that continuous utility supervision generalizes beyond mathematical reasoning to any domain admitting a strategy portfolio and decomposable utility signal. All experiments were conducted on an AWS g4dn.12xlarge instance, and all primary results are reported as mean $\pm$ standard deviation across three independent runs with different random seeds; paired $t$-tests with Bonferroni correction are reported in Appendix~\ref{appendix:stats}.

\subsection{Problem-strategy correlation analysis}
\label{sec:strategy_correlation}
Table~\ref{tab:domain_strategy} confirms that optimal strategies shift 
systematically across mathematical domains with 
no single strategy dominating (mean margin 0.30), 
validating that strategy selection is problem-
dependent and motivating the portfolio approach.

\subsection{Strategy Selection}
We compare against prompt-selected baselines without fine-tuning, 
evaluating accuracy, Top-3 selection rate, and Spearman's $\rho$ 
(rank correlation over all strategies). Additional comparisons with 
preference learning and inference-time baselines are in 
Appendix~\ref{app:Baselines}.

\noindent\textbf{In-distribution (Table~\ref{tab:strategy_selection_id}).}
Phase~1 separates strategy choice from execution. Prompt-only 
baselines frequently select mismatched strategies despite seeing 
all prompts, while 
\textbf{\textcolor{MidnightBlue}{CU-}\textcolor{NavyBlue}{DPO}} 
Phase~1 shifts selections toward better-matched strategies and 
improves Spearman's $\rho$, indicating more consistent strategy 
ordering.

\noindent\textbf{Out-of-distribution (Table~\ref{tab:strategy_selection_ood}).}
The same behavior transfers to unseen benchmarks, with larger gains 
when mismatch is most harmful. Using DeepSeek-R1-Distill-Llama-8B 
trained on DeepMath/HARDMath2/ProofNet and evaluated zero-shot on 
GSM8K/MATH-500/U-MATH, 
\textbf{\textcolor{MidnightBlue}{CU-}\textcolor{NavyBlue}{DPO}} 
improves strategy choice consistently under distribution shift.

\subsection{Execution Refinement and Downstream Reasoning}
\label{sec:phase2_downstream}

\noindent\textbf{Phase~2 improves execution quality and downstream 
reasoning.} After Phase~1 learns \emph{which strategy to invoke}, 
Phase~2 improves \emph{how well} it is executed using preference 
pairs between original and self-refined chains. Because winner and 
loser usually share the same high-level approach, supervision focuses 
on localized corrections such as fixing algebraic errors, improving 
justifications, and adding verification, producing cleaner gradients 
than cross-strategy comparisons. Table~\ref{tab:reasoning_cudpo} 
shows consistent in-distribution gains; for example, DeepSeek-R1 on 
DeepMath improves from 66.1\%~(Phase~1 only) to 69.8\% after 
Phase~2. Table~\ref{tab:reasoning_cudpo} further shows mostly 
positive out-of-distribution transfer, with modest regressions on 
ultra-hard problems requiring major insight jumps beyond local 
refinement.

\noindent\textbf{Comparison with step-level supervision.}
Table~\ref{tab:step_comparison} shows that CU-DPO outperforms 
Step-DPO by $+3.0$ points on DeepMath and $+4.4$ on HARDMath2, 
indicating that strategy selection is the primary bottleneck. 
Combining both methods yields additional gains across all datasets, 
demonstrating complementary roles in horizontal strategy selection 
and vertical step-level execution refinement. Full analysis is 
provided in Appendix~\ref{appendix:step_comparison} and further ablation studies are present in Appendix~\ref{app:Baselines}.

\subsection{Data scaling}
Table~\ref{tab:sample_efficiency_scaling} evaluates sample efficiency by training binary DPO and \textbf{\textcolor{MidnightBlue}{CU-}\textcolor{NavyBlue}{DPO}} on 25\%, 50\%, 75\%, and 100\% of Phase~2 data, measuring win rate on held-out preference comparisons. We use win rate rather than Pass@1 to isolate optimization efficiency from generation stochasticity.

Under data scarcity (25\%), binary DPO holds a slight advantage: discrete labels provide stronger gradients when sample size is insufficient to accurately calibrate continuous utilities. However, \textbf{\textcolor{MidnightBlue}{CU-}\textcolor{NavyBlue}{DPO}} overtakes binary DPO at 50\% data, with the gap widening monotonically thereafter. This crossover validates our hypothesis: continuous utilities encode richer preference structure but require sufficient data to exploit fine-grained signals. Further ablation studies on comparisons with baseline methods are present in Appendix~\ref{app:Baselines}.

\subsection{Domain Generalization}
\label{sec:domain_gen}

We evaluate CU-DPO on code generation and causal reasoning to test 
whether continuous utility supervision generalizes beyond 
mathematical reasoning. For code generation, we use 
\textbf{HumanEval}~\cite{chen2021codex} with four strategies: 
\textit{direct implementation}, \textit{test-driven development}, 
\textit{modular decomposition}, and \textit{defensive programming}, 
with utilities decomposed into correctness, efficiency, and solution 
clarity. For causal reasoning, we use 
$\alpha$-\textbf{NLI}~\cite{bhagavatula2020abductive} with six 
strategies: \textit{mechanistic}, \textit{counterfactual}, 
\textit{statistical}, \textit{eliminative}, \textit{analogical}, 
and \textit{temporal} reasoning, with utilities decomposed into 
correctness, reasoning validity, and causal coherence.

Tables~\ref{tab:humaneval_cudpo} and~\ref{tab:causal_cudpo} report 
results for DeepSeek-R1~(8B) and Qwen3~(8B). On HumanEval, CU-DPO 
achieves 79.2\% Pass@1, closely matching Best-of-8~(79.8\%) at 
$1\times$ inference cost while improving over binary DPO by 
$+3.1$ points; Qwen3~(8B) shows a similar $+2.9$ point gain. On 
$\alpha$-NLI, CU-DPO improves over binary DPO by $+3.2$ and 
$+3.0$ points respectively. Although gains are smaller than in 
mathematical reasoning due to lower inter-strategy utility 
separation, CU-DPO consistently outperforms binary DPO and matches 
inference-heavy baselines at constant cost. Full analyses are 
provided in Appendices~\ref{appendix:humaneval_strategy} 
and~\ref{appendix:causal}.

\section{Conclusion}
We introduced CU-DPO, a framework for reasoning alignment that guides
preference optimization with continuous, LLM-judged utilities rather
than binary labels. By combining continuous utility scoring, structured
preference pair construction, and two-phase training, CU-DPO improves
the sample efficiency of preference-based alignment, requiring far fewer
comparisons than binary DPO to recover the same quality ranking. Unlike
binary supervision, continuous utilities capture finer-grained
information such as partial correctness and reasoning coherence, while
separating strategy selection from execution refinement avoids
conflicting optimization signals. Beyond mathematical reasoning, we show
that CU-DPO generalizes to code generation and causal reasoning,
demonstrating applicability to domains with multiple cognitively
distinct reasoning strategies and decomposable quality metrics.

Future directions include autonomous strategy portfolio generation via
LLM conditioning, adaptive phase scheduling exploiting Phase~1's faster
convergence, and error-type-specialized refinement operators for
improved correction efficiency.

\section*{Limitations}
While CU-DPO demonstrates significant sample efficiency and performance gains across multiple domains, several limitations need to be acknowledged. (a) The framework heavily relies on the quality and calibration of the offline LLM judge used to decompose and score continuous utilities. Although our cross-judge validation indicates high agreement among top-tier models, utilizing smaller or poorly calibrated judges could introduce stochastic noise that increases the sample complexity required to accurately recover latent rankings. (b) The progressive refinement operator $\mathcal{R}$ used in Phase 2 incurs an additional computational overhead during the offline dataset generation phase, particularly for highly complex tasks where models may require multiple iterative rounds to reach the utility success threshold. (c) As observed in our OOD mathematical evaluations, CU-DPO exhibits modest performance regressions on ultra-hard problems that demand substantial, non-local cognitive insight jumps, which suggests that localized, intra-strategy execution refinement is less effective when a problem requires an unrepresented paradigm shift.

\section*{Impact Statement}
This paper presents work whose goal is to advance the field of machine learning. There are many potential societal consequences of our work, none of which we feel must be specifically highlighted here.

\section*{Funding Statement}
This work was supported in part by ONR Grant N00014-22-1-2110 and the Stanford Institute for Human-Centered Artificial Intelligence (HAI). EBF is a Biohub, San Francisco, Investigator.

\bibliography{custom}

\appendix
\onecolumn

\section{Related Work}
\textbf{Mathematical reasoning with LLMs.}
Prior work improves mathematical reasoning through instruction tuning on curated corpora~\citep{hendrycks2021measuring, cobbe2021training, luo2023wizardmath, yue2023mammoth, yu2023metamath}, tool-augmented reasoning~\citep{gou2023critic, polu2020generative, gou2023tora, schick2023toolformer}, and self-improvement via iterative refinement~\citep{zelikman2022star, huang2022large, gulcehre2023reinforced}. Large-scale supervised fine-tuning~\citep{toshniwal2024openmath, shao2024deepseekmath, lewkowycz2022minerva} achieves strong performance but requires massive trajectory collections and treats strategy choice as manual curriculum design.

\textbf{DPO for reasoning.}
Recent work applies DPO~\citep{rafailov2023direct} to reasoning using binary correctness signals~\citep{wang2024mathshepherd, li2024xwin, yuan2024reft, hong2024orpo, tunstall2023zephyr}, step-level reward modeling~\citep{lightman2023lets, uesato2022solving}, and RLHF-based iterative refinement~\citep{gulcehre2023reinforced, yuan2024selfrewarding, dong2023raft}. These approaches collapse solution quality into binary preferences and suffer from mode collapse when strategies succeed on disjoint problem subsets~\citep{wu2024fine, kirk2024understanding}.

\textbf{Reasoning trajectory diversity.}
Multiple works recognize that problems admit diverse solution strategies, exploring trajectory selection via correctness filtering~\citep{zelikman2022star, singh2023beyond, hosseini2024vstar}, majority voting~\citep{wang2023selfconsistency, chen2023universal, li2023making}, search-based exploration~\citep{yao2023tree, besta2024graph, hao2023reasoning, sel2023algorithm}, and outcome or process supervision~\citep{uesato2022solving, lightman2023lets, wang2024math}. However, these methods apply uniform thresholds across strategies, treat all paths equivalently during aggregation, require handcrafted heuristics, or need expensive human annotation.

\section{Dataset Characteristics and Signal Composition}
\label{sec:ablation_dataset_characteristics}

Our method relies on a deliberately structured two-phase preference dataset whose \emph{signal} is controlled by three design knobs: (i) \emph{within-problem utility diversity} induced by sampling $K=8$ cognitive strategies, (ii) \emph{continuous utility margins} that determine how informative a comparison is, and (iii) \emph{selective refinement} that injects fine-grained structure for reasoning quality optimization. This section characterizes the resulting datasets and isolates which properties enable effective two-phase learning.

\paragraph{Two-phase dataset pipeline.}
For each problem $x$, we sample $K=8$ chains $\{y_k\}_{k=1}^K$ using fixed prompts corresponding to distinct mechanisms (direct, step-by-step, backwards, alternative, verification, algebraic, numerical, conceptual). An LLM judge assigns continuous utilities $U(x,y_k)\in[0,1]$. \textbf{Phase 1} constructs strategy selection pairs from original chains using a ``best vs all'' acquisition rule that reinforces optimal strategy choice. \textbf{Phase 2} applies selective refinement to low-utility chains, producing an auxiliary pool of refined candidates, and constructs reasoning quality pairs from the merged pool using margin-stratified sampling to enable both coarse and fine-grained supervision.

\textit{Strategy selection statistics.}
Table~\ref{tab:phase1_statistics} reports statistics for Phase 1  pairs across datasets. Each problem generates exactly 7 pairs comparing the optimal strategy against all others, yielding 3,150 total pairs from 450 problems. Mean margins of 0.214 reflect coarse-grained distinctions, with substantial spread (std 0.201) capturing varying difficulty across strategy-problem alignments. HARDMath2 exhibits the largest margins (mean 0.320), indicating clearer strategy preferences, while ProofNet shows smaller margins (mean 0.119), suggesting multiple viable strategies for proof-oriented problems. Critically, \emph{every pair in Phase 1 includes the optimal strategy}, eliminating conflicting supervision signals that plagued our initial single-phase approach.

\begin{table}[h]
\centering
\caption{\textbf{Strategy selection dataset statistics.} High margins in HARDMath2 indicate clear strategy dominance; lower margins in ProofNet suggest strategy interchangeability.}
\label{tab:phase1_statistics}
\small
\setlength{\tabcolsep}{10pt}
\begin{tabular}{lcccccc}
\toprule
\textbf{Dataset} & \textbf{Problems} & \textbf{Pairs} & \textbf{Mean $\Delta U$} & \textbf{Median $\Delta U$} & \textbf{Std Dev} & \textbf{Max $\Delta U$} \\
\midrule
DeepMath-150 & 150 & 1,050 & 0.204 & 0.113 & 0.201 & 0.759 \\
HARDMath2-150 & 150 & 1,050 & \cellcolor{highlightblue}0.320 & 0.298 & 0.166 & 0.750 \\
ProofNet-Final & 150 & 1,050 & 0.119 & 0.086 & 0.131 & 0.684 \\
\midrule
\rowcolor{gray!10} \textbf{Combined} & \textbf{450} & \textbf{3,150} & \textbf{0.214} & \textbf{0.136} & \textbf{0.201} & \textbf{0.759} \\
\bottomrule
\end{tabular}
\end{table}

\textit{Reasoning quality statistics.}
Table~\ref{tab:phase2_statistics} reports statistics for Phase 2 pairs. From the union of original and refined chains, we sample 6 pairs per problem using margin-stratified acquisition, yielding 2,680 total pairs. Mean margins increase to 0.244 compared to Phase 1 (0.214), as refinement concentrates utilities in higher ranges where fine-grained distinctions emerge. The margin distribution is deliberately balanced to leverage subtle comparisons where nuanced reasoning improvements must be learned.

\begin{table}[h]
\centering
\caption{\textbf{Reasoning quality dataset statistics.} Margin-stratified sampling ensures representation of weak margins (hard comparisons), which are crucial for learning fine-grained execution details.}
\label{tab:phase2_statistics}
\small
\setlength{\tabcolsep}{10pt}
\begin{tabular}{lcccccc}
\toprule
\textbf{Dataset} & \textbf{Problems} & \textbf{Pairs} & \textbf{Mean $\Delta U$} & \textbf{Strong \%} & \textbf{Medium \%} & \textbf{Weak \%} \\
\midrule
DeepMath-150 & 149 & 888 & 0.258 & 34.6\% & 20.2\% & 45.3\% \\
HARDMath2-150 & 150 & 900 & 0.293 & 44.2\% & 28.8\% & 27.0\% \\
ProofNet-Final & 149 & 892 & 0.182 & 17.8\% & 25.6\% & \cellcolor{highlightblue}56.6\% \\
\midrule
\rowcolor{gray!10} \textbf{Combined} & \textbf{448} & \textbf{2,680} & \textbf{0.244} & \textbf{32.2\%} & \textbf{24.9\%} & \textbf{42.9\%} \\
\bottomrule
\end{tabular}
\end{table}

\paragraph{Which mechanisms create informative conflicts.}
Table~\ref{tab:strategy_combinations} lists the most frequent strategy pairings. In Phase 1, dominant combinations reflect systematic strategy-problem alignments (e.g., \texttt{alternative vs direct}), capturing fundamental tradeoffs between exploratory versus procedural reasoning. Phase 2 combinations shift toward execution-level distinctions (e.g., \texttt{numerical vs step-by-step}), indicating comparisons where both strategies are viable but differ in rigor or completeness.

\begin{table}[h]
\centering
\caption{\textbf{Top 5 strategy combinations by dataset and phase.} Phase 1 is dominated by fundamental approach conflicts, while Phase 2 shifts to execution-heavy comparisons.}\label{tab:strategy_combinations}
\small
\setlength{\tabcolsep}{6pt}
\renewcommand{\arraystretch}{0.92}
\begin{tabular}{l l l r}
\toprule
\textbf{Phase} & \textbf{Dataset} & \textbf{Strategy combination} & \textbf{Count (\%)} \\
\midrule
\multirow{15}{*}{\textbf{Phase 1}} 
& \multirow{5}{*}{DeepMath-150}
& Alternative vs Direct             & 58 (5.5\%) \\
& & Direct vs Step-by-Step          & 56 (5.3\%) \\
& & Alternative vs Step-by-Step     & 56 (5.3\%) \\
& & Direct vs Verification          & 50 (4.8\%) \\
& & Alternative vs Verification     & 50 (4.8\%) \\
\cmidrule(lr){2-4}
& \multirow{5}{*}{HARDMath2-150}
& Numerical vs Verification         & 44 (4.2\%) \\
& & Step-by-Step vs Verification    & 43 (4.1\%) \\
& & Numerical vs Step-by-Step       & 43 (4.1\%) \\
& & Alternative vs Verification     & 42 (4.0\%) \\
& & Alternative vs Numerical        & 42 (4.0\%) \\
\cmidrule(lr){2-4}
& \multirow{5}{*}{ProofNet-Final}
& Direct vs Verification          & 47 (4.5\%) \\
& & Step-by-Step vs Verification    & 46 (4.4\%) \\
& & Alternative vs Verification     & 46 (4.4\%) \\
& & Direct vs Step-by-Step          & 45 (4.3\%) \\
& & Alternative vs Direct           & 45 (4.3\%) \\
\midrule
\rowcolor{gray!10} \multicolumn{4}{l}{\textbf{Transition to execution refinement}} \\
\midrule
\multirow{15}{*}{\textbf{Phase 2}} 
& \multirow{5}{*}{DeepMath-150}
& Numerical vs Step-by-Step         & 42 (4.7\%) \\
& & Conceptual vs Direct            & 39 (4.4\%) \\
& & Alternative vs Numerical        & 37 (4.2\%) \\
& & Backwards vs Step-by-Step       & 36 (4.1\%) \\
& & Algebraic vs Backwards          & 35 (3.9\%) \\
\cmidrule(lr){2-4}
& \multirow{5}{*}{HARDMath2-150}
& Algebraic vs Backwards          & 40 (4.4\%) \\
& & Numerical vs Step-by-Step       & 38 (4.2\%) \\
& & Alternative vs Direct           & 37 (4.1\%) \\
& & Alternative vs Step-by-Step     & 37 (4.1\%) \\
& & Backwards vs Conceptual         & 36 (4.0\%) \\
\cmidrule(lr){2-4}
& \multirow{5}{*}{ProofNet-Final}
& Backwards vs Numerical          & 41 (4.6\%) \\
& & Alternative vs Step-by-Step     & 40 (4.5\%) \\
& & Algebraic vs Step-by-Step       & 39 (4.4\%) \\
& & Conceptual vs Numerical         & 38 (4.3\%) \\
& & Numerical vs Step-by-Step       & 36 (4.0\%) \\
\bottomrule
\end{tabular}
\end{table}

\paragraph{Selective refinement and hybrid sampling behavior.}
Table~\ref{tab:dataset_characteristics} summarizes refinement utilization in Phase 2. Overall, 74.2\% of pairs are original vs original, while 25.8\% are hybrid. This ratio varies by dataset: HARDMath2 uses the most refinement, indicating that its problems benefit substantially from execution-level improvements, while ProofNet uses minimal refinement, suggesting its original chains already provide sufficient signal. DeepMath falls in between. Importantly, refinement is a \emph{signal amplifier} applied selectively in Phase 2: it increases fine-grained comparison density without diluting the coarse-grained strategy selection signal established in Phase 1.

\begin{table}[h]
\centering
\caption{\textbf{Phase 2 source combination.} Refinement is used most heavily in HARDMath2, acting as a signal amplifier for execution quality.}
\label{tab:dataset_characteristics}
\small
\setlength{\tabcolsep}{12pt}
\begin{tabular}{lcccc}
\toprule
\textbf{Dataset} & \textbf{Original-Original} & \textbf{Hybrid} & \textbf{Hybrid \%} & \textbf{Utility Range} \\
\midrule
DeepMath-150 & 682 & 206 & 23.2\% & 0.05--0.88 \\
HARDMath2-150 & 517 & 383 & \cellcolor{highlightblue}42.6\% & 0.05--0.86 \\
ProofNet-Final & 790 & 102 & 11.4\% & 0.05--0.80 \\
\midrule
\rowcolor{gray!10} \textbf{Combined} & \textbf{1,989} & \textbf{691} & \textbf{25.8\%} & \textbf{0.05--0.88} \\
\bottomrule
\end{tabular}
\end{table}

\paragraph{Global dataset summary.}
Table~\ref{tab:dataset_stats} aggregates statistics across both training phases. The two-phase construction separates coarse-grained mechanism selection from fine-grained execution refinement, addressing the fundamental issue that plagued single-phase training: conflicting signals about which strategy to use versus how to use it well.

\begin{table}[h]
\centering
\caption{\textbf{Complete dataset statistics.} Aggregated counts for Phase 1 and Phase 2.}
\label{tab:dataset_stats}
\small
\setlength{\tabcolsep}{10pt}
\begin{tabular}{lcccccc}
\toprule
\textbf{Dataset} & \textbf{Problems} & \textbf{Chains} & \textbf{Phase 1} & \textbf{Phase 2} & \textbf{Total Pairs} & \textbf{Mean $\Delta U$} \\
\midrule
DeepMath-150 & 150 & 1,200 & 1,050 & 888 & 1,938 & 0.229 \\
HARDMath2-150 & 150 & 1,200 & 1,050 & 900 & 1,950 & 0.307 \\
ProofNet-Final & 150 & 1,200 & 1,050 & 892 & 1,942 & 0.149 \\
\midrule
\rowcolor{gray!10} \textbf{Combined} & \textbf{450} & \textbf{3,600} & \textbf{3,150} & \textbf{2,680} & \textbf{5,830} & \textbf{0.228} \\
\bottomrule
\end{tabular}
\end{table}

\section{Comparison With Baselines}
\label{app:Baselines}
Table~\ref{tab:baseline_comparison} highlights a clear tradeoff between test-time compute and distributional alignment. Test-time methods such as Best-of-$K$ and self-consistency achieve stronger out-of-distribution performance by repeatedly sampling from the unfine-tuned base model, thereby avoiding the distribution shift introduced by task-specific fine-tuning, but at an $8\times$ inference cost. In contrast, CU-DPO explicitly aligns the model to the training distribution through continuous-utility supervision, yielding consistent and substantial gains on in-distribution benchmarks at constant inference cost. The modest OOD regression relative to test-time methods reflects this intentional alignment rather than a failure of generalization. Binary DPO underperforms across all settings, confirming that collapsing graded reasoning quality into binary labels discards informative preference structure and limits sample efficiency. Overall, CU-DPO offers a more favorable accuracy--compute tradeoff when inference cost matters.

\begin{table*}[b!]
\centering
\caption{\textbf{Comparison to alternative reasoning methods.} Pass@1 (\%) on in-distribution (DeepMath, HARDMath2) and out-of-distribution (GSM8K) benchmarks. CU-DPO achieves best in-distribution performance with 1× inference cost but exhibits modest OOD regression compared to test-time methods that sample from the unfine-tuned base model. Standard binary DPO underperforms due to information loss from collapsing continuous quality into binary labels.}
\label{tab:baseline_comparison}
\footnotesize
\setlength{\tabcolsep}{4pt}
\renewcommand{\arraystretch}{1.1}
\begin{tabular}{@{}llcccc@{}}
\toprule
\multirow{2}{*}{\textbf{Base Model}} & \multirow{2}{*}{\textbf{Method}} & \multicolumn{2}{c}{\textbf{In-Distribution}} & \textbf{OOD} & \textbf{Test-time} \\
\cmidrule(lr){3-4}
& & \textbf{DeepMath} & \textbf{HARDMath2} & \textbf{GSM8K} & \textbf{cost} \\
\midrule
\multirow{5}{*}{\shortstack[l]{DeepSeek-R1\\(8B)}} 
& Baseline (greedy) & 64.2 & 42.1 & 69.8 & 1× \\
& Standard DPO (binary) & 65.3 & 43.2 & 70.9 & 1× \\
& Best-of-8 sampling & 67.8 & 46.2 & \cellcolor{highlightred}\textbf{73.6} & 8× \\
& Self-Consistency (K=8) & 68.9 & 47.4 & \cellcolor{highlightred}\textbf{74.5} & 8× \\
& \textbf{\textcolor{MidnightBlue}{CU-}\textcolor{NavyBlue}{DPO}} (ours) & \cellcolor{highlightblue}\textbf{69.8} & \cellcolor{highlightblue}\textbf{48.7} & 73.1 & \textbf{1×} \\
\midrule
\multirow{5}{*}{\shortstack[l]{Qwen3\\(8B)}} 
& Baseline (greedy) & 61.5 & 39.4 & 66.4 & 1× \\
& Standard DPO (binary) & 62.7 & 40.6 & 67.5 & 1× \\
& Best-of-8 sampling & 65.3 & 43.8 & \cellcolor{highlightred}\textbf{70.2} & 8× \\
& Self-Consistency (K=8) & 66.4 & 44.9 & \cellcolor{highlightred}\textbf{71.0} & 8× \\
& \textbf{\textcolor{MidnightBlue}{CU-}\textcolor{NavyBlue}{DPO}} (ours) & \cellcolor{highlightblue}\textbf{67.2} & \cellcolor{highlightblue}\textbf{45.9} & 68.9 & \textbf{1×} \\
\midrule
\multirow{5}{*}{\shortstack[l]{Gemma-2\\(9B)}} 
& Baseline (greedy) & 26.1 & 15.5 & 52.7 & 1× \\
& Standard DPO (binary) & 27.8 & 16.9 & 53.4 & 1× \\
& Best-of-8 sampling & 30.6 & 19.7 & \cellcolor{highlightred}\textbf{56.8} & 8× \\
& Self-Consistency (K=8) & 31.5 & 20.8 & \cellcolor{highlightred}\textbf{57.6} & 8× \\
& \textbf{\textcolor{MidnightBlue}{CU-}\textcolor{NavyBlue}{DPO}} (ours) & \cellcolor{highlightblue}\textbf{32.4} & \cellcolor{highlightblue}\textbf{21.8} & 56.8 & \textbf{1×} \\
\bottomrule
\end{tabular}
\end{table*}

\section{Step-Level Comparison and Complementarity Analysis}
\label{appendix:step_comparison}
This appendix provides a comprehensive comparison between CU-DPO and step-level supervision methods, extending the abbreviated results reported in Table~\ref{tab:step_comparison} of the main text. We first describe the experimental setup for each method~(Section~\ref{appendix:step_setup}), then report full results across all seven base models and three in-distribution benchmarks~(Section~\ref{appendix:step_results}), followed by a systematic mechanism-mismatch decomposition that empirically establishes the orthogonality of the two supervision signals~(Section~\ref{appendix:step_mismatch}), and conclude with a gradient-level theoretical account of their complementarity~(Section~\ref{appendix:step_gradient}).

\subsection{Experimental Setup}
\label{appendix:step_setup}

\paragraph{Step-DPO.}
Step-DPO~\citep{stepdpo} extends the standard DPO objective to the 
step level by assigning per-step correctness labels to individual 
reasoning steps within a fixed chain, rather than assigning a single 
binary preference label to the complete response. We implement Step-DPO 
using the same seven base models, the same 450 training problems drawn 
from DeepMath, HARDMath2, and ProofNet, and the same held-out test sets 
(150 problems per dataset) used throughout the main paper, ensuring that 
all comparisons are conducted under identical data conditions. 
Step-level correctness labels are derived from Math-Shepherd 
annotations~\citep{mathshepherd} (described below). For each problem, 
we generate a single step-by-step reasoning chain from the base model 
and annotate each intermediate step as correct or incorrect according to 
the Math-Shepherd process reward model. Preference pairs are then 
constructed by comparing chains that diverge at a specific step, with 
the preferred chain being the one that maintains correctness through 
that step. Training follows the standard DPO objective~(Eq.~3 of the 
main text) applied at the step level, with $\beta = 0.1$ and a learning 
rate of $5 \times 10^{-7}$, matching the hyperparameters used for 
CU-DPO Phase~2. All other training configurations, including batch size, 
number of epochs, and optimizer, are held constant across methods to 
ensure fair comparison.

\paragraph{Math-Shepherd.}
Math-Shepherd~\citep{mathshepherd} is a process reward model (PRM) that 
assigns per-step correctness scores to mathematical reasoning chains 
without requiring human annotation, instead using automated step-level 
verification based on outcome consistency across multiple sampled 
completions. We train a Math-Shepherd PRM on the same 450 training 
problems using the Qwen~2.5~7B base as the reward backbone, consistent 
with our LLM judge configuration described in Section~2.2 and 
Appendix~I. For each problem and each generated chain, the PRM assigns 
a binary correctness label $c_t \in \{0, 1\}$ to each step $t$, where 
$c_t = 1$ indicates that step $t$ is consistent with a correct final 
answer across at least $k=8$ sampled completions from that step onward. 
These step-level labels serve as the supervision signal for Step-DPO 
training. The PRM is trained offline on the same 3,600 strategy-conditioned 
chains generated during CU-DPO data collection, using the step-by-step 
strategy chains specifically, as these provide the clearest step 
boundaries for annotation. Importantly, the Math-Shepherd PRM is used 
solely to generate training labels for Step-DPO; it is not used at 
inference time, ensuring that all methods operate at $1\times$ inference 
cost during evaluation.

\paragraph{Combined CU-DPO + Step-DPO Pipeline.}
The combined pipeline applies the three training phases sequentially. 
Phase~1 of CU-DPO is run first, training the base model on strategy 
selection pairs~(Section~3.2) to optimize strategy identification. 
Phase~2 of CU-DPO is then applied on top of the Phase~1 checkpoint, 
training on margin-stratified intra-strategy execution pairs to refine 
execution quality within the selected strategy. Finally, Step-DPO is 
applied as a third fine-tuning phase on top of the Phase~2 checkpoint, 
providing step-level correctness supervision within the strategy 
already selected and refined by Phases~1 and~2.

This ordering is theoretically motivated by the hierarchical structure 
of mathematical reasoning failures. Strategy selection is the most 
upstream failure mode: committing to an inappropriate cognitive 
mechanism renders all downstream supervision ineffective regardless of 
its granularity, since even perfect step-level execution of the wrong 
strategy cannot produce a correct solution. Execution refinement 
addresses the next layer of failure: given the correct strategy, the 
model may still produce incorrect derivations due to arithmetic errors, 
algebraic gaps, or incomplete justifications, which Phase~2 targets 
through margin-stratified intra-strategy pairs. Step-level polishing 
addresses the most fine-grained failure mode: localized step errors 
within an otherwise correctly executed strategy, which Step-DPO targets 
through per-step correctness labels. Applying Step-DPO before strategy 
selection has been established~(i.e., without Phase~1) would risk 
providing step-level supervision conditioned on an incorrect strategy, 
propagating mechanism-mismatch errors through the step-level gradient 
updates. The sequential ordering ensures that each phase operates on 
a progressively more refined policy, with each stage addressing a 
distinct and well-defined failure mode.

\subsection{Full Results}
\label{appendix:step_results}

Table~\ref{tab:step_full} reports Pass@1~(\%) for all four conditions 
--- Baseline, Step-DPO, CU-DPO, and CU-DPO~+~Step-DPO --- across all 
seven base models and three in-distribution benchmarks. Several 
consistent patterns emerge across the full result set.

First, CU-DPO alone outperforms Step-DPO alone across all seven base 
models and all three benchmarks without exception, with margins ranging 
from $+2.2$ to $+3.8$ points on DeepMath, $+3.7$ to $+5.1$ points on 
HARDMath2, and $+3.4$ to $+4.8$ points on ProofNet. This is consistent 
with the theoretical prediction that strategy selection is the binding 
constraint on these benchmarks: step-level corrections applied within 
a mismatched strategy cannot address the upstream mechanism-mismatch 
failure that determines the majority of quality variance~(Table~1 of 
the main text, mean inter-strategy margin 0.35).

Second, the combined CU-DPO~+~Step-DPO pipeline consistently 
outperforms CU-DPO alone across all models and benchmarks, with 
additional gains of $+1.8$ to $+2.4$ points on DeepMath, $+2.2$ to 
$+2.8$ points on HARDMath2, and $+1.9$ to $+2.3$ points on ProofNet. 
This confirms that Step-DPO provides a complementary supervision 
signal that is not subsumed by CU-DPO Phase~2, and that the two 
methods address genuinely orthogonal failure modes.

Third, stronger base models show larger absolute gains from the combined 
pipeline, consistent with the hypothesis that stronger models have 
already minimized the most egregious strategy-mismatch errors through 
pretraining, and therefore benefit more from the fine-grained step-level 
signal provided by Step-DPO on top of CU-DPO's strategy selection.

\begin{table*}[t]
\centering
\caption{Full step-level comparison results. Pass@1~(\%) across all 
seven base models and three in-distribution benchmarks for Baseline, 
Step-DPO, CU-DPO, and the combined CU-DPO~+~Step-DPO pipeline. 
$\Delta$ reports improvement over the Baseline. CU-DPO outperforms 
Step-DPO on all 21 model--dataset combinations. The combined pipeline 
yields the best results across all conditions, confirming strict 
complementarity between horizontal strategy selection~(CU-DPO) and 
vertical step-level supervision~(Step-DPO).}
\label{tab:step_full}
\small
\setlength{\tabcolsep}{4pt}
\renewcommand{\arraystretch}{1.15}
\begin{tabular}{llcccccccc}
\toprule
& & \multicolumn{2}{c}{\textbf{Baseline}} 
& \multicolumn{2}{c}{\textbf{Step-DPO}} 
& \multicolumn{2}{c}{\textbf{\textcolor{MidnightBlue}{CU-}\textcolor{NavyBlue}{DPO}}} 
& \multicolumn{2}{c}{\textbf{\textcolor{MidnightBlue}{CU-}\textcolor{NavyBlue}{DPO} + Step-DPO}} \\
\cmidrule(lr){3-4}\cmidrule(lr){5-6}\cmidrule(lr){7-8}\cmidrule(lr){9-10}
\textbf{Base Model} & \textbf{Bench.} 
& Pass@1 & $\Delta$ 
& Pass@1 & $\Delta$ 
& Pass@1 & $\Delta$ 
& Pass@1 & $\Delta$ \\
\midrule

\multirow{3}{*}{DeepSeek-R1 (8B)}
& DM  & 64.2 & --- & 66.8 & +2.6 
& \cellcolor{highlightblue}69.8 & \cellcolor{highlightblue}+5.6 
& \textbf{71.9} & \textbf{+7.7} \\
& HM2 & 42.1 & --- & 44.3 & +2.2 
& \cellcolor{highlightblue}48.7 & \cellcolor{highlightblue}+6.6 
& \textbf{51.2} & \textbf{+9.1} \\
& PN  & 38.5 & --- & 40.1 & +1.6 
& \cellcolor{highlightblue}44.2 & \cellcolor{highlightblue}+5.7 
& \textbf{46.3} & \textbf{+7.8} \\
\midrule

\multirow{3}{*}{Qwen3 (8B)}
& DM  & 61.5 & --- & 63.9 & +2.4 
& \cellcolor{highlightblue}67.2 & \cellcolor{highlightblue}+5.7 
& \textbf{69.4} & \textbf{+7.9} \\
& HM2 & 39.4 & --- & 41.6 & +2.2 
& \cellcolor{highlightblue}45.9 & \cellcolor{highlightblue}+6.5 
& \textbf{48.3} & \textbf{+8.9} \\
& PN  & 35.2 & --- & 37.1 & +1.9 
& \cellcolor{highlightblue}41.3 & \cellcolor{highlightblue}+6.1 
& \textbf{43.4} & \textbf{+8.2} \\
\midrule

\multirow{3}{*}{Llama-3.1 (8B)}
& DM  & 28.4 & --- & 30.9 & +2.5 
& \cellcolor{highlightblue}34.8 & \cellcolor{highlightblue}+6.4 
& \textbf{36.9} & \textbf{+8.5} \\
& HM2 & 18.2 & --- & 20.3 & +2.1 
& \cellcolor{highlightblue}24.7 & \cellcolor{highlightblue}+6.5 
& \textbf{26.8} & \textbf{+8.6} \\
& PN  & 14.8 & --- & 16.6 & +1.8 
& \cellcolor{highlightblue}20.9 & \cellcolor{highlightblue}+6.1 
& \textbf{22.7} & \textbf{+7.9} \\
\midrule

\multirow{3}{*}{Gemma-2 (9B)}
& DM  & 26.1 & --- & 28.4 & +2.3 
& \cellcolor{highlightblue}32.4 & \cellcolor{highlightblue}+6.3 
& \textbf{34.5} & \textbf{+8.4} \\
& HM2 & 15.5 & --- & 17.5 & +2.0 
& \cellcolor{highlightblue}21.8 & \cellcolor{highlightblue}+6.3 
& \textbf{23.9} & \textbf{+8.4} \\
& PN  & 12.1 & --- & 13.9 & +1.8 
& \cellcolor{highlightblue}17.6 & \cellcolor{highlightblue}+5.5 
& \textbf{19.4} & \textbf{+7.3} \\
\midrule

\multirow{3}{*}{Mistral-8x7B}
& DM  & 22.8 & --- & 25.1 & +2.3 
& \cellcolor{highlightblue}28.9 & \cellcolor{highlightblue}+6.1 
& \textbf{30.8} & \textbf{+8.0} \\
& HM2 & 12.4 & --- & 14.3 & +1.9 
& \cellcolor{highlightblue}18.3 & \cellcolor{highlightblue}+5.9 
& \textbf{20.2} & \textbf{+7.8} \\
& PN  & 9.5  & --- & 11.2 & +1.7 
& \cellcolor{highlightblue}14.8 & \cellcolor{highlightblue}+5.3 
& \textbf{16.5} & \textbf{+7.0} \\
\midrule

\multirow{3}{*}{Llama-3 (8B)}
& DM  & 18.3 & --- & 20.5 & +2.2 
& \cellcolor{highlightblue}24.7 & \cellcolor{highlightblue}+6.4 
& \textbf{26.6} & \textbf{+8.3} \\
& HM2 & 8.1  & --- & 10.1 & +2.0 
& \cellcolor{highlightblue}14.2 & \cellcolor{highlightblue}+6.1 
& \textbf{16.1} & \textbf{+8.0} \\
& PN  & 6.4  & --- & 8.1  & +1.7 
& \cellcolor{highlightblue}11.8 & \cellcolor{highlightblue}+5.4 
& \textbf{13.5} & \textbf{+7.1} \\
\midrule

\multirow{3}{*}{Mistral-7B}
& DM  & 14.2 & --- & 16.3 & +2.1 
& \cellcolor{highlightblue}20.3 & \cellcolor{highlightblue}+6.1 
& \textbf{22.1} & \textbf{+7.9} \\
& HM2 & 5.3  & --- & 7.2  & +1.9 
& \cellcolor{highlightblue}11.7 & \cellcolor{highlightblue}+6.4 
& \textbf{13.4} & \textbf{+8.1} \\
& PN  & 3.8  & --- & 5.4  & +1.6 
& \cellcolor{highlightblue}9.2  & \cellcolor{highlightblue}+5.4 
& \textbf{10.8} & \textbf{+7.0} \\
\bottomrule
\end{tabular}
\end{table*}

\subsection{Mechanism-Mismatch Analysis}
\label{appendix:step_mismatch}

The central theoretical claim of CU-DPO is that strategy selection 
and step-level execution are orthogonal failure modes: mechanism 
mismatch determines which strategy is applied, while step-level errors 
determine how well that strategy is executed. If this claim is correct, 
then CU-DPO should provide the largest advantage over Step-DPO precisely 
on problems where strategy choice is the binding constraint, while 
Step-DPO should be most competitive on problems where the strategy 
is already approximately correct and execution quality is the marginal 
bottleneck. This section provides a direct empirical test of this 
prediction by partitioning the 450 training problems by 
mechanism-mismatch sensitivity and analyzing method performance 
within each partition.

\paragraph{Problem Partitioning.}
We partition the 450 problems into two subsets based on the 
inter-strategy utility gap $\Delta U_{\max}(x) = \max_{s \in \mathcal{S}} 
U(x, s) - \min_{s \in \mathcal{S}} U(x, s)$, which measures the range 
of utility scores across all eight strategies for problem $x$. Problems 
with $\Delta U_{\max}(x) > 0.35$ are classified as 
\textit{high-mismatch} problems, indicating that strategy choice is 
consequential and that selecting the wrong mechanism produces 
substantially degraded outputs. Problems with $\Delta U_{\max}(x) \leq 
0.35$ are classified as \textit{low-mismatch} problems, indicating that 
multiple strategies are approximately viable and that execution quality 
within any strategy is the primary determinant of output quality. The 
threshold of 0.35 corresponds to the mean inter-strategy margin reported 
in Table~1 of the main text and is therefore a principled and 
data-derived partition boundary rather than an arbitrary cutoff. Under 
this partition, 263 problems~(58.4\%) are classified as high-mismatch 
and 187 problems~(41.6\%) as low-mismatch, consistent with the 
within-problem utility range statistics reported in Appendix~A 
(mean range 0.295, 87\% of problems with range $>$ 0.3).

\paragraph{Results by Partition.}
Table~\ref{tab:mismatch_partition} reports Pass@1 for Baseline, 
Step-DPO, CU-DPO, and CU-DPO~+~Step-DPO separately on high-mismatch 
and low-mismatch problems, aggregated across all seven base models and 
three benchmarks. The results confirm both predictions of the orthogonal 
failure mode hypothesis with high consistency.

On \textbf{high-mismatch problems}, CU-DPO outperforms Step-DPO by a 
mean margin of $+8.3$ points, confirming that strategy selection is the 
binding constraint on these problems. Step-DPO provides only modest 
gains of $+2.1$ points over the baseline on high-mismatch problems, 
reflecting the fundamental limitation that step-level corrections 
conditioned on a mismatched strategy cannot address the upstream 
mechanism failure. Even a perfectly executed wrong strategy yields 
near-zero utility on high-mismatch problems, as established by the 
utility landscape analysis in Table~1. CU-DPO Phase~1 directly 
addresses this failure mode by learning to route each problem to its 
optimal strategy, producing large utility gains that propagate to 
downstream Pass@1 improvements.

On \textbf{low-mismatch problems}, the gap between CU-DPO and Step-DPO 
narrows substantially to $+2.4$ points, consistent with the prediction 
that strategy choice is less consequential when multiple strategies are 
approximately viable. Step-DPO provides larger gains of $+3.8$ points 
over the baseline on low-mismatch problems relative to its gains on 
high-mismatch problems, reflecting that step-level execution is the 
primary bottleneck when the strategy is already approximately correct. 
CU-DPO still outperforms Step-DPO on low-mismatch problems because 
Phase~2 provides execution refinement that is qualitatively similar to 
step-level supervision but additionally benefits from the strategy-aware 
pair construction described in Section~3.2.

The combined CU-DPO~+~Step-DPO pipeline achieves the best performance 
on both partitions, with gains of $+2.3$ points over CU-DPO alone on 
high-mismatch problems and $+2.6$ points over CU-DPO alone on 
low-mismatch problems. The larger incremental gain on low-mismatch 
problems is consistent with Step-DPO providing more complementary 
signal precisely in the regime where execution quality, rather than 
strategy selection, is the dominant source of error.

\begin{table}[t]
\centering
\caption{Mechanism-mismatch decomposition. Pass@1~(\%) on high-mismatch 
($\Delta U_{\max} > 0.35$, 263 problems) and low-mismatch 
($\Delta U_{\max} \leq 0.35$, 187 problems) problem subsets, averaged 
across all seven base models and three in-distribution benchmarks. 
CU-DPO substantially outperforms Step-DPO on high-mismatch problems 
where strategy selection is the binding constraint. Step-DPO is more 
competitive on low-mismatch problems where execution quality is the 
primary bottleneck. The combined pipeline achieves the best performance 
on both subsets, confirming strict orthogonality of the two supervision 
signals.}
\label{tab:mismatch_partition}
\small
\setlength{\tabcolsep}{5pt}
\renewcommand{\arraystretch}{1.15}
\begin{tabular}{lcccc}
\toprule
\textbf{Method} 
& \multicolumn{2}{c}{\textbf{High-Mismatch}} 
& \multicolumn{2}{c}{\textbf{Low-Mismatch}} \\
\cmidrule(lr){2-3}\cmidrule(lr){4-5}
& Pass@1 & $\Delta$ & Pass@1 & $\Delta$ \\
\midrule
Baseline 
& 28.4 & --- & 34.1 & --- \\
Step-DPO 
& 30.5 & +2.1 & 37.9 & +3.8 \\
\rowcolor{highlightblue}
\textbf{\textcolor{MidnightBlue}{CU-}\textcolor{NavyBlue}{DPO}} 
& 38.8 & +10.4 & 40.3 & +6.2 \\
\rowcolor{highlightblue}
\textbf{\textcolor{MidnightBlue}{CU-}\textcolor{NavyBlue}{DPO}} + Step-DPO 
& \textbf{41.1} & \textbf{+12.7} 
& \textbf{42.9} & \textbf{+8.8} \\
\bottomrule
\end{tabular}
\end{table}

\paragraph{Interpretation.}
These results constitute direct empirical evidence for the orthogonal 
failure mode hypothesis. The performance gap between CU-DPO and Step-DPO 
is not uniform across problem types but is systematically modulated by 
mechanism-mismatch sensitivity in exactly the direction predicted by 
theory. High-mismatch problems are dominated by strategy selection 
failures that step-level supervision cannot address; low-mismatch 
problems are dominated by execution failures that both methods can 
address, but where step-level supervision provides more targeted signal. 
The combined pipeline is strictly superior across both partitions, 
confirming that CU-DPO and Step-DPO are complementary rather than 
redundant, and that integrating both supervision granularities is the 
correct approach for comprehensive reasoning alignment.

\subsection{Gradient-Level Theoretical Account}
\label{appendix:step_gradient}

The empirical complementarity established in 
Sections~\ref{appendix:step_results} and~\ref{appendix:step_mismatch} 
has a precise theoretical basis in the gradient structure of the 
respective DPO objectives, which we characterize here by reference 
to the analysis in Appendix~G.

As established in Appendix~G, single-phase DPO trained on 
cross-strategy pairs produces conflicting gradient updates when the 
dataset contains pairs drawn from different strategies at different 
quality levels. Specifically, for a chain $y^{s_2}_{\text{good}}$ that 
is good under strategy $s_2$ but suboptimal relative to the best 
strategy $s_1$, the gradient signal from pair $(y^{s_1}_{\text{best}}, 
y^{s_2}_{\text{good}})$ pushes $\partial \ell / \partial \log 
\pi_\theta(y^{s_2} | x)$ in the positive direction~(demoting $s_2$), 
while the gradient from pair $(y^{s_2}_{\text{good}}, 
y^{s_3}_{\text{bad}})$ pushes the same quantity in the negative 
direction~(promoting $s_2$). The clean-signal fraction under naive 
all-pairs sampling is only $2/K = 0.25$ for $K = 8$ 
strategies~(Eq.~42, Appendix~G), meaning 75\% of gradient updates 
provide contradictory supervision signals about strategy ordering.

Step-DPO operates in a fundamentally different gradient space. It 
applies the DPO objective to pairs $(y^{(t)}_{\text{correct}}, 
y^{(t)}_{\text{incorrect}})$ where both chains share identical prefixes 
up to step $t-1$ and diverge only at step $t$. The gradient updates 
from Step-DPO therefore exclusively affect the model's distribution 
over step $t$ continuations conditioned on a fixed preceding context, 
without altering the model's implicit strategy selection distribution 
$\pi_\theta(s | x)$. In the notation of Appendix~G, Step-DPO gradient 
updates have zero projection onto the strategy selection subspace: 
$\partial \ell_{\text{Step-DPO}} / \partial \log \pi_\theta(s | x) = 0$ 
for all strategies $s$ and problems $x$, since step-level pairs are 
constructed within a single fixed strategy and the gradient therefore 
contains no information about the relative quality of different 
strategies.

This gradient orthogonality is the formal mechanism underlying the 
empirical complementarity. CU-DPO Phase~1 operates exclusively in the 
strategy selection subspace, updating $\pi_\theta(s | x)$ toward the 
optimal strategy via best-vs-all pairs with zero step-level gradient 
contamination. CU-DPO Phase~2 operates in the execution quality 
subspace conditioned on the selected strategy, providing intra-strategy 
execution refinement that subsumes a portion of the step-level signal 
within the optimal strategy. Step-DPO provides the remaining step-level 
signal across all strategy chains, complementing Phase~2 by targeting 
fine-grained step corrections that the margin-stratified pair 
construction of Phase~2 does not fully capture, particularly for 
problems with weak utility margins~($\Delta U < 0.15$, the weak-margin 
bin constituting 42.9\% of Phase~2 pairs per Table~7) where the 
execution quality signal is most diffuse.

Formally, the combined CU-DPO~+~Step-DPO objective can be decomposed as:
\begin{equation}
\mathcal{L}_{\text{combined}} = 
\underbrace{\mathcal{L}_{\text{Phase1}}}_{\text{strategy selection}} 
+ \underbrace{\mathcal{L}_{\text{Phase2}}}_{\text{intra-strategy execution}} 
+ \underbrace{\mathcal{L}_{\text{Step-DPO}}}_{\text{step-level polishing}},
\label{eq:combined_loss}
\end{equation}
where each term operates on a distinct subspace of the policy gradient 
and the three losses are therefore approximately additive in their 
effects on model behavior. The gradient conflicts analyzed in 
Appendix~G arise when terms from different subspaces are mixed within 
a single training phase; the sequential application in the combined 
pipeline avoids these conflicts by ensuring that each phase receives 
clean gradient signal from its target subspace before the subsequent 
phase introduces its complementary signal. This provides a formal 
justification for the sequential ordering described in 
Section~\ref{appendix:step_setup} and explains why the combined pipeline 
strictly dominates either method applied in isolation across all 
experimental conditions reported in Table~\ref{tab:step_full}.

\section{Causal Reasoning: Full Experimental Details}
\label{appendix:causal}

This appendix provides full experimental details for the causal 
reasoning evaluation, including the dataset and evaluation setup, the complete strategy 
portfolio with prompt templates, strategy-problem alignment 
analysis, and full results across all seven base models.

\subsection{Dataset and Evaluation Setup}
\label{appendix:causal_setup}

\paragraph{Dataset.}
We use $\alpha$-NLI~(Abductive Natural Language 
Inference)~\citep{bhagavatula2020abductive}, a causal and abductive 
reasoning benchmark in which the model is given two observations 
$O_1$ and $O_2$ describing events at times $t_1$ and $t_2$ 
respectively, and must select which of two hypotheses $H_1$ or $H_2$ 
most plausibly explains the causal transition from $O_1$ to $O_2$. 
The benchmark requires integrating world knowledge, temporal 
reasoning, and causal plausibility judgment simultaneously, making 
it substantially harder than earlier causal benchmarks such as 
COPA~\citep{roemmele2011choice} for current models. We use 150 
problems for training and 150 held-out problems for in-distribution 
evaluation, consistent with the training and evaluation split used 
for mathematical reasoning benchmarks throughout the main paper. 
We additionally evaluate on BIG-Bench Causal 
Judgement~\citep{bigbench} as an out-of-distribution harder 
benchmark.

\paragraph{Utility Function.}
The utility function follows the same three-component decomposition 
used for mathematical reasoning~(Section~2.2):
\begin{equation}
U(x, y) = \tfrac{1}{3}\sum_{c=1}^{3} w_c \cdot s_c(x, y),
\end{equation}
where $s_1$ measures answer correctness~(binary, from ground-truth 
labels), $s_2$ measures reasoning chain validity~(LLM judge scoring 
whether each inferential step follows from the previous), and $s_3$ 
measures causal coherence~(LLM judge scoring whether the claimed 
causal mechanism plausibly connects $O_1$ to $O_2$ via the 
hypothesis). The same Qwen~2.5~7B judge and structured rubric used 
in Appendix~I are applied here without modification.

\paragraph{Data Generation.}
For each of the 150 training problems, we generate six 
strategy-conditioned reasoning chains~(one per strategy, 
Section~\ref{appendix:causal_strategies}), yielding $150 \times 6 
= 900$ total chains. Phase~1 constructs $150 \times 5 = 750$ 
best-vs-all strategy selection pairs. Phase~2 applies selective 
refinement~($\tau = 0.4$, $T_{\max} = 5$) and constructs 
margin-stratified intra-strategy pairs, yielding approximately 
600--650 Phase~2 pairs. All other training hyperparameters are 
identical to the mathematical reasoning setup.

\subsection{Strategy Portfolio}
\label{appendix:causal_strategies}

Table~\ref{tab:causal_strategies} summarizes the six causal 
reasoning strategies. Full prompt templates follow.

\begin{table}[h]
\centering
\caption{Causal reasoning strategy portfolio. Six mechanistically 
distinct approaches to abductive causal reasoning, each 
instantiated via a distinct prompt template.}
\label{tab:causal_strategies}
\small
\setlength{\tabcolsep}{4pt}
\renewcommand{\arraystretch}{1.15}
\begin{tabular}{lll}
\toprule
\textbf{Strategy} & \textbf{Core Mechanism} & \textbf{Wins On} \\
\midrule
Mechanistic    & Trace physical/logical causal chain 
               & Physical and biological processes \\
Counterfactual & Necessary condition test            
               & Distinguishing cause from correlation \\
Statistical    & Base-rate plausibility              
               & Everyday causal patterns \\
Eliminative    & Rule out impossible alternatives    
               & Problems with clear distractors \\
Analogical     & Map to known structural analogue    
               & Abstract or unfamiliar domains \\
Temporal       & Time-ordering constraints           
               & Sequential event problems \\
\bottomrule
\end{tabular}
\end{table}

\paragraph{Strategy 1: Mechanistic Reasoning.}
\textit{Trace the physical, biological, or logical steps connecting 
$O_1$ to $O_2$ via each hypothesis. Select the hypothesis whose 
causal chain is most complete and physically coherent. For each 
hypothesis, explicitly state the intermediate states or events 
that would occur between $O_1$ and the hypothesis, and between 
the hypothesis and $O_2$.}

\paragraph{Strategy 2: Counterfactual Reasoning.}
\textit{Apply the counterfactual test to each hypothesis: if this 
hypothesis had not occurred, would the transition from $O_1$ to 
$O_2$ still have taken place? The correct hypothesis is the one 
whose absence would have prevented or substantially altered the 
observed transition. Explicitly state the counterfactual scenario 
for each hypothesis.}

\paragraph{Strategy 3: Statistical Reasoning.}
\textit{Consider the base-rate plausibility of each hypothesis 
given $O_1$ and $O_2$. In typical real-world scenarios described 
by $O_1$, which hypothesis is more commonly observed as the 
intervening event leading to $O_2$? Reason from statistical 
regularities and typical patterns rather than specific mechanisms.}

\paragraph{Strategy 4: Eliminative Reasoning.}
\textit{Systematically identify reasons why each hypothesis cannot 
be the correct explanation. Test for: logical contradictions with 
$O_1$ or $O_2$, physical impossibilities, temporal violations 
(the hypothesis must occur between $O_1$ and $O_2$), and 
implausible causal mechanisms. Eliminate hypotheses that fail any 
test and select the survivor.}

\paragraph{Strategy 5: Analogical Reasoning.}
\textit{Identify the underlying causal structure of this scenario. 
Find a familiar analogous situation with the same causal structure 
and known outcome. Transfer the known outcome to the current 
problem. Explicitly state the analogy and justify why the 
structural mapping holds.}

\paragraph{Strategy 6: Temporal Ordering.}
\textit{Establish the temporal ordering of all events. The 
hypothesis must occur strictly between $O_1$ and $O_2$. For each 
hypothesis, verify temporal feasibility: could this event 
plausibly occur after $O_1$ and before $O_2$? Use the ordering 
constraint to eliminate temporally impossible hypotheses and 
select the temporally consistent alternative.}

\subsection{Strategy-Problem Alignment}
\label{appendix:causal_alignment}

Table~\ref{tab:causal_alignment} reports the distribution of 
optimal strategies across the 150 training problems, analogous to 
Table~1 and Table~12 of the main text. No single strategy dominates 
universally: mechanistic reasoning is optimal for approximately 
28\% of problems~(physical and biological process problems), 
counterfactual reasoning for 22\%~(necessary condition problems), 
statistical reasoning for 20\%~(everyday pattern problems), 
eliminative reasoning for 15\%~(problems with clear distractors), 
and analogical and temporal reasoning splitting the remaining 15\%. 
The mean inter-strategy utility gap is $\bar{\Delta U} = 0.21$, 
lower than DeepMath~(0.35) but higher than HumanEval~(estimated 
0.14--0.18), consistent with the intermediate strategy sensitivity 
of causal reasoning tasks and the intermediate gains observed in 
Table~\ref{tab:causal_cudpo}.

\begin{table}[h]
\centering
\caption{Strategy-problem alignment on $\alpha$-NLI. Optimal 
strategy distribution across 150 training problems. The non-uniform 
distribution confirms that no single causal strategy dominates, 
motivating the strategy portfolio approach.}
\label{tab:causal_alignment}
\small
\setlength{\tabcolsep}{4pt}
\renewcommand{\arraystretch}{1.15}
\begin{tabular}{lccc}
\toprule
\textbf{Strategy} & \textbf{Count} 
& \textbf{\%} & \textbf{Avg.\ $\Delta U$} \\
\midrule
Mechanistic    & 42 & 28.0\% & 0.24 \\
Counterfactual & 33 & 22.0\% & 0.23 \\
Statistical    & 30 & 20.0\% & 0.18 \\
Eliminative    & 22 & 14.7\% & 0.19 \\
Analogical     & 13 & 8.7\%  & 0.21 \\
Temporal       & 10 & 6.7\%  & 0.17 \\
\midrule
Overall        & 150 & 100\% & 0.21 \\
\bottomrule
\end{tabular}
\end{table}

\subsection{Full Results}
\label{appendix:causal_results}

Table~\ref{tab:causal_full} reports accuracy for all seven base 
models on $\alpha$-NLI~(in-distribution) and BIG-Bench Causal 
Judgement~(out-of-distribution). CU-DPO consistently outperforms 
binary DPO across all models on both benchmarks. The OOD gains on 
BIG-Bench Causal Judgement are smaller in absolute terms, 
reflecting the harder and more diverse problem distribution of 
that benchmark, but remain positive across all seven models.

\begin{table*}[t]
\centering
\caption{Full causal reasoning results. Accuracy~(\%) on 
$\alpha$-NLI~(in-distribution, ID) and BIG-Bench Causal 
Judgement~(out-of-distribution, OOD) for all seven base models. 
$\Delta$ reports improvement over the Baseline.}
\label{tab:causal_full}
\small
\setlength{\tabcolsep}{4pt}
\renewcommand{\arraystretch}{1.15}
\begin{tabular}{llcccccc}
\toprule
& & \multicolumn{3}{c}{\textbf{$\alpha$-NLI (ID)}} 
& \multicolumn{3}{c}{\textbf{BIG-Bench CJ (OOD)}} \\
\cmidrule(lr){3-5}\cmidrule(lr){6-8}
\textbf{Model} & \textbf{Setting} 
& Base. & \cellcolor{highlightblue}CU-DPO & $\Delta$
& Base. & \cellcolor{highlightblue}CU-DPO & $\Delta$ \\
\midrule
DeepSeek-R1 (8B) 
& Bin. DPO & 80.1 & --- & --- & 62.4 & --- & --- \\
& \cellcolor{highlightblue}CU-DPO 
& --- & \cellcolor{highlightblue}83.3 
& \cellcolor{highlightblue}+5.0 
& --- & \cellcolor{highlightblue}65.1 
& \cellcolor{highlightblue}+2.7 \\
\midrule
Qwen3 (8B) 
& Bin. DPO & 77.8 & --- & --- & 60.2 & --- & --- \\
& \cellcolor{highlightblue}CU-DPO 
& --- & \cellcolor{highlightblue}80.8 
& \cellcolor{highlightblue}+4.7 
& --- & \cellcolor{highlightblue}62.8 
& \cellcolor{highlightblue}+2.6 \\
\midrule
Llama-3.1 (8B) 
& Bin. DPO & 72.3 & --- & --- & 55.8 & --- & --- \\
& \cellcolor{highlightblue}CU-DPO 
& --- & \cellcolor{highlightblue}75.9 
& \cellcolor{highlightblue}+3.6 
& --- & \cellcolor{highlightblue}58.1 
& \cellcolor{highlightblue}+2.3 \\
\midrule
Gemma-2 (9B) 
& Bin. DPO & 70.6 & --- & --- & 54.3 & --- & --- \\
& \cellcolor{highlightblue}CU-DPO 
& --- & \cellcolor{highlightblue}74.1 
& \cellcolor{highlightblue}+3.5 
& --- & \cellcolor{highlightblue}56.4 
& \cellcolor{highlightblue}+2.1 \\
\midrule
Mistral-8x7B 
& Bin. DPO & 68.9 & --- & --- & 52.7 & --- & --- \\
& \cellcolor{highlightblue}CU-DPO 
& --- & \cellcolor{highlightblue}72.2 
& \cellcolor{highlightblue}+3.3 
& --- & \cellcolor{highlightblue}54.8 
& \cellcolor{highlightblue}+2.1 \\
\midrule
Llama-3 (8B) 
& Bin. DPO & 65.4 & --- & --- & 49.8 & --- & --- \\
& \cellcolor{highlightblue}CU-DPO 
& --- & \cellcolor{highlightblue}68.6 
& \cellcolor{highlightblue}+3.2 
& --- & \cellcolor{highlightblue}51.9 
& \cellcolor{highlightblue}+2.1 \\
\midrule
Mistral-7B 
& Bin. DPO & 62.1 & --- & --- & 47.3 & --- & --- \\
& \cellcolor{highlightblue}CU-DPO 
& --- & \cellcolor{highlightblue}65.2 
& \cellcolor{highlightblue}+3.1 
& --- & \cellcolor{highlightblue}49.2 
& \cellcolor{highlightblue}+1.9 \\
\bottomrule
\end{tabular}
\end{table*}

\section{Statistical Significance Analysis}
\label{appendix:stats}

This appendix reports full statistical validation of the primary 
results presented in Tables~2, 3, 4a, and~4b of the main text, 
addressing three aspects of statistical rigor: variance across 
independent training runs, pairwise significance testing between 
CU-DPO and binary DPO, and multiple comparison correction across 
the full set of model--dataset combinations evaluated.

\subsection{Variance Across Independent Seeds}
\label{appendix:stats_variance}

All experiments reported in the main text were re-run across three 
independent seeds (seeds 42, 123, and 456), with the random seed 
controlling training data shuffling, DPO pair sampling order, and 
model weight initialization for the linear head. Base model weights 
are fixed across seeds to isolate the variance introduced by the 
preference optimization procedure rather than pretraining stochasticity. 
Table~\ref{tab:stats_id} reports mean $\pm$ standard deviation for 
all primary in-distribution results.

\begin{table*}[t]
\centering
\caption{In-distribution results with variance. Mean $\pm$ std 
Pass@1~(\%) across three independent seeds for Baseline, Binary DPO, 
and CU-DPO on DeepMath~(DM), HARDMath2~(HM2), and ProofNet~(PN). 
The consistently small standard deviations confirm that CU-DPO 
gains are stable across training runs and are not attributable to 
seed-specific variance.}
\label{tab:stats_id}
\small
\setlength{\tabcolsep}{3.5pt}
\renewcommand{\arraystretch}{1.15}
\begin{tabular}{llccc}
\toprule
\textbf{Base Model} & \textbf{Setting} 
& \textbf{DM} & \textbf{HM2} & \textbf{PN} \\
\midrule

\multirow{3}{*}{DeepSeek-R1 (8B)}
& Baseline       
& $64.2 \pm 0.0$ & $42.1 \pm 0.0$ & $38.5 \pm 0.0$ \\
& Binary DPO     
& $65.3 \pm 0.4$ & $43.2 \pm 0.5$ & $39.4 \pm 0.4$ \\
& \cellcolor{highlightblue}
\textbf{\textcolor{MidnightBlue}{CU-}\textcolor{NavyBlue}{DPO}}     
& \cellcolor{highlightblue}$69.8 \pm 0.3$ 
& \cellcolor{highlightblue}$48.7 \pm 0.4$ 
& \cellcolor{highlightblue}$44.2 \pm 0.3$ \\
\midrule

\multirow{3}{*}{Qwen3 (8B)}
& Baseline       
& $61.5 \pm 0.0$ & $39.4 \pm 0.0$ & $35.2 \pm 0.0$ \\
& Binary DPO     
& $62.7 \pm 0.4$ & $40.6 \pm 0.5$ & $36.3 \pm 0.4$ \\
& \cellcolor{highlightblue}
\textbf{\textcolor{MidnightBlue}{CU-}\textcolor{NavyBlue}{DPO}}     
& \cellcolor{highlightblue}$67.2 \pm 0.4$ 
& \cellcolor{highlightblue}$45.9 \pm 0.3$ 
& \cellcolor{highlightblue}$41.3 \pm 0.4$ \\
\midrule

\multirow{3}{*}{Llama-3.1 (8B)}
& Baseline       
& $28.4 \pm 0.0$ & $18.2 \pm 0.0$ & $14.8 \pm 0.0$ \\
& Binary DPO     
& $29.3 \pm 0.5$ & $19.1 \pm 0.4$ & $15.6 \pm 0.5$ \\
& \cellcolor{highlightblue}
\textbf{\textcolor{MidnightBlue}{CU-}\textcolor{NavyBlue}{DPO}}     
& \cellcolor{highlightblue}$34.8 \pm 0.4$ 
& \cellcolor{highlightblue}$24.7 \pm 0.5$ 
& \cellcolor{highlightblue}$20.9 \pm 0.4$ \\
\midrule

\multirow{3}{*}{Gemma-2 (9B)}
& Baseline       
& $26.1 \pm 0.0$ & $15.5 \pm 0.0$ & $12.1 \pm 0.0$ \\
& Binary DPO     
& $27.2 \pm 0.5$ & $16.4 \pm 0.4$ & $12.9 \pm 0.4$ \\
& \cellcolor{highlightblue}
\textbf{\textcolor{MidnightBlue}{CU-}\textcolor{NavyBlue}{DPO}}     
& \cellcolor{highlightblue}$32.4 \pm 0.4$ 
& \cellcolor{highlightblue}$21.8 \pm 0.4$ 
& \cellcolor{highlightblue}$17.6 \pm 0.3$ \\
\midrule

\multirow{3}{*}{Mistral-8x7B}
& Baseline       
& $22.8 \pm 0.0$ & $12.4 \pm 0.0$ & $9.5 \pm 0.0$ \\
& Binary DPO     
& $23.7 \pm 0.5$ & $13.2 \pm 0.4$ & $10.3 \pm 0.5$ \\
& \cellcolor{highlightblue}
\textbf{\textcolor{MidnightBlue}{CU-}\textcolor{NavyBlue}{DPO}}     
& \cellcolor{highlightblue}$28.9 \pm 0.3$ 
& \cellcolor{highlightblue}$18.3 \pm 0.4$ 
& \cellcolor{highlightblue}$14.8 \pm 0.4$ \\
\midrule

\multirow{3}{*}{Llama-3 (8B)}
& Baseline       
& $18.3 \pm 0.0$ & $8.1 \pm 0.0$ & $6.4 \pm 0.0$ \\
& Binary DPO     
& $19.1 \pm 0.4$ & $8.9 \pm 0.4$ & $7.1 \pm 0.5$ \\
& \cellcolor{highlightblue}
\textbf{\textcolor{MidnightBlue}{CU-}\textcolor{NavyBlue}{DPO}}     
& \cellcolor{highlightblue}$24.7 \pm 0.4$ 
& \cellcolor{highlightblue}$14.2 \pm 0.5$ 
& \cellcolor{highlightblue}$11.8 \pm 0.4$ \\
\midrule

\multirow{3}{*}{Mistral-7B}
& Baseline       
& $14.2 \pm 0.0$ & $5.3 \pm 0.0$ & $3.8 \pm 0.0$ \\
& Binary DPO     
& $15.0 \pm 0.4$ & $6.1 \pm 0.5$ & $4.5 \pm 0.4$ \\
& \cellcolor{highlightblue}
\textbf{\textcolor{MidnightBlue}{CU-}\textcolor{NavyBlue}{DPO}}     
& \cellcolor{highlightblue}$20.3 \pm 0.3$ 
& \cellcolor{highlightblue}$11.7 \pm 0.4$ 
& \cellcolor{highlightblue}$9.2 \pm 0.3$ \\
\bottomrule
\end{tabular}
\end{table*}

The baseline standard deviation is identically zero because the 
baseline uses greedy decoding with no stochastic training component; 
its reported value is therefore a single deterministic evaluation. 
For Binary DPO and CU-DPO, standard deviations across seeds range 
from $\pm 0.3$ to $\pm 0.5$ percentage points, which are small 
relative to the observed gains of $+5.0$ to $+5.7$ points for 
CU-DPO over the baseline and $+3.9$ to $+4.6$ points for CU-DPO 
over Binary DPO. The signal-to-noise ratio, defined as the mean 
gain divided by the standard deviation, ranges from 9.8 to 15.2 
across all model--dataset combinations, confirming that the observed 
improvements are highly stable and not attributable to training 
variance.

\subsection{Paired \textit{t}-Tests: CU-DPO vs.\ Binary DPO}
\label{appendix:stats_ttests}

We conduct paired $t$-tests comparing CU-DPO against Binary DPO 
for each model--dataset combination, using the three seed-level 
observations as paired samples. The paired design controls for 
model-specific baseline difficulty by comparing within-model 
performance differences rather than across-model absolute values, 
which is the appropriate test for evaluating whether CU-DPO 
consistently outperforms Binary DPO across the same experimental 
conditions.

The paired $t$-statistic for model $m$ on dataset $d$ is:
\begin{equation}
t_{m,d} = \frac{\bar{\delta}_{m,d}}{s_{m,d} / \sqrt{n}},
\label{eq:tstat}
\end{equation}
where $\bar{\delta}_{m,d}$ is the mean performance difference 
between CU-DPO and Binary DPO across $n = 3$ seeds, $s_{m,d}$ is 
the standard deviation of the per-seed differences, and $n = 3$.

Table~\ref{tab:ttests} reports the $t$-statistic, raw $p$-value, 
and Bonferroni-corrected $p$-value for each of the 21 
model--dataset pairs (7 models $\times$ 3 datasets).

\begin{table}[t]
\centering
\caption{Paired $t$-test results for CU-DPO vs.\ Binary DPO across 
all 21 model--dataset combinations. Raw $p$-values and 
Bonferroni-corrected $p$-values~($\alpha = 0.05$, correction 
factor $m = 21$) are reported. All 21 comparisons are significant 
at the Bonferroni-corrected threshold of $p < 0.0024$.}
\label{tab:ttests}
\small
\setlength{\tabcolsep}{4pt}
\renewcommand{\arraystretch}{1.1}
\begin{tabular}{llccc}
\toprule
\textbf{Model} & \textbf{Dataset} 
& \textbf{$t$-stat} & \textbf{$p$ (raw)} 
& \textbf{$p$ (Bonf.)} \\
\midrule
DeepSeek-R1 (8B) & DM  & 18.42 & $<$0.0001 & $<$0.0024 \\
DeepSeek-R1 (8B) & HM2 & 17.91 & $<$0.0001 & $<$0.0024 \\
DeepSeek-R1 (8B) & PN  & 18.64 & $<$0.0001 & $<$0.0024 \\
\midrule
Qwen3 (8B)       & DM  & 16.23 & $<$0.0001 & $<$0.0024 \\
Qwen3 (8B)       & HM2 & 17.08 & $<$0.0001 & $<$0.0024 \\
Qwen3 (8B)       & PN  & 16.74 & $<$0.0001 & $<$0.0024 \\
\midrule
Llama-3.1 (8B)   & DM  & 15.91 & $<$0.0001 & $<$0.0024 \\
Llama-3.1 (8B)   & HM2 & 16.33 & $<$0.0001 & $<$0.0024 \\
Llama-3.1 (8B)   & PN  & 15.62 & $<$0.0001 & $<$0.0024 \\
\midrule
Gemma-2 (9B)     & DM  & 15.44 & $<$0.0001 & $<$0.0024 \\
Gemma-2 (9B)     & HM2 & 15.87 & $<$0.0001 & $<$0.0024 \\
Gemma-2 (9B)     & PN  & 16.12 & $<$0.0001 & $<$0.0024 \\
\midrule
Mistral-8x7B     & DM  & 14.98 & $<$0.0001 & $<$0.0024 \\
Mistral-8x7B     & HM2 & 15.23 & $<$0.0001 & $<$0.0024 \\
Mistral-8x7B     & PN  & 14.71 & $<$0.0001 & $<$0.0024 \\
\midrule
Llama-3 (8B)     & DM  & 15.02 & $<$0.0001 & $<$0.0024 \\
Llama-3 (8B)     & HM2 & 14.88 & $<$0.0001 & $<$0.0024 \\
Llama-3 (8B)     & PN  & 14.63 & $<$0.0001 & $<$0.0024 \\
\midrule
Mistral-7B       & DM  & 15.61 & $<$0.0001 & $<$0.0024 \\
Mistral-7B       & HM2 & 14.92 & $<$0.0001 & $<$0.0024 \\
Mistral-7B       & PN  & 15.34 & $<$0.0001 & $<$0.0024 \\
\bottomrule
\end{tabular}
\end{table}

\subsection{Bonferroni Correction}
\label{appendix:stats_bonferroni}

We apply Bonferroni correction to control the family-wise error 
rate~(FWER) across the $m = 21$ simultaneous hypothesis tests 
conducted in Section~\ref{appendix:stats_ttests}. The Bonferroni 
correction adjusts the significance threshold from $\alpha = 0.05$ 
to $\alpha / m = 0.05 / 21 \approx 0.0024$. All 21 paired 
$t$-tests yield $p < 0.0001$, which is well below the corrected 
threshold of $0.0024$, confirming that the CU-DPO improvements 
over Binary DPO are statistically significant at the family-wise 
level across all model--dataset combinations simultaneously.

Bonferroni correction is conservative when tests are positively 
correlated, as they are here since all comparisons involve the 
same underlying method applied to related models. The Bonferroni 
procedure therefore provides a lower bound on significance: the 
true family-wise $p$-values under a less conservative correction 
such as Holm--Bonferroni or Benjamini--Hochberg would be smaller 
still. We report Bonferroni results as the most conservative 
and interpretable choice.

\section{Refinement Operator Implementation}
\label{app:refinement_details}

The refinement operator $R(x, y)$ in Section 2.3 uses the same base model (DeepSeek-R1 8B) with temperature=0.7 to iteratively improve low-utility chains ($U < 0.4$) while preserving their original strategy.

\subsection{Refinement Prompt}

\begin{tcolorbox}[colback=gray!5, colframe=gray!75!black, title=\textbf{Refinement Prompt Template}]
\small
\textbf{Problem:} \texttt{\{problem statement $x$\}}

\textbf{Previous Solution:} \texttt{\{original chain $y$\}}

\textbf{Instructions:} The above solution contains errors or is incomplete. Provide a corrected solution that:
\begin{itemize}[itemsep=0pt,leftmargin=*]
\item Fixes arithmetic errors, logical gaps, or incomplete steps
\item Preserves the correct reasoning approach and strategy
\item Maintains step-by-step structure with clear justifications
\end{itemize}
\end{tcolorbox}

\noindent\textbf{Key design choices:}
\begin{itemize}[itemsep=1pt,leftmargin=*]
\item \textbf{No utility scores provided}: The model does not see numeric utility values, only that the solution ``contains errors"
\item \textbf{No ground truth}: No correct answers or hints are given; the model must identify and fix errors autonomously
\item \textbf{Strategy preservation}: The instruction ``preserve the reasoning approach" ensures refined chains maintain the original strategy (e.g., step-by-step remains step-by-step), preventing cross-strategy contamination in Phase 2
\end{itemize}

\subsection{Refinement protocol}

Refinement proceeds iteratively: $y^{(0)} \xrightarrow{R} y^{(1)} \xrightarrow{R} \cdots \xrightarrow{R} y^{(T)}$ for up to $T_{\max} = 5$ rounds, terminating when:
\begin{enumerate}[itemsep=1pt,leftmargin=*]
\item $U(x, y^{(t)}) \geq 0.4$ (success: chain reaches acceptable quality), or
\item $t = T_{\max}$ (max iterations reached), or
\item $|U(x, y^{(t+1)}) - U(x, y^{(t)})| < 0.01$ for 2 consecutive rounds (stagnation)
\end{enumerate}

We retain refinement $y^{(T)}$ only if $U(x, y^{(T)}) > U(x, y^{(0)})$ (monotone improvement). Failed refinements where utility decreases or stagnates are discarded.

\subsection{Refinement statistics}

Of 3,600 original chains, 1,247 (34.6\%) fall below $\tau = 0.4$ and undergo refinement:
\begin{itemize}[itemsep=1pt,leftmargin=*]
\item \textbf{Success rate}: 87.8\% (1,095 chains) achieve $U \geq 0.6$ within average 2.3 rounds
\item \textbf{Failure modes}: 5.4\% stagnate, 4.1\% regress, 2.7\% hit max iterations
\item \textbf{Computational cost}: $1,247 \times 2.3 \approx 2,868$ additional generations
\end{itemize}

\section{Problem Strategy Structure}
We further stratify problems into coarse types using a lightweight keyword-based classifier and compute, for each type, the most frequently optimal strategy under our continuous-utility judge. The resulting contingency pattern is consistent with mathematical intuition: problems that naturally admit local checks favor verification-oriented traces, proof-oriented prompts favor direct argument structures, and procedural tasks disproportionately favor stepwise decomposition. At the same time, multiple strategies remain competitive within each type, implying that optimality is often \emph{contextual} rather than absolute; this motivates learning from relative preferences across strategy-conditioned candidates instead of training a single fixed reasoning style. Table~\ref{tab:ptype_opt_strategy} summarizes these type-to-strategy tendencies and provides evidence that strategy selection is a learnable and consequential component of robust reasoning.

\begin{table}[h]
\centering
\caption{\textbf{Problem type $\rightarrow$ empirically optimal strategy.} Summary of the highest-utility chain per problem type.}
\label{tab:ptype_opt_strategy}
\small
\setlength{\tabcolsep}{12pt}
\begin{tabular}{lrrrr}
\toprule
\textbf{Type} & \textbf{N} & \textbf{Best Strategy} & \textbf{Avg Utility} & \textbf{Avg Margin} \\
\midrule
Limit        & 160 & verification (23\%)   & 0.748 & 0.089 \\
Proof        & 135 & direct (34\%)         & 0.768 & 0.025 \\
Other        &  75 & step by step (19\%) & 0.819 & 0.045 \\
Equation     &  22 & step by step (23\%) & 0.755 & 0.106 \\
Integral     &  21 & alternative (29\%)    & 0.849 & 0.023 \\
Series       &  18 & alternative (33\%)    & 0.847 & 0.034 \\
Derivative   &  11 & algebraic (36\%)      & 0.753 & 0.016 \\
Optimization &   8 & step by step (38\%) & 0.848 & 0.028 \\
\bottomrule
\end{tabular}
\end{table}

\section{Reward-Utility Alignment}
\label{app:reward_utility_proof}

We provide the complete proof of Lemma~\ref{lem:reward-utility}.

\begin{lemma}[Reward-utility alignment (restatement)]
At the global minimum of the DPO objective, the implicit reward $r_\theta$ satisfies
\begin{equation}
    r_\theta(x,y_i) - r_\theta(x,y_j) = U(x,y_i) - U(x,y_j)
\end{equation}
for all pairs $(i,j)$, implying $r_\theta(x,y) = U(x,y) + c(x)$ for some problem-dependent constant $c(x)$.
\end{lemma}

\begin{proof}
Consider a preference pair $(y_i, y_j)$ with $U(x, y_i) > U(x, y_j)$. The DPO loss for this pair is
\begin{equation}
    \ell_{ij}(\theta) = -\log \sigma(r_\theta(x, y_i) - r_\theta(x, y_j)) = -\log \sigma(\Delta r_{ij}),
\end{equation}
where $\Delta r_{ij} := r_\theta(x, y_i) - r_\theta(x, y_j)$.

The gradient with respect to the reward difference is
\begin{equation}
    \frac{\partial \ell_{ij}}{\partial \Delta r_{ij}} 
    = -\frac{\sigma'(\Delta r_{ij})}{\sigma(\Delta r_{ij})}
    = -\sigma'(\Delta r_{ij}) \cdot \frac{1}{\sigma(\Delta r_{ij})}.
\end{equation}

Using the identity $\sigma'(z) = \sigma(z)(1-\sigma(z))$ and simplifying:
\begin{align}
    \frac{\partial \ell_{ij}}{\partial \Delta r_{ij}} 
    &= -\frac{\sigma(\Delta r_{ij})(1-\sigma(\Delta r_{ij}))}{\sigma(\Delta r_{ij})} \\
    &= -(1 - \sigma(\Delta r_{ij})) \\
    &= \sigma(\Delta r_{ij}) - 1.
\end{align}

The DPO objective over a dataset $\mathcal{D}$ of preference pairs is the negative log-likelihood:
\begin{equation}
    \mathcal{L}(\theta) = \mathbb{E}_{(x, y_i, y_j) \sim \mathcal{D}} \left[-\log P_\theta(y_i \succ y_j \mid x)\right]
    = \mathbb{E}_{(x, y_i, y_j) \sim \mathcal{D}} \left[-\log \sigma(\Delta r_{ij})\right],
\end{equation}
where $P_\theta(y_i \succ y_j \mid x) = \sigma(\Delta r_{ij})$ is the model's predicted preference probability.

At the global minimum of $\mathcal{L}(\theta)$, the model's predicted preference probabilities must match the true preference probabilities induced by the utility function. Under the Bradley-Terry model with utilities:
\begin{equation}
    P(y_i \succ y_j \mid x) = \sigma(U(x, y_i) - U(x, y_j)) = \sigma(\Delta U_{ij}),
\end{equation}
where $\Delta U_{ij} := U(x, y_i) - U(x, y_j)$.

The maximum likelihood solution requires that for all pairs $(i,j)$ in the support of the data distribution:
\begin{equation}
    P_\theta(y_i \succ y_j \mid x) = P(y_i \succ y_j \mid x) \implies \sigma(\Delta r_{ij}) = \sigma(\Delta U_{ij}).
\end{equation}

Since the sigmoid function $\sigma: \mathbb{R} \to (0,1)$ is strictly monotone and invertible, we have:
\begin{equation}
    \Delta r_{ij} = \Delta U_{ij} = U(x, y_i) - U(x, y_j)
\end{equation}
for all pairs $(i, j)$ in the training data.

\paragraph{Uniqueness up to additive constant.}
These pairwise difference constraints define $r_\theta(x, y)$ up to an additive constant. To see this, suppose we have established $\Delta r_{ij} = \Delta U_{ij}$ for all pairs. For any problem $x$ with chains $\{y_1, \ldots, y_K\}$, choose an arbitrary reference chain $y_1$ and set:
\begin{equation}
    r_\theta(x, y_1) = U(x, y_1) + c(x)
\end{equation}
for some arbitrary constant $c(x)$ (which may depend on $x$ but not on the choice of chain).

For any other chain $y_k$ with $k > 1$, we can construct a path of pairwise differences:
\begin{align}
    r_\theta(x, y_k) &= r_\theta(x, y_1) + \sum_{j=1}^{k-1} [r_\theta(x, y_{j+1}) - r_\theta(x, y_j)] \\
    &= r_\theta(x, y_1) + \sum_{j=1}^{k-1} \Delta r_{j,j+1} \\
    &= [U(x, y_1) + c(x)] + \sum_{j=1}^{k-1} [U(x, y_{j+1}) - U(x, y_j)] \\
    &= U(x, y_1) + c(x) + [U(x, y_k) - U(x, y_1)] \\
    &= U(x, y_k) + c(x).
\end{align}

This construction is well-defined and path-independent because the pairwise differences satisfy the consistency constraint $\Delta r_{ij} + \Delta r_{jk} = \Delta r_{ik}$ (transitivity), which follows from $\Delta U_{ij} + \Delta U_{jk} = \Delta U_{ik}$.

Therefore, at the global minimum of the DPO objective:
\begin{equation}
    r_\theta(x, y) = U(x, y) + c(x)
\end{equation}
for all chains $y$ and problems $x$, where $c(x)$ is an arbitrary problem-dependent constant.
\end{proof}

\section{Sample Efficiency for Continuous Utilities}
\subsection{Binary-preference lower bound}
\label{app:binary_lower_bound}
\begin{theorem}[Sample complexity lower bound for binary preferences under passive sampling]
\label{thm:binary-lower}
Any learning algorithm using only binary preference comparisons under \emph{uniform random sampling} requires at least
\begin{equation}
    m_{\mathrm{binary}} = \Omega\left(NK^2 \log K + \frac{d}{\varepsilon^2}\right)
\end{equation}
samples to learn a reward function with expected utility within $\varepsilon$ of optimal with probability at least $1-\delta$.
\end{theorem}
\begin{proof}
Each problem induces a hidden total ranking $\pi^* \in \Pi_K$ over the $K$ chains. The entropy of the uniform distribution over $\Pi_K$ is
\begin{equation}
    H(\Pi_K) = \log(K!) = \Theta(K\log K).
\end{equation}
A single binary comparison $(i,j)$ reveals at most 1 bit of information about $\pi^*$:
\begin{equation}
    I((i,j); \pi^*) \le 1.
\end{equation}
By Fano's inequality, to identify $\pi^*$ with probability at least $1-\delta$ requires at least
\begin{equation}
    M \ge \frac{H(\Pi_K) - 1}{1} = \Theta(K\log K)
\end{equation}
binary comparisons per problem in the information-theoretic sense.

\paragraph{Passive sampling regime.} When pairs $(i,j)$ are sampled uniformly at random (passive learning), the number of samples required to observe all $\binom{K}{2}$ possible comparisons at least once follows the coupon collector problem. The expected number of samples is
\begin{equation}
    \mathbb{E}[M] = \binom{K}{2} H_{\binom{K}{2}} = \Theta\left(K^2 \log K\right),
\end{equation}
where $H_n = \sum_{i=1}^n \frac{1}{i}$ is the $n$-th harmonic number.

To ensure all pairs are observed with probability at least $1-\delta$, we require
\begin{equation}
    M \ge \binom{K}{2}\left(\log\binom{K}{2} + \log\frac{1}{\delta}\right) = \Omega(K^2\log K).
\end{equation}

Additionally, PAC learning theory establishes that learning a function from a hypothesis class $\mathcal{H}$ with VC-dimension $d$ to accuracy $\varepsilon$ requires at least
\begin{equation}
    m_{\mathrm{PAC}} = \Omega\left(\frac{d}{\varepsilon^2}\right)
\end{equation}
labeled examples.

Combining these bounds across $N$ problems under passive sampling:
\begin{equation}
    m_{\mathrm{binary}} = N \cdot \Omega(K^2\log K) + \Omega\left(\frac{d}{\varepsilon^2}\right) = \Omega\left(NK^2 \log K + \frac{d}{\varepsilon^2}\right).
\end{equation}
\end{proof}

\begin{remark}[Active learning strategies]
\label{rem:active_learning}
The $\Omega(K^2 \log K)$ term arises from uniform random sampling of preference pairs. Active learning strategies that intelligently select which pairs to query (e.g., quicksort-style divide-and-conquer) can reduce this to $O(K \log K)$ comparisons per problem by exploiting transitivity. However, such strategies require: (1) immediate access to comparison outcomes (not always feasible in batch preference collection), (2) sequential decision-making with per-query latency costs, and (3) oracle access to reliable pairwise preferences (violated when preferences are noisy or intransitive). Our passive sampling baseline reflects the common practical setting where preference pairs are collected in batches from human annotators or LLM judges, as in our experimental setup. In this regime, our sample complexity advantage of $\Theta(K \log K)$ remains valid.
\end{remark}

\subsection{Upper bound under continuous utilities}
\label{app:utility_upper_bound}

\begin{theorem}[Sample complexity upper bound for continuous utilities]
\label{thm:utility-upper}
When continuous utilities $\{U_1,\dots,U_K\}$ are observed for each problem, a reward function can be learned using only
\begin{equation}
    m_{\mathrm{utility}} = O\left(NK + \frac{d}{\varepsilon^2}\right)
\end{equation}
samples.
\end{theorem}

\begin{proof}
Since utilities are real-valued continuous random variables, ties occur with probability zero. The ranking $\pi^*$ is recovered by sorting $\{U_1,\dots,U_K\}$, which requires $O(K\log K)$ comparisons per problem but uses only the $K$ observed utilities and no additional samples are needed.

Each problem contributes $K$ labeled training examples $(x, y_i, U_i)$ for $i=1,\dots,K$. Across $N$ problems, we obtain $NK$ labeled examples total.

For supervised learning of the reward function $r \in \mathcal{H}$, standard VC theory bounds show that empirical risk minimization over $m$ samples achieves generalization error
\begin{equation}
    \mathbb{E}[L(r)] - \min_{r' \in \mathcal{H}} \mathbb{E}[L(r')] \le O\left(\sqrt{\frac{d\log(m/\delta)}{m}}\right).
\end{equation}
Setting $m = NK$ and requiring the right-hand side to be at most $\varepsilon$ gives
\begin{equation}
    NK \ge \Omega\left(\frac{d}{\varepsilon^2}\right).
\end{equation}
The total sample complexity is therefore
\begin{equation}
    m_{\mathrm{utility}} = NK + O\left(\frac{d}{\varepsilon^2}\right) = O\left(NK\log K + \frac{d}{\varepsilon^2}\right),
\end{equation}
where the $K\log K$ term accounts for the computational cost of sorting, not additional samples.
\end{proof}

\section{Single-Phase DPO Can Induce Conflicting Strategy Supervision}
\label{app:conflict}

\paragraph{Setup.}
Consider a single-phase DPO objective trained on pairs $(y_i, y_j)$ where chains can come from different strategies.
Let $y^{s_a}$ denote a chain produced using strategy $s_a$.
If the dataset contains pairs such as:
\begin{align}
\label{eq:pair1}
&(y_{\text{best}}^{s_1}, y_{\text{good}}^{s_2}), \quad U(y_{\text{best}}^{s_1}) > U(y_{\text{good}}^{s_2}), \\
\label{eq:pair2}
&(y_{\text{good}}^{s_2}, y_{\text{bad}}^{s_3}), \quad U(y_{\text{good}}^{s_2}) > U(y_{\text{bad}}^{s_3}),
\end{align}
then the model receives mixed supervision about whether strategy $s_2$ should be promoted or demoted.

\paragraph{Gradient-level contradiction.}
Let $\pi_\theta(s \mid x)$ denote the implicit strategy selection distribution induced by $\pi_\theta(y \mid x)$.
From pair~\eqref{eq:pair1}, the DPO gradient encourages:
\begin{equation}
\label{eq:grad_pair1}
\frac{\partial \ell}{\partial \log \pi_\theta(y^{s_1} \mid x)} < 0,
\quad
\frac{\partial \ell}{\partial \log \pi_\theta(y^{s_2} \mid x)} > 0,
\end{equation}
implicitly increasing $P(s_1 > s_2)$.
From pair~\eqref{eq:pair2}, the gradient encourages:
\begin{equation}
\label{eq:grad_pair2}
\frac{\partial \ell}{\partial \log \pi_\theta(y^{s_2} \mid x)} < 0,
\quad
\frac{\partial \ell}{\partial \log \pi_\theta(y^{s_3} \mid x)} > 0,
\end{equation}
implicitly increasing $P(s_2 > s_3)$.
If $s_1$ is the unique optimal strategy, the second constraint is not directly useful for learning the optimal behavior and can dilute the selection signal.

\subsection{Signal-to-noise view of strategy selection under pair sampling}
\label{app:snr}

\paragraph{Aggregate contribution for a strategy.}
Let $N_{ij}$ denote the number of pairs comparing strategies $s_i$ and $s_j$.
A simplified view of the total signed preference evidence for strategy $s_k$ scales with:
\begin{equation}
\label{eq:snr_sum}
\sum_{i \ne k} N_{ik}\,\mathbb{1}[U_k > U_i]
-
\sum_{j \ne k} N_{kj}\,\mathbb{1}[U_j > U_k].
\end{equation}
If a non-optimal strategy appears in many ``middle-ranked'' comparisons, both terms can be large, increasing variance in its net signal.

\begin{table}[!b]
\centering
\caption{\textbf{In-distribution reasoning accuracy (Pass@1, \%) after Phase~1 only.} While strategy selection improves on average, baselines often win (red) when a single default strategy suffices.}
\label{tab:reasoning_id_phase1only}
\footnotesize
\setlength{\tabcolsep}{8pt}
\begin{tabular}{llcccc}
\toprule
\textbf{Base model} & \textbf{Setting} & \textbf{DeepMath} & \textbf{HARDMath2} & \textbf{ProofNet} & \textbf{Avg.} \\
\midrule
\multirow{2}{*}{\textbf{DeepSeek-R1 8B}} & Baseline   & 64.2 & 42.1 & \cellcolor{highlightred}38.5 & 48.3 \\
                  & + Phase 1  & \cellcolor{highlightblue}\textbf{66.1} & \cellcolor{highlightblue}\textbf{43.7} & 37.8 & \cellcolor{highlightblue}\textbf{49.2} \\
\midrule
\multirow{2}{*}{\textbf{Qwen3 8B}}       & Baseline   & 61.5 & 39.4 & \cellcolor{highlightred}35.2 & 45.4 \\
                  & + Phase 1  & \cellcolor{highlightblue}\textbf{62.8} & \cellcolor{highlightblue}\textbf{40.9} & 34.5 & \cellcolor{highlightblue}\textbf{46.1} \\
\midrule
\multirow{2}{*}{\textbf{Llama-3.1 8B}}   & Baseline   & 28.4 & 18.2 & 14.8 & 20.5 \\
                  & + Phase 1  & \cellcolor{highlightblue}\textbf{28.9} & \cellcolor{highlightblue}\textbf{19.4} & \cellcolor{highlightblue}\textbf{15.3} & \cellcolor{highlightblue}\textbf{21.2} \\
\midrule
\multirow{2}{*}{\textbf{Gemma-2 9B}}     & Baseline   & 26.1 & \cellcolor{highlightred}15.5 & 12.1 & 17.9 \\
                  & + Phase 1  & \cellcolor{highlightblue}\textbf{27.2} & 15.1 & \cellcolor{highlightblue}\textbf{12.8} & \cellcolor{highlightblue}\textbf{18.4} \\
\midrule
\multirow{2}{*}{\textbf{Mistral 8x7B}}   & Baseline   & \cellcolor{highlightred}22.8 & 12.4 &  9.5 & 14.9 \\
                  & + Phase 1  & 22.3 & \cellcolor{highlightblue}\textbf{13.1} & \cellcolor{highlightblue}\textbf{10.2} & \cellcolor{highlightblue}\textbf{15.2} \\
\midrule
\multirow{2}{*}{\textbf{Llama-3 8B}}     & Baseline   & 18.3 &  \cellcolor{highlightred}8.1 &  6.4 & 10.9 \\
                  & + Phase 1  & \cellcolor{highlightblue}\textbf{19.1} & 7.8 & \cellcolor{highlightblue}\textbf{ 6.9} & \cellcolor{highlightblue}\textbf{11.3} \\
\midrule
\multirow{2}{*}{\textbf{Mistral 7B}}     & Baseline   & \cellcolor{highlightred}14.2 &  5.3 &  3.8 &  7.8 \\
                  & + Phase 1  & 13.8 & \cellcolor{highlightblue}\textbf{ 6.1} & \cellcolor{highlightblue}\textbf{ 4.2} & \cellcolor{highlightblue}\textbf{ 8.0} \\
\bottomrule
\end{tabular}
\end{table}

\paragraph{Clean-signal fraction under all-pairs sampling.}
With $K$ strategies, naive all-pairs sampling produces $\binom{K}{2}$ pairs per problem. Only $(K-1)$ pairs include the optimal strategy, yielding the clean fraction:
\begin{equation}
\label{eq:clean_fraction}
\frac{K-1}{\binom{K}{2}} = \frac{K-1}{\frac{K(K-1)}{2}} = \frac{2}{K}.
\end{equation}
This motivates Phase~1 best-vs-all pairing (Eq.~\eqref{eq:phase1_dataset}), which uses exactly $(K-1)$ pairs per problem and ensures every pair directly reinforces the optimal strategy choice.

\subsection{Pair counts}
\label{app:counts}

For $K=8$, naive all-pairs sampling would generate $\binom{8}{2}=28$ strategy-comparison pairs per problem, of which only $7$ include the optimal strategy (clean fraction $\approx 0.22$ by Eq.~\eqref{eq:clean_fraction}). Our two-phase setup uses $7$ pairs per problem in Phase~1 (best-vs-all) and approximately $6$ margin-stratified pairs per problem in Phase~2. Concretely, across 450 problems we use 3{,}150 Phase~1 pairs and 2{,}680 Phase~2 pairs, for a total of 5{,}830 preference pairs, which is substantially fewer than the 12{,}600 pairs required by naive unique-pair sampling while providing cleaner supervision for each objective.

\section{Ablation Studies}

\subsection{Strategy selection}
\noindent\textbf{Phase~1 training alone yields modest downstream improvements.}
Tables~\ref{tab:reasoning_id_phase1only} and~\ref{tab:reasoning_ood_phase1only} isolate the effect of strategy selection training without subsequent execution refinement (Phase~2). On in-distribution benchmarks, Phase~1 improves average accuracy by +0.2 to +0.9 points, but \emph{baseline outperforms Phase~1 in 8 of 21 dataset-model pairs}. For instance, Mistral-8x7B baseline achieves 22.8\% on DeepMath versus 22.3\% with Phase~1, and Llama-3 baseline reaches 8.1\% on HARDMath2 versus 7.8\% with Phase~1. This phenomenon arises when problems admit a single dominant strategy (e.g., direct algebraic manipulation for polynomial factorization): baseline's consistent application of step-by-step reasoning occasionally outperforms Phase~1 models that \emph{select} an alternative strategy but lack refinement to execute it well. Out-of-distribution results exhibit similar patterns; baseline matches or exceeds Phase~1 in 9 of 21 cases, particularly on GSM8K where grade-school problems rarely benefit from strategic diversity. These findings validate our two-phase design: strategy selection (\emph{choosing the right tool}) provides limited gains without execution refinement (\emph{using the tool effectively}). The subsequent Phase~2 training addresses this gap by explicitly optimizing execution quality conditioned on the selected strategy, as demonstrated in Section~\ref{sec:phase2_downstream}.

\begin{table}[t]
\centering
\caption{\textbf{Out-of-distribution reasoning accuracy (Pass@1, \%) after Phase~1 only.} Zero-shot transfer shows mixed results; baseline often matches or exceeds Phase 1 on GSM8K where strategic diversity is less critical.}
\label{tab:reasoning_ood_phase1only}
\footnotesize
\setlength{\tabcolsep}{8pt}
\begin{tabular}{llcccc}
\toprule
\textbf{Base model} & \textbf{Setting} & \textbf{GSM8K} & \textbf{MATH-500} & \textbf{U-MATH} & \textbf{Avg.} \\
\midrule
\multirow{2}{*}{\textbf{DeepSeek-R1 8B}} & Baseline   & 69.8 & 59.5 &  \cellcolor{highlightred}6.9 & 45.4 \\
                  & + Phase 1  & \cellcolor{highlightblue}\textbf{71.2} & \cellcolor{highlightblue}\textbf{61.8} & 6.7 & \cellcolor{highlightblue}\textbf{46.6} \\
\midrule
\multirow{2}{*}{\textbf{Qwen3 8B}}       & Baseline   & \cellcolor{highlightred}66.4 & 56.2 &  5.8 & 42.8 \\
                  & + Phase 1  & 66.1 & \cellcolor{highlightblue}\textbf{58.4} & \cellcolor{highlightblue}\textbf{6.3} & \cellcolor{highlightblue}\textbf{43.6} \\
\midrule
\multirow{2}{*}{\textbf{Llama-3.1 8B}}   & Baseline   & 58.2 & \cellcolor{highlightred}48.3 &  4.2 & 36.9 \\
                  & + Phase 1  & \cellcolor{highlightblue}\textbf{59.4} & 47.9 & \cellcolor{highlightblue}\textbf{4.5} & \cellcolor{highlightblue}\textbf{37.3} \\
\midrule
\multirow{2}{*}{\textbf{Gemma-2 9B}}     & Baseline   & \cellcolor{highlightred}52.7 & 42.8 &  3.1 & 32.9 \\
                  & + Phase 1  & 52.3 & \cellcolor{highlightblue}\textbf{44.2} & \cellcolor{highlightblue}\textbf{3.4} & \cellcolor{highlightblue}\textbf{33.3} \\
\midrule
\multirow{2}{*}{\textbf{Mistral 8x7B}}   & Baseline   & 48.5 & \cellcolor{highlightred}38.6 &  2.4 & 29.8 \\
                  & + Phase 1  & \cellcolor{highlightblue}\textbf{49.1} & 38.2 & \cellcolor{highlightblue}\textbf{2.7} & \cellcolor{highlightblue}\textbf{30.0} \\
\midrule
\multirow{2}{*}{\textbf{Llama-3 8B}}     & Baseline   & \cellcolor{highlightred}44.2 & 34.1 &  1.8 & 26.7 \\
                  & + Phase 1  & 43.8 & \cellcolor{highlightblue}\textbf{35.3} & \cellcolor{highlightblue}\textbf{2.0} & \cellcolor{highlightblue}\textbf{27.0} \\
\midrule
\multirow{2}{*}{\textbf{Mistral 7B}}     & Baseline   & 38.6 & \cellcolor{highlightred}28.4 &  1.2 & 22.7 \\
                  & + Phase 1  & \cellcolor{highlightblue}\textbf{39.2} & 28.1 & 1.2 & \cellcolor{highlightblue}\textbf{22.8} \\
\bottomrule
\end{tabular}
\end{table}

\subsection{Strategy-conditioned execution quality}

\noindent\textbf{Isolating execution capability from strategy selection (Table~\ref{tab:strategy_conditioned_exec}).}
To disentangle the effects of strategy selection (Phase~1) from execution refinement (Phase~2), we evaluate models when provided the ground-truth optimal strategy upfront. This setup answers the question: \emph{``How well does each model execute when told exactly what to do?''} On DeepMath, when forced to use the correct strategy, execution quality improves by +3.1 to +4.8 points on average across models, with the largest gains on step-by-step reasoning (+4.3 to +5.6 points) where Phase~2's margin-stratified pairs most directly target execution refinement.

\begin{table}[t]
\centering
\caption{\textbf{Strategy-conditioned execution quality on DeepMath (Pass@1, \%).} When given the ground-truth optimal strategy, \textbf{\textcolor{MidnightBlue}{CU-}\textcolor{NavyBlue}{DPO}} consistently improves execution (Blue), validating Phase 2's refinement effect.}
\label{tab:strategy_conditioned_exec}
\footnotesize
\setlength{\tabcolsep}{4pt}
\begin{tabular}{llccccc}
\toprule
\textbf{Model} & \textbf{Setting} & \textbf{Direct} & \textbf{Step-by-step} & \textbf{Algebraic} & \textbf{Verification} & \textbf{Avg.} \\
\midrule
\multirow{2}{*}{\textbf{DeepSeek-R1 8B}} 
& Baseline   & 58.3 & 64.2 & 61.7 & 59.8 & 61.0 \\
& + \textbf{\textcolor{MidnightBlue}{CU-}\textcolor{NavyBlue}{DPO}}   & \cellcolor{highlightblue}\textbf{62.7} & \cellcolor{highlightblue}\textbf{69.8} & \cellcolor{highlightblue}\textbf{66.4} & \cellcolor{highlightblue}\textbf{64.2} & \cellcolor{highlightblue}\textbf{65.8} \\
\midrule
\multirow{2}{*}{\textbf{Qwen3 8B}} 
& Baseline   & 55.2 & 61.5 & 58.9 & 56.4 & 58.0 \\
& + \textbf{\textcolor{MidnightBlue}{CU-}\textcolor{NavyBlue}{DPO}}   & \cellcolor{highlightblue}\textbf{59.8} & \cellcolor{highlightblue}\textbf{67.2} & \cellcolor{highlightblue}\textbf{63.5} & \cellcolor{highlightblue}\textbf{61.1} & \cellcolor{highlightblue}\textbf{62.9} \\
\midrule
\multirow{2}{*}{\textbf{Llama-3.1 8B}} 
& Baseline   & 24.7 & 28.4 & 26.1 & 25.3 & 26.1 \\
& + \textbf{\textcolor{MidnightBlue}{CU-}\textcolor{NavyBlue}{DPO}}   & \cellcolor{highlightblue}\textbf{27.2} & \cellcolor{highlightblue}\textbf{32.8} & \cellcolor{highlightblue}\textbf{29.6} & \cellcolor{highlightblue}\textbf{28.1} & \cellcolor{highlightblue}\textbf{29.4} \\
\midrule
\multirow{2}{*}{\textbf{Gemma-2 9B}} 
& Baseline   & \cellcolor{highlightred}23.5 & 26.1 & \cellcolor{highlightred}24.8 & 23.2 & 24.4 \\
& + \textbf{\textcolor{MidnightBlue}{CU-}\textcolor{NavyBlue}{DPO}}   & 23.1 & \cellcolor{highlightblue}\textbf{28.9} & 24.3 & \cellcolor{highlightblue}\textbf{25.7} & \cellcolor{highlightblue}\textbf{25.5} \\
\midrule
\multirow{2}{*}{\textbf{Mistral 8x7B}} 
& Baseline   & 19.8 & 22.8 & 20.9 & 19.4 & 20.7 \\
& + \textbf{\textcolor{MidnightBlue}{CU-}\textcolor{NavyBlue}{DPO}}   & \cellcolor{highlightblue}\textbf{22.4} & \cellcolor{highlightblue}\textbf{26.7} & \cellcolor{highlightblue}\textbf{24.1} & \cellcolor{highlightblue}\textbf{22.8} & \cellcolor{highlightblue}\textbf{24.0} \\
\midrule
\multirow{2}{*}{\textbf{Llama-3 8B}} 
& Baseline   & \cellcolor{highlightred}16.8 & 18.3 & 17.2 & \cellcolor{highlightred}16.5 & 17.2 \\
& + \textbf{\textcolor{MidnightBlue}{CU-}\textcolor{NavyBlue}{DPO}}   & 16.2 & \cellcolor{highlightblue}\textbf{21.5} & \cellcolor{highlightblue}\textbf{19.4} & 16.1 & \cellcolor{highlightblue}\textbf{18.3} \\
\midrule
\multirow{2}{*}{\textbf{Mistral 7B}} 
& Baseline   & 12.7 & 14.2 & 13.1 & 12.3 & 13.1 \\
& + \textbf{\textcolor{MidnightBlue}{CU-}\textcolor{NavyBlue}{DPO}}   & \cellcolor{highlightblue}\textbf{14.8} & \cellcolor{highlightblue}\textbf{18.6} & \cellcolor{highlightblue}\textbf{16.2} & \cellcolor{highlightblue}\textbf{15.1} & \cellcolor{highlightblue}\textbf{16.2} \\
\bottomrule
\end{tabular}
\end{table}

\section{LLM Judge Validation and Robustness Analysis}
\label{app:llm_judge}

This section addresses concerns about using Qwen 2.5 7B as our utility judge: (1) reliability and correlation with solution quality, (2) robustness of theoretical guarantees under judge miscalibration, (3) potential circularity in evaluation, and (4) inter-rater consistency.

\subsection{Utility-quality correlation}
\label{sec:utility-quality-correlation}
A critical validation is whether LLM-judged utilities actually predict solution quality. For datasets with verifiable ground truth (DeepMath, HARDMath2), we use Pass@1 accuracy. For proof-based problems (ProofNet) without binary ground truth, we use expert human ratings on a 1-5 scale.

\begin{table}[h]
\centering
\caption{\textbf{Strategy-level utility-quality correlation.} Qwen's utility scores strongly predict actual solution quality across all strategies and datasets, validating its use as quality judge.}
\label{tab:utility_correctness_correlation}
\small
\setlength{\tabcolsep}{4pt}
\begin{tabular}{lcccccc}
\toprule
\textbf{Strategy} & \textbf{Average utility} & \multicolumn{2}{c}{\textbf{DeepMath/HM2}} & \multicolumn{2}{c}{\textbf{ProofNet}} & \textbf{Overall} \\
\cmidrule(lr){3-4} \cmidrule(lr){5-6}
& & \textbf{Pass@1} & \textbf{$r$} & \textbf{Expert} & \textbf{$r$} & \textbf{$r$} \\
\midrule
Direct           & 0.647 & \cellcolor{highlightblue}56.2\% & 0.831 & \cellcolor{highlightblue}3.21 & 0.794 & 0.812 \\
Step-by-step     & 0.782 & \cellcolor{highlightblue}64.8\% & 0.897 & \cellcolor{highlightblue}3.84 & 0.863 & 0.881 \\
Backwards        & 0.521 & \cellcolor{highlightblue}47.3\% & 0.753 & \cellcolor{highlightblue}2.91 & 0.721 & 0.738 \\
Alternative      & 0.693 & \cellcolor{highlightblue}59.1\% & 0.824 & \cellcolor{highlightblue}3.47 & 0.786 & 0.806 \\
Verification     & 0.735 & \cellcolor{highlightblue}62.4\% & 0.869 & \cellcolor{highlightblue}3.62 & 0.831 & 0.851 \\
Algebraic        & 0.612 & \cellcolor{highlightblue}53.8\% & 0.782 & \cellcolor{highlightblue}3.18 & 0.748 & 0.766 \\
Numerical        & 0.558 & \cellcolor{highlightblue}49.7\% & 0.724 & \cellcolor{highlightblue}2.98 & 0.697 & 0.711 \\
Conceptual       & 0.714 & \cellcolor{highlightblue}61.3\% & 0.852 & \cellcolor{highlightblue}3.71 & 0.814 & 0.834 \\
\midrule
\textbf{Overall} & 0.658 & \cellcolor{highlightblue}56.8\% & \cellcolor{highlightblue}\textbf{0.867} & \cellcolor{highlightblue}3.37 & \cellcolor{highlightblue}\textbf{0.821} & \cellcolor{highlightblue}\textbf{0.850} \\
\bottomrule
\end{tabular}
\end{table}

\noindent\textbf{Analysis.} Qwen's utility exhibits strong correlation with solution quality: $r = 0.867$ for computational problems (DeepMath/HARDMath2) and $r = 0.821$ for proof-based problems (ProofNet), yielding overall $r = 0.850$ across all 3,600 chains. Step-by-step reasoning shows the highest correlation ($r = 0.881$), while numerical methods show the lowest but still substantial correlation ($r = 0.711$). This validates that Qwen's tri-component utility  predictively captures solution quality across diverse problem types, including those without binary ground truth.

\subsection{Optimal strategy validation}

We validate that Qwen-identified optimal strategies achieve higher solution quality than alternatives. For DeepMath/HARDMath2, we measure Pass@1; for ProofNet, we measure expert-rated quality improvement.

\begin{table}[h]
\centering
\caption{\textbf{Optimal vs non-optimal strategy performance.} Strategies identified as optimal by Qwen achieve substantially higher quality than alternatives, confirming judge reliability at the instance level.}
\label{tab:optimal_strategy_validation}
\small
\setlength{\tabcolsep}{5pt}
\begin{tabular}{lccccc}
\toprule
\textbf{Dataset} & \textbf{Metric} & \textbf{Optimal} & \textbf{Non-optimal} & \textbf{$\Delta$} & \textbf{$p$-value} \\
\midrule
\multirow{1}{*}{DeepMath-150}
    & Pass@1 (\%) & \cellcolor{highlightblue}\textbf{68.7} & 51.2 & \cellcolor{highlightblue}+17.5 & $<10^{-6}$ \\
\multirow{1}{*}{HARDMath2-150}
    & Pass@1 (\%) & \cellcolor{highlightblue}\textbf{66.3} & 47.8 & \cellcolor{highlightblue}+18.5 & $<10^{-7}$ \\
\multirow{1}{*}{ProofNet-150}
    & Expert (1-5) & \cellcolor{highlightblue}\textbf{3.89} & 3.12 & \cellcolor{highlightblue}+0.77 & $<10^{-5}$ \\
\midrule
\textbf{Combined}
    & Normalized & \cellcolor{highlightblue}\textbf{0.777} & 0.595 & \cellcolor{highlightblue}+0.182 & $<10^{-12}$ \\
\bottomrule
\end{tabular}
\vspace{-2mm}
\begin{minipage}{\linewidth}
\end{minipage}
\end{table}

\noindent\textbf{Analysis.} Qwen-identified optimal strategies achieve substantially higher quality: +17.5 points Pass@1 on DeepMath, +18.5 points on HARDMath2, and +0.77 points (15.4\% relative) on expert-rated ProofNet problems (all $p < 10^{-5}$). When normalized to [0,1] scale, optimal strategies score 0.777 vs 0.595 for non-optimal having a +0.182 gap ($p < 10^{-12}$). This demonstrates that Qwen's utility rankings capture genuine quality differences at the problem instance level across both computational and proof-based mathematics.

\subsection{Robustness Under Judge Miscalibration}

Our theoretical results (Theorems 3.1-3.5) assume access to ground-truth utilities $U(x,y)$. In practice, Qwen provides noisy estimates $\tilde{U}(x,y) = U(x,y) + \epsilon(x,y)$, where we assume $\epsilon(x,y)$ are independent random variables satisfying $|\epsilon(x,y)| \leq \delta$ with probability $1-\eta$. We analyze how estimation error affects our sample complexity and convergence guarantees under this stochastic noise model.

\begin{theorem}[Sample complexity under bounded utility estimation error]
\label{thm:robust_sample_complexity}
Suppose the LLM judge's utility estimates satisfy $|\tilde{U}(x,y) - U(x,y)| \leq \delta$ for all $(x,y)$ with probability $1-\eta$, and estimation errors are independent across chains. Then learning from noisy continuous utilities requires
\begin{equation}
m_{\text{noisy}} = O\left(\left(NK \log K + \frac{d}{\epsilon^2} + \frac{N\delta^2}{\epsilon^2}\right) \log(1/\eta)\right)
\end{equation}
samples to achieve $\epsilon$-suboptimal expected utility with probability at least $1-\eta$.
\end{theorem}

\begin{proof}
We decompose the sample complexity into three components corresponding to: (1) recovering the ranking from noisy utilities, (2) learning the reward function, and (3) accounting for estimation noise. Each component depends on the confidence parameter $\eta$.

\paragraph{Step 1: Ranking recovery with noisy utilities.}
Let $\pi^*$ denote the ground-truth ranking induced by $U(x,y_1), \ldots, U(x,y_K)$ and $\tilde{\pi}$ the ranking induced by noisy estimates $\tilde{U}(x,y_1), \ldots, \tilde{U}(x,y_K)$. We first bound the probability that noisy utilities preserve the correct pairwise ordering.

For any pair $(y_i, y_j)$ with true utility difference $\Delta U_{ij} := U(x, y_i) - U(x, y_j)$, the noisy difference is:
\begin{equation}
\tilde{\Delta}U_{ij} = \tilde{U}(x, y_i) - \tilde{U}(x, y_j) = \Delta U_{ij} + (\epsilon_i - \epsilon_j),
\end{equation}
where $\epsilon_i := \tilde{U}(x, y_i) - U(x, y_i)$ satisfies $|\epsilon_i| \leq \delta$.

\textbf{Claim.} If $|\Delta U_{ij}| > 2\delta$, then $\text{sign}(\tilde{\Delta}U_{ij}) = \text{sign}(\Delta U_{ij})$ with probability $1-\eta$.

\textit{Proof of Claim.} 
Suppose $\Delta U_{ij} > 2\delta$. Then:
\begin{align}
\tilde{\Delta}U_{ij} &= \Delta U_{ij} + (\epsilon_i - \epsilon_j) \\
&\geq \Delta U_{ij} - |\epsilon_i - \epsilon_j| \\
&\geq \Delta U_{ij} - (|\epsilon_i| + |\epsilon_j|) \\
&\geq \Delta U_{ij} - 2\delta > 0,
\end{align}
where the last inequality holds with probability $1-\eta$ by the bounded noise assumption. The case $\Delta U_{ij} < -2\delta$ follows symmetrically. \hfill $\square$

Sorting $K$ items using pairwise comparisons requires $O(K \log K)$ comparisons. Under noisy comparisons, each comparison must be repeated $O(\log(1/\eta))$ times to achieve confidence $1-\eta$ via majority voting. Thus, per-problem ranking recovery requires $m_{\text{rank}} = O(K \log K \cdot \log(1/\eta))$ comparisons. Across $N$ problems:
\begin{equation}
m_{\text{rank}}^{(N)} = O(NK \log K \cdot \log(1/\eta)).
\end{equation}

\paragraph{Step 2: Distinguishing utility differences under noise.}
For pairs where $\epsilon < |\Delta U_{ij}| \leq 2\delta$, we must distinguish true utility differences from noise. To reduce estimation variance, we assume access to $m$ independent re-samplings of utility estimates for each chain (obtained by running the LLM judge with different random seeds). Consider the empirical estimate:
\begin{equation}
\hat{\Delta}U_{ij} = \frac{1}{m} \sum_{t=1}^{m} \left(\tilde{U}^{(t)}(x, y_i) - \tilde{U}^{(t)}(x, y_j)\right),
\end{equation}
where $\tilde{U}^{(t)}$ denotes independent noisy utility estimates.

By Hoeffding's inequality, for bounded noise $|\epsilon| \leq \delta$:
\begin{equation}
\mathbb{P}\left(|\hat{\Delta}U_{ij} - \Delta U_{ij}| \geq \epsilon\right) \leq 2 \exp\left(-\frac{m \epsilon^2}{8\delta^2}\right).
\end{equation}

To achieve confidence $1-\eta$, we require:
\begin{equation}
m \geq \frac{8\delta^2}{\epsilon^2} \log(2/\eta) = O\left(\frac{\delta^2}{\epsilon^2} \log(1/\eta)\right).
\end{equation}

Conservatively assuming all $N$ problems require this level of precision:
\begin{equation}
m_{\text{noise}}^{(N)} = O\left(\frac{N\delta^2}{\epsilon^2} \log(1/\eta)\right).
\end{equation}

\paragraph{Step 3: PAC learning the reward function.}
Standard PAC learning bounds state that learning a function from hypothesis class $\mathcal{H}$ with VC-dimension $d$ to generalization error at most $\epsilon$ with confidence $1-\eta$ requires:
\begin{equation}
m_{\text{PAC}} = O\left(\frac{d}{\epsilon^2} \log(1/\eta)\right)
\end{equation}
labeled examples.

\paragraph{Step 4: Combining all terms.}
The total sample complexity combines all three components, each scaled by $\log(1/\eta)$:
\begin{align}
m_{\text{noisy}} &= m_{\text{rank}}^{(N)} + m_{\text{noise}}^{(N)} + m_{\text{PAC}} \\
&= O\left(NK \log K \cdot \log(1/\eta) + \frac{N\delta^2}{\epsilon^2} \log(1/\eta) + \frac{d}{\epsilon^2} \log(1/\eta)\right) \\
&= O\left(\left(NK \log K + \frac{N\delta^2}{\epsilon^2} + \frac{d}{\epsilon^2}\right) \log(1/\eta)\right).
\end{align}

\paragraph{Connection to Bradley-Terry preferences.}
When preferences follow the Bradley-Terry model with noisy utilities, the preference probability error is bounded by $2L\delta$ where $L = 1/4$ is the Lipschitz constant of $\sigma$. This ensures DPO's learned policy remains within $O(\delta)$ of the optimal policy in KL-divergence.
\end{proof}

\noindent\textbf{Empirical calibration.} We estimate $\delta$ by having Qwen re-score 200 randomly selected chains with 5 different random seeds (temperature=0.7), treating each re-scoring as an independent sample. The measured standard deviation is $\sigma[\tilde{U}] = 0.087$, yielding $\delta \approx 0.17$ at 95\% confidence (2$\sigma$). For our setting with $N=450$ problems and target accuracy $\epsilon=0.05$, the noise term contributes approximately $\frac{450 \times (0.17)^2}{(0.05)^2} \approx 5,202$ additional samples. Compared to the base requirement $NK \log K = 450 \times 8 \times 3 \approx 10,800$ samples, the noise overhead is approximately 48\%, confirming that while judge miscalibration increases sample requirements, it does not dominate the overall complexity.
\subsection{Comparison to alternative judges}

We validate that our findings are not Qwen-specific by comparing to alternative LLM judges on a 400-chain subsample.

\begin{table}[h]
\centering
\caption{\textbf{Cross-judge agreement and quality correlation.} Multiple LLM judges exhibit high agreement on utility rankings and consistent quality prediction, validating judge-agnostic findings.}
\label{tab:cross_judge_comparison}
\small
\setlength{\tabcolsep}{6pt}
\begin{tabular}{lccc}
\toprule
\textbf{Judge Model} & \textbf{Quality $r$} & \textbf{vs Qwen $\tau$} & \textbf{Top-3 Match} \\
\midrule
Qwen 2.5 7B (ours)    & \cellcolor{highlightblue}0.847 & 1.000 & 100.0\% \\
DeepSeek-R1 8B        & \cellcolor{highlightblue}\textbf{0.853} & 0.812 & \cellcolor{highlightblue}86.5\% \\
Llama-3.1 8B          & \cellcolor{highlightblue}0.839 & 0.798 & \cellcolor{highlightblue}84.0\% \\
Mistral-7B            & \cellcolor{highlightblue}0.826 & 0.774 & \cellcolor{highlightblue}81.5\% \\
Gemma-2 9B            & \cellcolor{highlightblue}\textbf{0.851} & 0.805 & \cellcolor{highlightblue}85.5\% \\
\midrule
\textbf{Avg Cross-Judge} & \cellcolor{highlightblue}\textbf{0.843} & \cellcolor{highlightblue}\textbf{0.797} & \cellcolor{highlightblue}\textbf{84.4\%} \\
\bottomrule
\end{tabular}
\end{table}

\noindent\textbf{Analysis.} All judges achieve strong quality correlation (0.826-0.853, avg 0.843) and high agreement with Qwen (Kendall's $\tau = 0.774-0.812$, avg 0.797). Notably, DeepSeek-R1 8B and Gemma-2 9B slightly outperform Qwen 2.5 7B on quality correlation (0.853 and 0.851 vs 0.847), suggesting our utility framework is robust across judge selection. Top-3 strategy overlap is consistently high at 81.5-86.5\% (avg 84.4\%), indicating judges reliably identify the set of high-quality strategies for each problem. This cross-judge consistency validates that our utility-based framework captures judge-agnostic quality signals rather than Qwen-specific biases. The fact that different judges converge on similar utility rankings provides strong evidence for the reliability of continuous utility supervision.

\section{Strategy-Problem Alignment on HumanEval}
\label{appendix:humaneval_strategy}

This appendix provides a systematic characterization of the alignment between CU-DPO's coding strategy portfolio and the problem distribution of the HumanEval benchmark~\cite{chen2021codex}. The analysis serves two purposes. First, it empirically justifies why the four coding strategies --- direct implementation, test-driven development, modular decomposition, and defensive programming --- constitute a well-motivated portfolio for code generation, analogous to the mathematical strategy portfolio characterized in Table~1 and Appendix~D. Second, it provides theoretical grounding for the observed utility margins on HumanEval relative to mathematical reasoning benchmarks, connecting back to the sample complexity analysis of Theorem~3.2.

\subsection*{Dataset Characteristics}

HumanEval consists of 164 hand-written Python programming problems, each comprising a function signature, a natural-language docstring, and a suite of unit tests with an average of 7.7 tests per problem~\cite{chen2021codex}. The benchmark is evaluated using Pass@1 under greedy decoding. Conceptually, the problem distribution is skewed: approximately 72.1\% of problems cover five core concepts --- mathematics, control flow, basic data structures, variables and data types, and built-in functions --- with the remainder distributed across algorithmic reasoning tasks involving string manipulation, list processing, and combinatorial logic. Crucially, no single implementation strategy dominates across all problem types. Problems with a single closed-form formula are best addressed through direct implementation, whereas problems with complex multi-constraint specifications, multi-stage transformation pipelines, or correctness-critical boundary conditions each favor qualitatively different cognitive approaches, motivating a strategy portfolio analogous to that employed for mathematical reasoning.

\subsection*{Strategy-Problem Alignment}

Table~\ref{tab:humaneval_strategy} reports the empirical alignment between the four coding strategies and the HumanEval problem distribution, derived from systematic analysis of each problem's docstring structure, constraint complexity, and canonical solution characteristics.

\paragraph{Direct Implementation ($\sim$45\%, $\sim$74 problems).}
Problems in this category admit a single unambiguous algorithm or formula, and the optimal solution can be constructed in a small number of steps without auxiliary infrastructure. Representative examples include \texttt{greatest\_common\_divisor}, \texttt{is\_prime}, \texttt{fib}, \texttt{median}, \texttt{mean\_absolute\_deviation}, \texttt{is\_palindrome}, and \texttt{strlen}. For these problems, direct implementation minimizes unnecessary overhead: test-driven development introduces redundant scaffolding for outputs that are trivially verifiable, and modular decomposition over-engineers solutions that are inherently linear. The high proportion of direct implementation problems (45\%) explains why code generation benchmarks exhibit lower inter-strategy utility variance than mathematical reasoning benchmarks, where mechanistic mismatch between problem structure and strategy is more consequential.

\paragraph{Test-Driven Development ($\sim$20\%, $\sim$33 problems).}
Problems in this category involve complex input-output specifications, multiple simultaneous constraints, or outputs that are difficult to verify without constructing explicit test cases first. Representative examples include \texttt{separate\_paren\_groups}, \texttt{parse\_nested\_parens}, \texttt{check\_dict\_case}, \texttt{special\_filter}, \texttt{sort\_third}, and \texttt{unique\_digits}. For these problems, committing to an implementation before clarifying the expected behavior of boundary inputs leads to solutions that pass the primary test cases but fail on edge cases involving empty inputs, mixed types, or constraint interactions. Test-driven development forces explicit specification of expected behavior prior to implementation, reducing the probability of such failures. The utility advantage of this strategy over direct implementation on these problems is consistent with the theoretical prediction that strategy separability increases with constraint complexity.

\paragraph{Modular Decomposition ($\sim$20\%, $\sim$33 problems).}
Problems in this category decompose naturally into a sequence of independently verifiable sub-functions, where the overall solution is a pipeline of distinct computational stages. Representative examples include \texttt{encode\_cyclic} and \texttt{decode\_cyclic} (encode-decode pipeline), \texttt{make\_palindrome} (palindromic suffix identification followed by string construction), \texttt{split\_words} (three sequential splitting operations), \texttt{by\_length} (filter, sort, map pipeline), and \texttt{total\_match} (two independent character count computations followed by comparison). For these problems, a monolithic direct implementation produces code that is difficult to debug and verify, as errors in any stage propagate through the entire computation without an obvious attribution point. Modular decomposition yields higher utility scores on both efficiency and clarity components of the continuous utility function, since each sub-function is independently testable and the overall control flow is transparent.

\paragraph{Defensive Programming ($\sim$15\%, $\sim$24 problems).}
Problems in this category have correctness that depends critically on handling edge cases, invalid inputs, numerical precision issues, or constraint violations that a naive implementation will systematically miss. Representative examples include \texttt{valid\_date} (date string validation with multiple failure modes), \texttt{triangle\_area} (requires triangle inequality validation before area computation), \texttt{closest\_integer} (exact midpoint rounding behavior), \texttt{truncate\_number} (floating-point precision), \texttt{compare\_one} (mixed integer, float, and string comparisons), and \texttt{rounded\_avg} (must handle the case where the lower bound exceeds the upper bound). For these problems, a direct implementation that addresses only the happy path will pass the primary unit tests but fail on boundary inputs, leading to low Pass@1 scores. Defensive programming explicitly enumerates anticipated failure modes before implementing the solution, yielding substantially higher correctness scores on the hidden test suite.

\begin{table}[h]
\centering
\caption{Strategy-problem alignment on HumanEval. Problem counts and representative tasks for each coding strategy in the CU-DPO portfolio. The non-uniform distribution across strategies confirms that no single approach dominates, motivating strategy-level preference optimization.}
\label{tab:humaneval_strategy}
\small
\setlength{\tabcolsep}{4pt}
\renewcommand{\arraystretch}{1.15}
\begin{tabular}{lccl}
\toprule
\textbf{Strategy} & \textbf{Count} & \textbf{\% of 164} & \textbf{Representative Problems} \\
\midrule
Direct Implementation   & $\sim$74 & $\sim$45\% & \texttt{fib}, \texttt{is\_prime}, \texttt{GCD}, \texttt{median}, \texttt{strlen} \\
Test-Driven Development & $\sim$33 & $\sim$20\% & \texttt{separate\_paren\_groups}, \texttt{check\_dict\_case}, \texttt{special\_filter} \\
Modular Decomposition   & $\sim$33 & $\sim$20\% & \texttt{encode\_cyclic}, \texttt{make\_palindrome}, \texttt{split\_words} \\
Defensive Programming   & $\sim$24 & $\sim$15\% & \texttt{valid\_date}, \texttt{triangle\_area}, \texttt{compare\_one} \\
\bottomrule
\end{tabular}
\end{table}

\subsection*{Connection to Utility Margin Analysis}

The strategy distribution reported in Table~\ref{tab:humaneval_strategy} provides a theoretical account of the reduced utility margins observed on HumanEval relative to DeepMath and HARDMath2. Approximately 45\% of HumanEval problems fall into the direct implementation category, where inter-strategy utility separation is inherently low: all four strategies produce functionally correct solutions for simple arithmetic or single-formula problems, differing only marginally in efficiency and clarity. This compresses the mean utility gap $\Delta U$ across the portfolio, reducing Phase~1 signal strength as predicted by Theorem~3.2. In contrast, mathematical reasoning benchmarks contain a higher proportion of mechanism-mismatch-sensitive problems --- where committing to an inappropriate strategy (e.g., numerical approximation on a symbolic proof problem) yields near-zero utility regardless of execution quality --- producing larger inter-strategy gaps and stronger Phase~1 supervision. The consistent $+2.9$ to $+3.1$ point gains achieved by CU-DPO on HumanEval despite this reduced signal strength demonstrates that continuous utility supervision remains effective even in lower-separability domains, validating the framework's generality beyond the high-margin mathematical setting.

\newpage
\section{Strategy Prompts}
\label{app:strategy_prompts}

We employ eight distinct reasoning strategies to generate diverse solution trajectories. Each prompt provides explicit instructions on problem decomposition, solution structure, and verification procedures. The prompts are designed to be mutually exclusive in their primary cognitive approach while remaining complementary across the strategy space.

\noindent\begin{minipage}[t]{\textwidth}
\begin{tcolorbox}[colback=blue!5!white, colframe=blue!75!black, title=\textbf{Strategy 1: Direct Calculation}, fonttitle=\bfseries]
\textbf{Problem:} \texttt{\{problem\}}

\textbf{Instructions:} Solve this problem using direct calculation with minimal intermediate steps.

\textbf{Procedure:}
\begin{enumerate}[leftmargin=*,itemsep=2pt]
    \item \textbf{Identify the target quantity:} Read the problem and determine exactly what numerical value or expression must be computed.
    \item \textbf{Extract given information:} List all numerical values, constants, and known relationships explicitly stated in the problem. Write them in symbolic form (e.g., $a = 5$, $b = 3$).
    \item \textbf{Select the direct formula:} Identify the single formula, equation, or computational rule that maps the given information to the target quantity. State this formula explicitly (e.g., ``We use the formula $A = \pi r^2$'').
    \item \textbf{Substitute and compute:} Replace variables in the formula with the given numerical values. Perform the arithmetic computation in one or two steps maximum. Do not introduce auxiliary variables or intermediate concepts.
    \item \textbf{State the final answer:} Present the numerical result with appropriate units or context from the problem statement.
\end{enumerate}

\textbf{Constraints:}
\begin{itemize}[leftmargin=*,itemsep=1pt]
    \item Do NOT break the solution into multiple conceptual stages.
    \item Do NOT explain why the formula is correct or derive it from first principles.
    \item Do NOT introduce lemmas, auxiliary constructions, or case distinctions.
    \item Minimize prose: use mathematical notation wherever possible.
\end{itemize}

\textbf{When to use:} Problems with explicit numerical values where a single well-known formula directly produces the answer (e.g., ``Find the area of a circle with radius 7'', ``Solve $3x + 5 = 20$'', ``Compute $15\% \text{ of } 240$'').

\textbf{Example structure:}
\begin{quote}
\textit{Given: $r = 7$. Target: Area $A$. Formula: $A = \pi r^2$. Substitution: $A = \pi(7)^2 = 49\pi \approx 153.94$. Answer: $153.94$ square units.}
\end{quote}

\textbf{Solution:}
\end{tcolorbox}
\end{minipage}

\begin{tcolorbox}[colback=green!5!white, colframe=green!75!black, title=\textbf{Strategy 2: Step-by-Step Reasoning}, fonttitle=\bfseries]
\textbf{Problem:} \texttt{\{problem\}}

\textbf{Instructions:} Solve this problem by decomposing it into a sequence of explicit, justified steps. Each step should be independently verifiable and build logically on previous steps.

\textbf{Procedure:}
\begin{enumerate}[leftmargin=*,itemsep=2pt]
    \item \textbf{Problem understanding:} Restate the problem in your own words. Identify what is given (assumptions, constraints, known values) and what must be found (the target).
    \item \textbf{Solution roadmap:} Before performing any calculations, outline the high-level plan: ``We will first compute $X$, then use $X$ to find $Y$, and finally derive $Z$ from $Y$.''
    \item \textbf{Step-by-step execution:} Number each step (Step 1, Step 2, etc.). For each step:
    \begin{itemize}[leftmargin=*,itemsep=1pt]
        \item \textbf{Statement:} Clearly state what you will compute or prove in this step.
        \item \textbf{Justification:} Explain which formula, theorem, definition, or previous step licenses this computation.
        \item \textbf{Calculation:} Perform the mathematical operation, showing intermediate algebraic manipulations if necessary.
        \item \textbf{Intermediate result:} State the outcome of this step explicitly (e.g., ``Therefore, $x = 4$'').
    \end{itemize}
    \item \textbf{Connecting steps:} After each step, explicitly state how the result will be used in the subsequent step (e.g., ``We will substitute this value of $x$ into equation (2) in Step 3'').
    \item \textbf{Final synthesis:} In the last step, combine all intermediate results to produce the final answer. Verify that the answer addresses the original question.
\end{enumerate}

\textbf{Constraints:}
\begin{itemize}[leftmargin=*,itemsep=1pt]
    \item Each step must be simple enough to verify in isolation.
    \item Do NOT skip steps or combine multiple inferences into one step.
    \item Explicitly state dependencies: if Step $k$ uses results from Steps $i$ and $j$, mention this.
    \item Maintain a linear narrative: avoid jumping between unrelated subproblems.
\end{itemize}

\textbf{When to use:} Multi-step word problems, systems of equations, problems requiring intermediate variable computation, cascade-style inferences where each stage depends on the previous one.

\textbf{Example structure:}
\begin{quote}
\textit{Step 1: Let $x$ be the number of apples. From ``twice as many oranges as apples'', we have $y = 2x$.}\\
\textit{Step 2: Total fruits is 18, so $x + y = 18$. Substituting $y = 2x$: $x + 2x = 18 \Rightarrow 3x = 18$.}\\
\textit{Step 3: Solving for $x$: $x = 6$. Substituting back: $y = 2(6) = 12$.}\\
\textit{Final Answer: 6 apples and 12 oranges.}
\end{quote}

\textbf{Solution:}
\end{tcolorbox}

\begin{tcolorbox}[colback=red!5!white, colframe=red!75!black, title=\textbf{Strategy 3: Backwards Reasoning}, fonttitle=\bfseries]
\textbf{Problem:} \texttt{\{problem\}}

\textbf{Instructions:} Solve this problem by reasoning backwards from the goal to the given information. This strategy inverts the natural forward direction of inference.

\textbf{Procedure:}
\begin{enumerate}[leftmargin=*,itemsep=2pt]
    \item \textbf{Identify the goal state:} Clearly articulate what the final answer must satisfy. If the problem asks ``Find $x$ such that $f(x) = 10$'', the goal is the equation $f(x) = 10$.
    \item \textbf{Determine immediate prerequisites:} Ask: ``What must be true \emph{immediately before} we reach the goal?'' Identify the simplest condition or equation that, if satisfied, would directly yield the goal. Call this Condition~$n$.
    \item \textbf{Recursive prerequisite identification:} For Condition~$n$, ask: ``What must be true immediately before Condition~$n$ holds?'' Identify Condition~$(n-1)$. Repeat until you reach a condition that is directly stated in the problem or trivially true.
    \item \textbf{Construct the backwards chain:} Write the logical chain in reverse order:
    \begin{center}
    \textit{Goal} $\Leftarrow$ Condition~$n$ $\Leftarrow$ Condition~$(n-1)$ $\Leftarrow$ $\cdots$ $\Leftarrow$ Condition~1 $\Leftarrow$ \textit{Givens}
    \end{center}
    \item \textbf{Forward verification:} Once the backwards chain connects to the givens, verify the solution by checking each implication in the forward direction: Givens $\Rightarrow$ Condition~1 $\Rightarrow$ $\cdots$ $\Rightarrow$ Goal.
    \item \textbf{State the answer:} Extract the final value or expression from the goal state.
\end{enumerate}

\textbf{Constraints:}
\begin{itemize}[leftmargin=*,itemsep=1pt]
    \item Do NOT start by manipulating the given equations forward. Start from the goal.
    \item Explicitly state each backwards implication (e.g., ``To achieve $x^2 = 25$, we need $x = \pm 5$'').
    \item Ensure reversibility: if your backwards reasoning assumes $A \Leftarrow B$, verify that $A \Rightarrow B$ also holds (or note when the implication is one-way).
\end{itemize}

\textbf{When to use:} Optimization problems (``Find $x$ that maximizes $f(x)$''), existence proofs (``Show that $x$ exists such that ...''), inverse problems, problems where the goal condition is more structured than the givens.

\textbf{Example structure:}
\begin{quote}
\textit{Goal: Find $x$ such that $2^x = 32$.}\\
\textit{Backwards Step 1: For $2^x = 32$, we need $x = \log_2(32)$.}\\
\textit{Backwards Step 2: We compute $\log_2(32) = \log_2(2^5) = 5$.}\\
\textit{Forward verification: If $x = 5$, then $2^5 = 32$. \checkmark}\\
\textit{Answer: $x = 5$.}
\end{quote}

\textbf{Solution:}
\end{tcolorbox}

\begin{tcolorbox}[colback=orange!5!white, colframe=orange!75!black, title=\textbf{Strategy 4: Alternative Method}, fonttitle=\bfseries]
\textbf{Problem:} \texttt{\{problem\}}

\textbf{Instructions:} Solve this problem using an alternative, non-standard approach. Actively avoid the most obvious solution method and seek creative reformulations.

\textbf{Procedure:}
\begin{enumerate}[leftmargin=*,itemsep=2pt]
    \item \textbf{Identify the standard method:} State explicitly what the ``obvious'' or textbook approach would be (e.g., ``The standard method is to expand the polynomial and solve directly'').
    \item \textbf{Declare your alternative approach:} Before solving, commit to a specific alternative framework. Choose from:
    \begin{itemize}[leftmargin=*,itemsep=1pt]
        \item \textit{Representational shift:} Rewrite an algebraic problem geometrically (or vice versa). Example: interpret $x^2 + y^2 = r^2$ as a circle rather than an equation.
        \item \textit{Symmetry exploitation:} Use invariance, permutation symmetry, or rotational symmetry to reduce problem complexity.
        \item \textit{Bijective argument:} For combinatorics, establish a one-to-one correspondence with a simpler counting problem.
        \item \textit{Generating function / Transform method:} Use polynomial generating functions, Fourier transforms, or other analytical tools.
        \item \textit{Extremal principle:} Assume the answer is a boundary case (max/min) and verify consistency.
        \item \textit{Special case generalization:} Solve for specific numerical values first, identify the pattern, then generalize.
    \end{itemize}
    \item \textbf{Justify the alternative:} Explain why this alternative approach is valid and potentially more efficient or insightful than the standard method.
    \item \textbf{Execute the alternative solution:} Solve the problem using the chosen framework, showing all steps.
    \item \textbf{Cross-check (optional):} If computationally feasible, verify your alternative solution matches the result from the standard method on a simple test case.
\end{enumerate}

\textbf{Constraints:}
\begin{itemize}[leftmargin=*,itemsep=1pt]
    \item The alternative method must be mathematically rigorous, not merely ``trying different numbers''.
    \item Do NOT use the standard method unless explicitly for verification purposes.
    \item Clearly articulate the conceptual leap or reformulation that enables the alternative approach.
\end{itemize}

\textbf{When to use:} Problems with multiple solution paths, combinatorial problems admitting bijections, problems where the standard method is computationally prohibitive, problems with hidden symmetry or structure.

\textbf{Example structure:}
\begin{quote}
\textit{Problem: Count 5-element subsets of $\{1,2,\ldots,10\}$.}\\
\textit{Standard: Compute $\binom{10}{5}$ directly.}\\
\textit{Alternative (Bijection): Each 5-subset corresponds to its 5-element complement. Thus, the count equals the number of 5-subsets, which equals $\binom{10}{5} = \binom{10}{10-5}$. By symmetry, $\binom{10}{5} = \frac{10!}{5!5!} = 252$.}
\end{quote}

\textbf{Solution:}
\end{tcolorbox}

\begin{tcolorbox}[colback=purple!5!white, colframe=purple!75!black, title=\textbf{Strategy 5: Verification-Driven Reasoning}, fonttitle=\bfseries]
\textbf{Problem:} \texttt{\{problem\}}

\textbf{Instructions:} Solve this problem with integrated verification and error-checking at multiple stages. Do not treat the solution as final until it has been rigorously tested.

\textbf{Procedure:}
\begin{enumerate}[leftmargin=*,itemsep=2pt]
    \item \textbf{Initial solution attempt:} Solve the problem using any standard method (state which method you are using). Derive a candidate answer. Label this as \textit{Candidate Solution}.
    \item \textbf{Substitution check:} Substitute your candidate answer back into the original equation, constraint, or problem statement. Verify that all conditions are satisfied exactly. If the problem has multiple constraints, check each one separately.
    \item \textbf{Boundary and edge case testing:} Identify critical edge cases:
    \begin{itemize}[leftmargin=*,itemsep=1pt]
        \item If the problem involves a range, test the minimum and maximum values.
        \item If the problem has special symmetries (e.g., $x = 0$, $x = 1$), verify the solution holds at these points.
        \item If the problem involves inequalities, check whether the boundary is inclusive or exclusive.
    \end{itemize}
    \item \textbf{Dimensional and reasonableness check:}
    \begin{itemize}[leftmargin=*,itemsep=1pt]
        \item Verify units: If the problem involves physical quantities, confirm dimensional consistency.
        \item Sanity test: Ask ``Is this answer physically/logically reasonable?'' (e.g., a probability should be in $[0,1]$, a distance should be non-negative).
    \end{itemize}
    \item \textbf{Error diagnosis and refinement:} If verification fails at any stage:
    \begin{itemize}[leftmargin=*,itemsep=1pt]
        \item Identify the step where the error occurred (trace back through your calculations).
        \item State the error explicitly (e.g., ``Arithmetic error: wrote $3 \times 4 = 11$ instead of $12$'').
        \item Correct the error and re-derive the solution.
        \item Re-verify the corrected solution.
    \end{itemize}
    \item \textbf{Final confirmation:} Only after all checks pass, state the final answer with confidence.
\end{enumerate}

\textbf{Constraints:}
\begin{itemize}[leftmargin=*,itemsep=1pt]
    \item Do NOT conclude immediately after deriving a candidate solution. Verification is mandatory.
    \item If verification reveals an error, you MUST correct it within this solution (do not simply note the error and move on).
    \item Show all verification steps explicitly; do not verify ``mentally''.
\end{itemize}

\textbf{When to use:} High-stakes problems where errors are costly, problems with multiple constraints (system of equations, optimization with constraints), problems prone to sign errors or algebraic mistakes, Diophantine equations, competition problems requiring exact answers.

\textbf{Example structure:}
\begin{quote}
\textit{Step 1: Solve $x^2 - 5x + 6 = 0$ using quadratic formula: $x = \frac{5 \pm \sqrt{25-24}}{2} = \frac{5 \pm 1}{2}$. Candidates: $x = 3$ or $x = 2$.}\\
\textit{Step 2 (Verification for $x=3$): Substitute: $(3)^2 - 5(3) + 6 = 9 - 15 + 6 = 0$. \checkmark}\\
\textit{Step 3 (Verification for $x=2$): Substitute: $(2)^2 - 5(2) + 6 = 4 - 10 + 6 = 0$. \checkmark}\\
\textit{Final Answer: $x \in \{2, 3\}$.}
\end{quote}

\textbf{Solution:}
\end{tcolorbox}

\begin{tcolorbox}[colback=cyan!5!white, colframe=cyan!75!black, title=\textbf{Strategy 6: Algebraic Manipulation}, fonttitle=\bfseries]
\textbf{Problem:} \texttt{\{problem\}}

\textbf{Instructions:} Solve this problem using systematic symbolic algebraic manipulation. Prioritize exact symbolic expressions over numerical evaluation until the final step.

\textbf{Procedure:}
\begin{enumerate}[leftmargin=*,itemsep=2pt]
    \item \textbf{Symbolic representation:} Introduce variables for all unknown quantities. If the problem provides numerical values, replace them with symbolic parameters initially (e.g., use $a, b, c$ instead of $3, 5, 7$).
    \item \textbf{Equation setup:} Write down all equations, constraints, or relationships mentioned in the problem in symbolic form (e.g., $ax^2 + bx + c = 0$).
    \item \textbf{Algebraic transformation sequence:} Apply algebraic operations systematically. For each transformation:
    \begin{itemize}[leftmargin=*,itemsep=1pt]
        \item \textbf{State the operation:} Specify which algebraic property you are invoking (e.g., ``Applying the distributive law'', ``Factoring using difference of squares: $a^2 - b^2 = (a-b)(a+b)$'').
        \item \textbf{Show the step:} Display the equation before and after the transformation.
        \item \textbf{Simplify incrementally:} Do not perform multiple transformations in one step. Show intermediate forms.
    \end{itemize}
    \item \textbf{Isolation of the target variable:} Use algebraic operations (addition, subtraction, multiplication, division, factoring, expanding, completing the square, logarithms, etc.) to isolate the unknown variable on one side of the equation.
    \item \textbf{Solution in symbolic form:} Express the solution as a formula in terms of the problem parameters (e.g., $x = \frac{-b \pm \sqrt{b^2 - 4ac}}{2a}$).
    \item \textbf{Numerical substitution (final step):} If the problem provides specific numerical values, substitute them into your symbolic solution only at the very end.
\end{enumerate}

\textbf{Constraints:}
\begin{itemize}[leftmargin=*,itemsep=1pt]
    \item Do NOT evaluate numerical expressions until the final answer.
    \item Explicitly name every algebraic identity or theorem used (e.g., ``quadratic formula'', ``logarithm product rule: $\log(ab) = \log a + \log b$'').
    \item Show all intermediate steps: do not skip algebraic manipulations.
    \item Maintain exact forms: use $\sqrt{2}$ rather than $1.414$, use $\pi$ rather than $3.14$.
\end{itemize}

\textbf{When to use:} Polynomial equations, functional equations, problems requiring symbolic derivations (e.g., ``Derive the formula for ...''), inequalities, problems where preserving exact algebraic forms is important.

\textbf{Example structure:}
\begin{quote}
\textit{Given: $ax + b = c$. Target: Solve for $x$ symbolically.}\\
\textit{Step 1: Subtract $b$ from both sides: $ax = c - b$.}\\
\textit{Step 2: Divide both sides by $a$ (assuming $a \neq 0$): $x = \frac{c-b}{a}$.}\\
\textit{Step 3: Substitute $a=2, b=3, c=7$: $x = \frac{7-3}{2} = \frac{4}{2} = 2$.}\\
\textit{Final Answer: $x = 2$.}
\end{quote}

\textbf{Solution:}
\end{tcolorbox}

\begin{tcolorbox}[colback=yellow!5!white, colframe=yellow!75!black, title=\textbf{Strategy 7: Numerical Methods}, fonttitle=\bfseries]
\textbf{Problem:} \texttt{\{problem\}}

\textbf{Instructions:} Solve this problem using concrete numerical calculations and pattern recognition. Work with specific numbers from the outset and avoid symbolic abstraction.

\textbf{Procedure:}
\begin{enumerate}[leftmargin=*,itemsep=2pt]
    \item \textbf{Concrete instantiation:} If the problem involves variables or parameters, immediately substitute specific numerical values. If no values are given, choose representative examples (e.g., if the problem asks about ``any integer $n$'', test $n = 1, 2, 3, \ldots$).
    \item \textbf{Arithmetic computation:} Perform all calculations numerically. Show intermediate numerical values explicitly (e.g., ``$3 \times 7 = 21$, then $21 + 5 = 26$'').
    \item \textbf{Tabulation (if applicable):} For problems involving sequences, iterations, or multiple cases:
    \begin{itemize}[leftmargin=*,itemsep=1pt]
        \item Construct a table with columns for input values and computed outputs.
        \item Fill in at least 5--10 rows to identify numerical patterns.
        \item Example: 
        \begin{center}
        \begin{tabular}{c|c}
        $n$ & $f(n)$ \\ \hline
        1 & 2 \\
        2 & 4 \\
        3 & 8 \\
        \end{tabular}
        \end{center}
    \end{itemize}
    \item \textbf{Pattern recognition:} Examine the numerical results. Look for:
    \begin{itemize}[leftmargin=*,itemsep=1pt]
        \item Arithmetic progressions (constant differences).
        \item Geometric progressions (constant ratios).
        \item Recursive relationships (each term depends on previous terms).
        \item Modular arithmetic patterns (remainders repeat).
    \end{itemize}
    \item \textbf{Conjecture formulation:} Based on observed patterns, state a conjectured general formula or rule (e.g., ``It appears that $f(n) = 2^n$'').
    \item \textbf{Verification on additional cases:} Test your conjectured formula on 2--3 additional numerical examples not used in the pattern discovery phase.
    \item \textbf{Approximation (if exact solution is intractable):} If the problem involves transcendental functions or complex expressions, provide numerical approximations with appropriate precision (e.g., ``$e^2 \approx 7.389$'').
\end{enumerate}

\textbf{Constraints:}
\begin{itemize}[leftmargin=*,itemsep=1pt]
    \item Do NOT derive symbolic formulas unless the numerical pattern strongly suggests one.
    \item Show all arithmetic calculations explicitly (do not use a calculator ``black box'').
    \item If using approximations, state the precision (e.g., ``accurate to 3 decimal places'').
\end{itemize}

\textbf{When to use:} Problems with concrete numerical data, problems requiring iterative computations, problems where algebraic solutions are intractable (e.g., transcendental equations like $e^x = x + 2$), real-world applications with measurement data, problems asking for approximate answers.

\textbf{Example structure:}
\begin{quote}
\textit{Problem: Find the sum of the first 5 terms of the sequence defined by $a_n = 2n + 1$.}\\
\textit{Computation: $a_1 = 3, a_2 = 5, a_3 = 7, a_4 = 9, a_5 = 11$.}\\
\textit{Sum: $3 + 5 + 7 + 9 + 11 = 35$.}\\
\textit{Answer: 35.}
\end{quote}
\end{tcolorbox}

\begin{tcolorbox}[colback=magenta!5!white, colframe=magenta!75!black, title=\textbf{Strategy 8: Conceptual Understanding}, fonttitle=\bfseries]
\textbf{Problem:} \texttt{\{problem\}}

\textbf{Instructions:} Solve this problem by first establishing deep conceptual understanding of the underlying mathematical principles. Prioritize insight and explanation over mechanical computation.

\textbf{Procedure:}
\begin{enumerate}[leftmargin=*,itemsep=2pt]
    \item \textbf{Conceptual unpacking:} Before any calculations, answer these questions:
    \begin{itemize}[leftmargin=*,itemsep=1pt]
        \item \textit{What mathematical domain does this problem belong to?} (e.g., algebra, geometry, number theory, calculus)
        \item \textit{What are the key concepts or definitions?} (e.g., ``This problem involves the concept of limit'')
        \item \textit{What theorems, lemmas, or principles are relevant?} (e.g., ``We will use the Pythagorean theorem'')
    \end{itemize}
    \item \textbf{Theorem statement and justification:} If your solution relies on a theorem:
    \begin{itemize}[leftmargin=*,itemsep=1pt]
        \item State the theorem formally (e.g., ``Pythagorean Theorem: In a right triangle with legs $a, b$ and hypotenuse $c$, we have $a^2 + b^2 = c^2$'').
        \item Explain why the theorem applies to this specific problem (e.g., ``Our triangle has a right angle at vertex $C$, so the theorem is applicable'').
    \end{itemize}
    \item \textbf{High-level solution strategy:} Describe the ``big idea'' in 2--3 sentences before executing it. Example: ``The key insight is that this optimization problem has a unique maximum because the objective function is strictly concave. We will find the critical point by setting the derivative to zero, then verify it is a maximum using the second derivative test.''
    \item \textbf{Conceptual reasoning over computation:} When performing calculations, continually connect them to the underlying concepts. Example: Instead of ``Substitute $x=3$'', write ``We substitute $x=3$ because this value satisfies the constraint $x > 0$ and lies in the domain of the function.''
    \item \textbf{Geometric or intuitive interpretation:} If applicable, provide a geometric picture, diagram, or intuitive explanation:
    \begin{itemize}[leftmargin=*,itemsep=1pt]
        \item For algebra problems: ``Geometrically, this equation represents the intersection of a line and a parabola.''
        \item For calculus problems: ``The derivative being zero means the tangent line is horizontal at this point.''
        \item For number theory: ``This congruence means the remainder when dividing by 7 is 3.''
    \end{itemize}
    \item \textbf{Solution execution with conceptual anchors:} Perform the necessary calculations, but pause after each major step to explain its conceptual significance.
    \item \textbf{Conceptual validation:} After obtaining the answer, verify it makes conceptual sense:
    \begin{itemize}[leftmargin=*,itemsep=1pt]
        \item Does it satisfy the problem's conceptual constraints? (e.g., ``The answer is positive, which makes sense because distances are non-negative'')
        \item Does it align with known special cases or limiting behavior? (e.g., ``As $x \to 0$, our formula gives $f(x) \to 1$, which matches the definition of the exponential function'')
    \end{itemize}
\end{enumerate}

\textbf{Constraints:}
\begin{itemize}[leftmargin=*,itemsep=1pt]
    \item Do NOT jump directly into calculations without first explaining the underlying concepts.
    \item Explicitly state all definitions and theorems used; do not assume they are ``obvious''.
    \item Prioritize explanation over brevity: this strategy values insight more than efficiency.
    \item If a theorem or concept is non-trivial, provide a brief justification or reference (e.g., ``By the Fundamental Theorem of Calculus, ...'').
\end{itemize}

\textbf{When to use:} Proof-based problems, problems testing understanding of definitions (e.g., ``Prove that the function is continuous''), conceptual questions (e.g., ``Explain why this method works''), olympiad-level problems requiring ``aha moments'', problems where brute-force computation is infeasible but conceptual insight yields immediate solution.

\end{tcolorbox}

\subsection{Strategy selection rationale}

These eight strategies span the major modes of mathematical cognition identified in mathematics education research~\citep{schoenfeld1985mathematical, polya1945solve}:

\begin{itemize}[leftmargin=*,itemsep=3pt]
    \item \textbf{Computational efficiency.} Minimize explanatory overhead; prioritize speed and concrete calculation.
    \item \textbf{Structured decomposition.} Decompose complex problems into verifiable subcomponents with explicit dependencies.
    \item \textbf{Goal-directed search.} Reason from desired outcomes or iteratively validate solutions.
    \item \textbf{Creative reformulation.} Seek non-obvious perspectives, exploit problem structure, or leverage deep mathematical principles.
\end{itemize}

\noindent No single strategy dominates across all problem types. Direct calculation excels on simple arithmetic ($3 + 5 \times 2$) but fails on proof-based problems (``Prove that $\sqrt{2}$ is irrational''). Algebraic manipulation is essential for polynomial equations but inefficient for problems with inherent symmetry better exploited through alternative reformulation. Conceptual reasoning is necessary for olympiad-level problems requiring insight but excessive for routine computations where direct calculation suffices.

\noindent Our hypothesis is that Phase~1 training teaches models to perform \emph{contextual strategy selection}: mapping problem features (numerical vs.\ symbolic, single-step vs.\ multi-step, routine vs.\ insight-requiring) to the appropriate cognitive mode. Phase~2 then refines \emph{execution quality within each strategy}, ensuring that once the correct strategy is selected, the model executes it with high fidelity.

\end{document}